\documentclass[11pt,letterpaper,reqno]{amsart}
\usepackage[foot]{amsaddr}
\usepackage[margin=1in]{geometry}
\usepackage{amsmath,amsthm,amsfonts,amssymb,mathrsfs}
\usepackage{bbm,algorithm,algpseudocode,xcolor}
\usepackage{hyperref,color,enumerate,cleveref}
\usepackage{enumitem}
\usepackage{graphicx}
\usepackage{bm}

\makeatletter
\def\paragraph{\@startsection{paragraph}{4}%
  \z@\z@{-\fontdimen2\font}%
  {\normalfont\bfseries}}
\makeatother

\allowdisplaybreaks

\newtheorem{thm}{Theorem}[section]

\newtheorem{theorem}[thm]{Theorem}
\newtheorem{corollary}[thm]{Corollary}
\newtheorem{lemma}[thm]{Lemma}
\newtheorem{proposition}[thm]{Proposition}

\theoremstyle{definition}

\newtheorem{remark}[thm]{Remark}

\newtheorem{assumption}[thm]{Assumption}
\newtheorem{definition}[thm]{Definition}

\newcommand{\R}{\mathbb{R}}

\newcommand{\bP}{\mathbb P}
\newcommand{\Var}{\mathrm{Var}}
\newcommand{\cH}{\mathcal H}

\newcommand{\tr}{\mathrm{Tr}}
\newcommand{\Sym}{\mathrm{Sym}}
\newcommand{\Skew}{\mathrm{Skew}}
\newcommand{\diag}{\mathrm{diag}}
\newcommand{\dist}{\mathrm{dist}}
\newcommand{\Spec}{\mathrm{Spec}}

\let \phi \varphi

\DeclareMathOperator{\Cov}{Cov}

\DeclareMathOperator{\supp}{supp}
\DeclareMathOperator*{\argmin}{argmin}

\newcommand{\rect}{\mathtt{Rectify}}
\newcommand{\crect}{\text{c-}\mathtt{Rectify}}

\newcommand{\mL}{\mathcal L}

\newcommand{\mP}{\mathcal P}

\newcommand{\N}{\mathcal{N}}
\newcommand{\1}{\mathbf 1}

\newcommand{\vv}[1]{{\boldsymbol{#1}}}
\newcommand{\diff}{\mathrm{d}}
\newcommand{\E}{\mathbb{E}}
\def\bbb#1\eee{\begin{align}#1\end{align}}
\def\bb#1\ee{\begin{align*}#1\end{align*}}
\DeclareMathOperator{\proj}{proj}
\newcommand{\id}{\mathrm{id}}

\let\cal\mathcal
\newenvironment{acks}[1][Acknowledgments]{\section*{#1}}{}

\title[Comp \& Stat Guarantees of \textit{c}-Rect flow]{Computational and Statistical Guarantees of the \textit{c}-Rectified flow}

\author[Leda Wang]{Leda Wang$^{1}$}
\author[Zhehao Xu]{Zhehao Xu$^{1}$}
\author[Qiang Liu]{Qiang Liu$^{2}$}
\author[Harrison H. Zhou]{Harrison H. Zhou$^{1}$}
\email{leda.wang@yale.edu}
\email{zhehao.xu@yale.edu}
\email{qiang.liu.research@gmail.com}
\email{huibin.zhou@yale.edu}
\address{$^{1}$Department of Statistics and Data Science, Yale University}
\address{$^{2}$Department of Computer Science, University of Texas at Austin}
\subjclass{Primary 49Q22, 62G05; Secondary 68T07}
\begin{document}

\begin{abstract} Recently, rectified flow has emerged as a fundamental framework for large-scale image generation, powering state-of-the-art systems such as FLUX.1 and Stable Diffusion 3. Despite its remarkable empirical success, the computational and statistical guarantees of iterative rectified flow have remained largely unexplored. We address this problem by studying \textit{c}-rectified flow, a cost-aware class of rectified flow that projects velocity fields onto a gradient class while preserving endpoint marginals. The ordinary rectified flow can fail to recover the optimal transport coupling: in a Gaussian case study, the iteration converges to the optimal coupling if and only if the source and target covariance matrices commute. More generally, for a Gaussian to log-concave coupling, the ordinary rectified flow can reach a fixed point after one step from independent initialization. In contrast, under suitable regularity and uniform-integrability assumptions, iterative \textit{c}-rectified flow always converges to the optimal transport coupling. We further establish quantitative one-step contraction and exponential convergence guarantees under projection-stability assumptions for both quadratic and strongly convex displacement costs. Finally, under a Hölder ball assumption, we develop new minimax-optimal score estimation rates and show that, when combined with iterative \textit{c}-rectified flow, they yield a rate-optimal estimator of the optimal transport for the dimension \(d \ge 3\) and a nearly parametric rate for \(d=1,2\).

\end{abstract}

\maketitle

% \tableofcontents
\section{Introduction}
Rectified flow, introduced by \cite{liu2022flow}, is a simple and effective
framework for learning continuous-time transports between distributions. Its
least-squares training objective is readily implemented with neural networks,
while iterative reflow straightens transport trajectories and enables accurate
generation with few discretization steps \cite{liu2024instaflow}. Closely related formulations such as
flow matching \cite{lipman2023flow} and stochastic interpolants
\cite{albergo2023stochastic} have further established regression-based
probability flows as a general framework for generative modeling. The original
rectified-flow experiments demonstrated strong performance in image generation,
image-to-image translation, and domain adaptation \cite{liu2022flow}; InstaFlow
subsequently used reflow to construct a one-step text-to-image generator
\cite{liu2024instaflow}. At a much larger scale, rectified-flow transformers
have become a practical foundation for high-resolution image synthesis: Stable
Diffusion~3 scales this architecture with a multimodal transformer
\cite{esser2024scaling}, while FLUX.1 Kontext uses flow matching for
large-scale in-context image generation and editing
\cite{blackforestlabs2025flux}. These developments make rectification more than
a device for simplifying simulation and motivate  a deeper understanding of the transport
coupling from the procedure. Yet the computational and statistical guarantees of iterative rectified flow have remained largely unexplored.

We study this question in the following statistical setting. Suppose that we observe two
independent samples
\[
X_1,\ldots,X_n\overset{\mathrm{i.i.d.}}{\sim}P,
\qquad
Y_1,\ldots,Y_n\overset{\mathrm{i.i.d.}}{\sim}Q,
\]
where the endpoint distributions \(P\) and \(Q\) are unknown. Our goal is to
estimate the optimal transport map \(T^*\) from \(P\) to \(Q\), namely the
cost-minimizing measurable map satisfying \(T_\#P=Q\). The
statistical objective is to construct an estimator \(\widehat T_n\) whose
\(L^2(P)\)-risk relative to \(T^*\) attains the optimal rate over the
distribution classes specified later.

In the one-dimensional setting, the connection between rectified flow and optimal
transport is particularly interesting. Under rectifiability assumptions that guarantee
uniqueness of the defining ODE, trajectories issued from two ordered initial
points cannot cross at a common time; otherwise, two ODE solutions would pass
through the same space--time point, contradicting uniqueness. The endpoint flow
map is therefore monotone, which immediately implies that it is the one-dimensional optimal
transport map. This noncrossing argument underlies one-step recovery results for
one-dimensional rectified flow
\cite{bansal2025wasserstein,hertrich2025relation}; in particular,
\cite{hertrich2025relation} show that every rectifiable one-dimensional
coupling is mapped to the optimal coupling by a single rectification step. The
conclusion is intrinsically one-dimensional: in higher dimensions, the
nonintersection of ODE trajectories does not by itself impose the cyclic
monotonicity, or equivalently the cost-dependent gradient structure, that
characterizes optimal transport.

The statistical analysis of rectified flow can be challenging.
The analysis in \cite{mena2025statistical} studies empirical rectified-flow
estimators and, under some regularity and H{\"o}lder smoothness assumptions, establishes a convergence rate that is slower than the minimax-optimal rate for optimal transport estimation but faster than the standard nonparametric rate. This naturally raises the question of whether rectified flow is statistically optimal for estimating the optimal transport.

These computational and statistical results leave three closely related questions. Is
a single step of rectified flow, or its cost-aware analogue, sufficient beyond the
one-dimensional setting? If repeated rectification is necessary, does the
iteration converge to optimal transport, and how fast? Finally, can an
iterative construction attain the
minimax rate for optimal transport map estimation? We answer these questions by studying the \textit{c}-rectified flow of
\cite{liu2022rectified}, which projects the velocity onto the appropriate
cost-dependent gradient class. Under suitable conditions, iterative
\textit{c}-rectified flow converges to optimal transport. Moreover, under additional structural assumptions, we establish quantitative one-step contraction and exponential convergence guarantees. Finally, when combined with score-based estimation of the endpoint
marginals, it can yield a rate-optimal estimator of the optimal transport map.

\subsection{Our Contributions}
The main contributions of the paper are as follows.

\begin{itemize}
    \item \textbf{Vanilla rectified flow can fail to select optimal transport.}
    In a Gaussian-to-Gaussian example, ordinary rectified flow initialized from the
    independent coupling produces--after one step--a deterministic linear
    coupling that is fixed by all subsequent iterations of the rectified flow. This fixed point is the
    Gaussian optimal transport map if and only if the source and target covariance matrices
    commute. Thus cost decrease and marginal preservation properties alone do not guarantee
    convergence to optimal transport. 
    We further extend the one-step fixed point result beyond Gaussian targets, showing that rectified flow initialized from the independent coupling can reach a fixed point after one step for a Gaussian to log-concave transport problem.

    \item \textbf{\textit{c}-rectified flow converges qualitatively to optimal transport.}
    The cost-aware gradient structure in \textit{c}-rectified flow rules out the
    non-optimal fixed points above. Under suitable regularity and uniform-integrability
    assumptions, for the quadratic cost, the transport cost along the iterates
    converges to the optimal value, and the uniqueness of the optimal coupling gives
    weak convergence of the couplings. We also prove the analogous qualitative
    convergence statement for more general displacement costs.

    \item \textbf{The convergence speed is exponentially fast.}
    We establish one-step decrease and exponential convergence guarantees under the rectifiable condition of \cite{liu2022flow}. In
    the Gaussian setting, the relevant projection mechanism can be computed
    explicitly. In a general quadratic setting, Lipschitz regularity, strong
    convexity of the Brenier map, and a projection-stability condition imply
    geometric decay of the excess transport cost. The same contraction framework
    extends beyond quadratic loss to strongly convex displacement costs.

    \item \textbf{A new rate-optimal score estimator yields a rate-optimal map estimator.}
    For a multivariate Gaussian-ratio model with unbounded support, we construct an aggregate score
    estimator and prove matching upper and lower bounds across the
    low-, intermediate-, and large-noise regimes. These new minimax-optimal
    score-estimation rates lead to score-based marginal estimators with explicit
    Wasserstein rates. Combining these marginal estimators
    with iterative \textit{c}-rectified flow yields a rate-optimal estimator of
    the optimal transport map for the dimension \(d \ge 3\) and a nearly parametric rate for \(d=1,2\).
\end{itemize}

Conceptually, the paper connects three viewpoints that are often studied
separately: classical optimal transport, iterative flow-based transport
algorithms, and statistical estimation from data. The main message is that
\textit{c}-rectified flow is not merely a heuristic cost-reduction procedure:
with the right cost-aware gradient structure, it can serve as a principled
algorithmic bridge from estimated marginals to optimal transport maps.

\subsection{Related Works}
\label{sec:related-work}

Recent work on rectified flow and flow matching provides the conceptual
starting point for our analysis. The original rectified flow formulation,
together with the marginal-preserving \textit{c}-rectified flow variant, shows that
iterative rectification can decrease transport costs while preserving endpoint
marginals \cite{liu2022flow,liu2022rectified}. Closely related frameworks,
including flow matching \cite{lipman2023flow} and stochastic interpolants
\cite{albergo2023stochastic}, place these methods in a broader
regression-based framework for learning transport dynamics between
distributions. Our work focuses instead on the limiting coupling generated by
rectification and on conditions under which this coupling converges to
classical optimal transport.

The closest strand of prior work examines when vanilla rectified flow recovers optimal
transport. The analysis in \cite{bansal2025wasserstein} gives Wasserstein error
bounds for the sampling distribution of rectified flow and relates
discretization error to a much stronger notion of straightness; in one dimension
with a Gaussian source, it also identifies regimes in which one-step rectified
flow recovers the optimal transport map. A complementary statistical analysis is given in
\cite{mena2025statistical}, which studies existence, uniqueness, regularity,
and limit distributions for empirical rectified-flow estimators. In the
Gaussian setting, that work shows that the rectified map agrees with the
optimal transport map in one-step precisely when the source and target
covariance matrices commute. 
A sharper structural characterization of the relation between rectified flow
and optimal transport is developed in \cite{hertrich2025relation}, which
derives explicit formulas in Gaussian and Gaussian-mixture settings and proves
that every rectifiable one-dimensional coupling is mapped to the optimal
transport coupling after a single rectification step. 
% It also shows that
% several previously claimed rectifiability statements require stronger
% regularity assumptions than had been emphasized. 
These results identify special cases in which
rectified flow is compatible with optimal transport. Our focus is on iterative
convergence guarantees and contraction mechanisms for \textit{c}-rectified
flow.

Another neighboring line improves the coupling or path geometry without
directly proving convergence to the globally optimal transport plan.
Minibatch optimal transport and multisample flow-matching methods use nontrivial
couplings to obtain straighter flows, lower transport cost, and lower training
variance \cite{tong2024improving,pooladian2023multisample}.
These methods are useful for constructing improved training couplings, but
they typically rely on batch-level approximations and do not by themselves
identify the population limit as the optimal coupling. Direct neural optimal transport solvers
take a different route by parameterizing Brenier potentials
\cite{makkuva2020optimal}, while plug-in methods recover the map from estimated
marginals \cite{manole2024plugin}. Our focus is complementary: we analyze
whether rectified flow itself is a convergent algorithm for classical optimal
transport.

The statistical estimation of optimal transport has been studied from several
complementary perspectives. A foundational line of work analyzes convergence
of empirical measures in Wasserstein distance. Nonasymptotic rates under
moment conditions were established by \cite{fournier2015rate}, and sharp
asymptotic and finite-sample rates were further developed by
\cite{weed2019sharp}. These results clarify the dimension dependence of
empirical optimal transport estimation and show that the empirical plug-in estimator can
exhibit the curse of dimensionality. In high-dimensional regimes, the
empirical Wasserstein distance typically converges much more slowly than
parametric rates unless additional structure or smoothing is
imposed. A related inferential literature studies optimal transport costs as statistical
functionals, including discrete settings \cite{sommerfeld2018inference} and
continuous settings \cite{delbarrio2019clt}. These works show that transport
costs can be used for statistical inference, but their limiting behavior is
more delicate than that of smooth integral functionals. Our statistical
estimation results use Wasserstein control of the marginals as an input to
transport-map estimation, rather than studying inference for the optimal transport cost
itself.

Minimax and plug-in estimation form another important part of the statistical
optimal transport literature. For density estimation under Wasserstein loss,
\cite{nilesweed2022minimax} show that minimax rates over smooth density
classes depend crucially on whether the density is bounded away from zero. For
optimal transport maps, \cite{manole2024plugin} study plug-in estimators
under smoothness and regularity assumptions on the Brenier map and establish
minimax optimal rates. Those results also give central limit theorems for
plug-in estimators of \(W_2^2(\mu,\nu)\) in sufficiently smooth regimes.
The work of \cite{divol2025general} develops rates for optimal transport map estimation in
general function spaces under assumptions expressed through metric entropy,
thereby moving beyond classical H\"older-type smoothness classes. These
results show that optimal transport maps can be estimated at nontrivial rates under suitable
structural assumptions on the underlying measures and maps. Our plug-in
analysis is complementary: it combines marginal estimation with the
\textit{c}-rectified-flow convergence theory developed in this paper.

Additional related work on flow-based distribution estimation,
Schr\"{o}dinger bridges, score-based sampling, and entropic regularization is
collected in Appendix~\ref{app:additional-related-work}.

\subsection{Organization}
The remainder of the paper is organized as follows. Section~\ref{sec:related-work}
reviews related work on rectified flow and statistical optimal transport, with
additional adjacent literature collected in
Appendix~\ref{app:additional-related-work}. Section~\ref{sec:background} introduces the
background on optimal transport, rectified flow, and \textit{c}-rectified flow.
Section~\ref{sec:rf-failure} gives the Gaussian counterexample showing that
ordinary rectified flow can stabilize at a non-optimal coupling, and Section~\ref{sec:rf-failure-general} generalizes the result to general log-concave target distributions.
Section~\ref{sec:crect-qualitative-convergence} proves qualitative convergence
of iterative \textit{c}-rectified flow. Section~\ref{sec:gaussian-convergence-speed}
studies convergence speed for Gaussian couplings, and
Section~\ref{sec:general-contraction-criterion} develops a general contraction
criterion for the quadratic cost. Section~\ref{sec:beyond-quadratic-contraction}
extends the contraction analysis beyond quadratic loss.
Section~\ref{sec:statest} constructs the score-based marginal estimator and
derives its Wasserstein risk. Finally, Section~\ref{sec:optimal-estimation}
combines marginal estimation with \textit{c}-rectified flow to obtain a
rate-optimal estimator of the optimal transport map.

\subsection{Notations}
Let $\mathbb{R}^d$ denote the $d$-dimensional Euclidean space. For a set $A$, let $\mathcal{P}(A)$ denote the set of probability measures supported on $A$. For an integer $m$ and two distributions $P$ and $Q$, let $W_m(P,Q)$ denote the $m$-Wasserstein distance between $P$ and $Q$. For a random variable $X$, let $\mathcal{L}(X)$ denote its law. Let $\|\cdot\|_2$ and $\langle \cdot,\cdot \rangle$ denote the Euclidean norm and inner product, respectively, and write $|\cdot|=\|\cdot\|_2$ for brevity. If $P \in \mathcal{P}(A)$ and $T:A\to B$, then $T_\#P$ denotes the pushforward of $P$ under $T$. For a probability measure $P$, let $\mathrm{spt}(P)$ denote its support. For distributions $P$ and $Q$, we say that a random vector $(X,Y)$ is a coupling of $P$ and $Q$ if $\mathcal{L}(X)=P$ and $\mathcal{L}(Y)=Q$. We use $P_n \xrightarrow{w} P$ to denote the weak convergence of a sequence of probability measures $P_n$ to $P$.

For two sequences $\{a_n\}$ and $\{b_n\}$, we write $a_n = O(b_n)$ if $a_n \le C b_n$ for a universal positive constant $C$ and all sufficiently large $n$. For a scalar function $T:\mathbb{R}^d\to\mathbb{R}$, let $\nabla T$ denote the gradient and $\nabla\cdot T$ the divergence. For a vector-valued map $T=(T_1,\dots,T_d)^\top:\mathbb{R}^d\to\mathbb{R}^d$, let $JT=(\nabla T_1,\dots,\nabla T_d)$. For a probability density $\rho$ on $\mathbb{R}^d$ and a function $f:\mathbb{R}^d\to\mathbb{R}$, write $\rho(f)=\int_{\mathbb{R}^d}\rho(x)f(x)\,\diff x$.

For a closed set $\Omega\subset\mathbb{R}^d$ with nonempty interior, let $\|\cdot\|_{C^\alpha(\Omega)}$ denote the H\"older norm on $\Omega$. For a vector-valued map $T=(T_1,\dots,T_d)^\top$, define $\|T\|_{C^\alpha(\Omega)}=\sum_{i=1}^d\|T_i\|_{C^\alpha(\Omega)}$. On space-time domains $\Omega_T = \Omega \times [0,T]$, we write $f \in C^{\alpha, \beta}(\Omega_T)$ to denote functions that are $\alpha$-H\"older continuous in space and $\beta$-H\"older continuous in time for $\alpha, \beta \in (0,1)$, and $C^{k+\alpha, l+\beta}(\Omega_T)$ for the corresponding spaces of higher H\"older-continuous derivatives for every $k,l \in {\mathbb N}$. Finally, for a Euclidean space $\mathbb{R}^d$, we write $h\in C_c^k(\mathbb{R}^d)$ if $h:\mathbb{R}^d\to\mathbb{R}$ is compactly supported and continuously differentiable up to order $k$.

Let ${\mathcal O}(d)$ be the $d$\-dimensional orthogonal group, defined by  
${\mathcal O}(d):=\{U\in \mathbb{R}^{d\times d}: U^\top U=UU^\top=I_d\}$.
For a matrix $A\in \R^{d\times d}$, we denote $\mathrm{Spec}(A)$ to be the set of all eigenvalues of $A$. Specifically, $\mathrm{Spec}(A) = \{\lambda \in \mathbb{C} \mid \det(A - \lambda I_d) = 0\}$. Denote $\operatorname{Re}$ to be the real part of complex numbers; for a set $S \subset \mathbb{C}$, $\operatorname{Re}(S) = \{\operatorname{Re} s : s\in S\}$, and we define $\dist(\mathrm{Spec}(A), \mathrm{Spec}(B)) = \inf \{|\lambda - \mu|: \lambda \in \mathrm{Spec}(A), \mu \in \mathrm{Spec}(B)\}$.

\section{Background}
\label{sec:background}

\subsection{Optimal Transport}
For $P,Q\in \mP_2(\R^d)$, let
\[
\Pi(P,Q)
:=
\left\{
\pi\in\mP(\R^d\times\R^d)
\,:
\pi(A\times \R^d)=P(A),\ 
\pi(\R^d\times A)=Q(A)
\right\}
\]
be the set of couplings between $P$ and $Q$. The quadratic optimal transport cost is
\begin{equation}\label{p:1}
\frac{1}{2}W_2^2(P,Q)
:=
\inf_{\pi\in\Pi(P,Q)}
\int_{\R^d\times\R^d}
\frac{1}{2}\|x-y\|^2
\,\diff\pi(x,y).
\end{equation}
We call any minimizer an optimal coupling. When a coupling is induced by a map,
we use the Monge formulation
\begin{equation}\label{p:2}
T^*
\in
\argmin_{T:\,T_\#P=Q}
\int_{\R^d}
\frac{1}{2}\|x-T(x)\|^2
\,\diff P(x).
\end{equation}
Under the usual absolute-continuity assumption on the source, Brenier's theorem
identifies the unique optimal map as $T^*=\nabla\varphi_0$ for a convex
potential $\varphi_0$, and the optimal coupling is $(\id,T^*)_\#P$.

For a general displacement cost $c(y-x)$, write
\begin{equation}\label{eq:Kantorovich-general-cost}
    W_c(P,Q)
    =
    \inf_{\pi\in\Pi(P,Q)}
    \int_{\R^d\times\R^d} c(y-x)\,\diff\pi(x,y)
    =
    \sup_{\varphi,\psi}
    \left\{
    \int \varphi\,\diff P+
    \int \psi\,\diff Q:
    \varphi(x)+\psi(y)\le c(y-x)
    \right\}.
\end{equation}
When the dual optimum is attained by $(\varphi,\psi)$ and the $c$-optimal map
is denoted by $T^*$, we use the slack function
\[
D(x,y):=c(y-x)-\psi(y)-\varphi(x).
\]
Then $D(x,y)\ge0$ and $D(x,T^*(x))=0$ for $P$-a.e. $x$; in particular,
$T^*(x)$ is a minimizer of $D(x,\cdot)$ and
$\nabla_yD(x,T^*(x))=0$ whenever the derivatives exist.

\subsection{Statistical Property for Optimal Transport}
The statistical problem considered later is to estimate the marginals and then
transport between the estimated distributions. Given samples from unknown
marginals $P$ and $Q$, we will construct estimators $\widehat P_n$ and
$\widehat Q_n$ and measure their quality in Wasserstein distance. These
marginal errors are then propagated to optimal-transport map error through
stability estimates. Detailed background on empirical OT rates and minimax OT
map estimation is deferred to Appendix~\ref{app:background-details}.

\subsection{Rectified Flow}
Rectified flow starts from a coupling $(X_0,X_1)$ of $P$ and $Q$ and the linear
interpolation $X_t=(1-t)X_0+tX_1$. It fits the conditional velocity field by
\begin{equation}\label{eq:v}
\min_v
\int_0^1
\E\|X_1-X_0-v_t(X_t)\|^2
\diff t,
\end{equation}
whose minimizer is
    $v_t^{\vv X}(z)=\E[X_1-X_0\mid X_t=z]$
for $P_{X_t}$-a.e. $z$. The rectified flow associated with $\vv X$ is the ODE
\begin{equation}\label{eq:rect}
\diff Z_t=v_t^{\vv X}(Z_t)\diff t,
\qquad Z_0\sim P.
\end{equation}
A process is rectifiable when this velocity exists, is locally bounded, and the
ODE is well posed. We write $(Z_0,Z_1)=\rect((X_0,X_1))$ for the endpoint
coupling produced by one rectification step.

\subsection{\textit{c}-Rectified Flow}
Unfortunately, rectified flow is not tied to a prescribed transport cost. The cost-dependent
variant studied in this paper replaces arbitrary velocity fields by the
cost-compatible gradient field generated by the convex conjugate of $c$.
Throughout, $c:\R^d\to\R$ is differentiable and strictly convex, $c^*$ denotes
its convex conjugate, and
\[
m_c(x,y)=c(x)+c^*(y)-\langle x,y\rangle
=b_c(x;\nabla c^*(y)),
\]
where $b_c$ is the Bregman divergence of $c$.

For a time-differentiable process $\vv X$, define
\[
L_{\vv X,c}(f)
:=
\int_0^1
\E\left[m_c(\dot X_s,\nabla f_s(X_s))\right]\diff s.
\]
The $c$-rectified flow is the process $\vv Z=\crect(\vv X)$ solving
\begin{equation}\label{evo:crect}
\diff Z_s
=
\nabla c^*(\nabla f^{\vv X,c}_s(Z_s))\diff s,
\qquad s\in[0,1],
\qquad Z_0=X_0,
\end{equation}
where $f^{\vv X,c}$ minimizes $L_{\vv X,c}$. A process is $c$-rectifiable if the
velocity $v^{\vv X}$ exists, the minimizer $f^{\vv X,c}$ exists and is locally
bounded, and \eqref{evo:crect} admits a unique solution.

Given an endpoint coupling $(X_0,X_1)$, we apply this construction to its
linear interpolation and write $(Z_0,Z_1)=\crect(X_0,X_1)$. Iterating this map
gives the coupling sequence used throughout the paper.
\begin{algorithm*}[th]
\caption{$c$-rectified Coupling}
\label{algo:crect}
\begin{algorithmic}[1]
\State Initialize $(Z_0^{(0)},Z_1^{(0)})$ such that $\mL(Z_0^{(0)})=P$, $\mL(Z_1^{(0)})=Q$.
\For{$k=1,2,\ldots,K$}
\State $(Z^{(k)}_0,Z^{(k)}_1)=\crect(Z^{(k-1)}_0,Z^{(k-1)}_1)$.
\EndFor{}
\State \Return $(Z^{(K)}_0,Z^{(K)}_1)$.
\end{algorithmic}
\end{algorithm*}

Additional background facts and earlier results used in this section are
collected in Appendix~\ref{app:background-details}.

With the basic objects fixed, we turn to the behavior of the
iteration of the (\(c\)-)rectified flow.  The first question is qualitative: if the marginals are known and
the \(c\)-rectified update is applied repeatedly, does the resulting coupling
approach the optimal transport coupling?

\subsection{Convergence of \textit{c}-Rectified Flow}
\label{sec:crect-qualitative-convergence}
% Recall that 
% \eqref{eq:1K} claims that there exist $k^* \in [K]$ such that
% $L_{\vv Z^{(k^*)},c}(f^{\vv Z^{(k^*)},c}) = O(1/K)$.

This subsection establishes the basic convergence guarantee for \textit{c}-rectified flow. The main point is that the additional gradient structure imposed by the \textit{c}-rectified objective rules out non-optimal limiting couplings under suitable compactness and integrability assumptions. We will always focus on the quadratic cost case in the following parts unless specified, where the classical optimal transport structure is most transparent.
To be precise,
we specialize to the quadratic cost $c(x)=\tfrac{\|x\|^2}{2}$, and always assume $P,Q\in\mathcal P_2(\R^d)$, as well as $\supp P,\, \supp Q = \R^d$. We assume $P$ is absolutely continuous to guarantee the uniqueness of optimal coupling. In this case, 
for all $x, y\in\R^d$,
\bb
c^*(x) = c(x)=\frac{\|x\|^2}{2},\qquad
\nabla c^*(x) = \nabla c(x) = \id, \qquad
% \ee
% % where $\mI$ denote the identity function such that $\mI(x)=x$ for $x\in\R^d$.
% \bb
m_c(x,y)= \frac{\|x-y\|^2}{2}.
\ee 

Thus, in the quadratic case, the \textit{c}-rectified objective measures the squared deviation between the displacement \(Z_1-Z_0\) and a gradient velocity field along the straight interpolation. This should be compared with ordinary rectified flow, where the minimizing velocity is an arbitrary conditional expectation. The restriction to gradient fields is precisely the structure compatible with Benamou--Brenier's theorem, and it is this restriction that allows the iteration to select the optimal coupling; see the discussion in Appendix~\ref{app:background-details} for details.

The following theorem makes this intuition precise. 
% Under a local compactness assumption on the minimizing potentials and a uniform integrability condition controlling the induced velocity fields, every subsequential zero-loss limit is an optimal coupling. Since the quadratic transport cost decreases along the iteration, this subsequential statement upgrades to convergence of the transport cost, and uniqueness of the optimal coupling further gives weak convergence of the full sequence.
% We start from a coupling $(Z_0,Z_1)$ with joint distribution $\pi$ iteratively apply \textit{c}-rectified flow  $(Z^{(k)}_0,Z^{(k)}_1) =  \crect(Z^{(k-1)}_0,Z^{(k-1)}_1)$ as proposed in Algorithm~\ref{algo:crect}. Considering the $L^2$ loss $c(x,y) = \|x-y\|^2/2$, we show that $(Z^{(k)}_0,Z^{(k)}_1)$ converges to $(Z^\infty_0,Z^\infty_1)$, where $(Z^\infty_0,Z^\infty_1)$ is the optimal transport (coupling) of $\pi$. 

%check ask Harry about UI condition, connection part
\begin{theorem}[Convergence of \textit{c}-rectified flow]\label{thm:convergence-square} 
Let $(Z_0^{(0)},Z_1^{(0)})=(X_0,X_1)$ be an initial coupling of $P$ and $Q$, and $\{\vv Z^{(k)}=(Z^{(k)}_0, Z^{(k)}_1)\}_{k=1}^\infty$ be a sequence of couplings of $(P,Q)$ defined recursively by $(Z^{(k)}_0, Z^{(k)}_1) = \crect(Z^{(k-1)}_0, Z^{(k-1)}_1)$.
For each $k\ge 1$, 
define
$Z_t^{(k)}=tZ_1^{(k)}+(1-t)Z_0^{(k)}$, $t\in[0,1]$,
and let $f_\cdot^{\vv Z^{(k)}}(\cdot): (x,t)\mapsto f_t^{\vv Z^{(k)}}(x)$ be a minimizer of
\[
L_{\vv Z^{(k)}}(f)
=
\int_0^1 \E\left[
\left\|Z_1^{(k)}-Z_0^{(k)}-\nabla f_s\bigl(Z_s^{(k)}\bigr)\right\|^2
\right]\diff s.
\] 
Suppose moreover that the following hold.
\begin{itemize}
    \item There exists an $\alpha>0$, such that for every $R>0$, there exists a constant $C_R>0$, 
    for all $k \ge 1$, 
    \[
    \|f^{\vv Z^{(k)}}\|_{C^{2+\alpha,1+\alpha}(\overline{B(0,R)}\times[0,1])}\le C_R;
    \]

    \item There exist a constant $C>0$ and a nonnegative measurable function $g:\R^d\to[0,\infty)$ such that
    $\|\nabla f_t^{\vv Z^{(k)}}(x)\|\le Cg(x)$
    for all $(x,t)\in\R^d\times[0,1]$ and all $k\ge 1$, and  the family
    $\left\{
    \int_0^1 g\bigl(Z_t^{(k)}\bigr)^2\diff t
    \right\}_{k=1}^\infty$
    is uniformly integrable.
\end{itemize}
Denote $(Z_0^*,Z_1^*)$ as the unique optimal coupling of $(P,Q)$, then  
\bb \lim_{k \to \infty} \E \|Z^{(k)}_1-Z^{(k)}_0\|^2 = \E \|Z_0^*-Z_1^*\|^2=W_2^2(P,Q), \qquad  (Z^{(k)}_0,Z^{(k)}_1)\xrightarrow{w} (Z_0^*,Z_1^*). \ee 
\end{theorem}
The proof of Theorem~\ref{thm:convergence-square} is given in
Appendix~\ref{app:crect-qualitative-proof}.
The theorem identifies the population target of the iteration: under the stated
regularity assumptions, repeated \textit{c}-rectification drives the quadratic
transport cost down to \(W_2^2(P,Q)\) and the couplings converge weakly to the
optimal coupling. 
In Remark~\ref{rmk:ui}, we show that the second assumption can be satisfied when $g$ has some polynomial growth rate and $P$, $Q$ have sufficiently high moments.
% This qualitative result does not yet give a rate, nor does
% it explain why the gradient restriction is essential.  The next section
% addresses both points by contrasting ordinary rectified flow with the
% \textit{c}-rectified update.

\section{Computational Guarantee of \textit{c}-Rectified Flow}

The convergence theorem above describes the limiting population behavior of
\textit{c}-rectified flow.  We now study the quantitative behavior of both types of rectified flow.  We first
show that ordinary rectified flow can fail to achieve the optimal transport coupling even in centered Gaussian and more general log-concave models,
and then prove quantitative contraction toward optimal transport for \textit{c}-rectified
update.

\subsection{Failure of Rectified Flow: Gaussian Case Study}
\label{sec:rf-failure}
For centered Gaussian marginals, ordinary rectified flow recovers the optimal transport coupling if and only if the two covariance matrices commute. Indeed, commuting covariances yield the optimal map after one rectification step, as shown in \cite{hertrich2025relation}; when the covariances do not commute, we show below that the first step produces a non-optimal deterministic coupling that is fixed by all subsequent reflow steps.

\begin{proposition}[Jointly Gaussian Coupling Velocity]
\label{prop:gaussian-joint-velocity}
Let $(X_0,X_1)$ be a centered jointly Gaussian coupling in
$\R^d\times\R^d$, with
$X_0\sim\N(0,\Sigma_1)$,
$X_1\sim\N(0,\Sigma_2)$,
$\Sigma_{01}:=\Cov(X_0,X_1)$.
For each $t\in[0,1]$, define the straight coupling
$X_t = (1-t)X_0 + tX_1$ with
$\Sigma_t
:=
\Var(X_t)
=
(1-t)^2\Sigma_1+t^2\Sigma_2
+t(1-t)(\Sigma_{01}+\Sigma_{01}^{\top})$.
Then the velocity field
$v_t(x):=\E[X_1-X_0\mid X_t=x]$
is linear in $x$ and is given by
\begin{equation}\label{eq:vt-joint}
v_t(x)
=
\Bigl[
t(\Sigma_2-\Sigma_{01})
+(1-t)(\Sigma_{01}^{\top}-\Sigma_1)
\Bigr]\Sigma_t^{-1}x.
\end{equation}
\end{proposition}

\begin{lemma}[Endpoint of one-step rectified flow]\label{lem:rf-1}
Let $\Sigma_1,\Sigma_2 \in \R^{d\times d}$ be symmetric positive definite, 
also denote $\Sigma_t = (1-t)^2\Sigma_1 + t^2\Sigma_2$. Consider the linear ODE
\[
\partial_t Z_t
=
\bigl(t\Sigma_2 - (1-t)\Sigma_1\bigr)\Sigma_t^{-1} Z_t,
\qquad
Z_0 = z \in \R^d,
\]
Then this ODE has solution 
\[
Z_t
=
\Sigma_1^{1/2}
\Bigl((1-t)^2I + t^2\Sigma_1^{-1/2}\Sigma_2\Sigma_1^{-1/2}\Bigr)^{1/2}
\Sigma_1^{-1/2} z.
\]
In particular,
$Z_1 = T z$,
where
\[
T
:=
\Sigma_1^{1/2}
\Bigl(\Sigma_1^{-1/2}\Sigma_2\Sigma_1^{-1/2}\Bigr)^{1/2}
\Sigma_1^{-1/2}.
\]
\end{lemma}

\begin{lemma}[Rectified flow of a deterministic linear coupling]\label{lem:rf-2}
Let $X_0 \sim \N(0,\Sigma_1)$ and define
$X_1 = L X_0$
for an $L \in \R^{d\times d}$ such that $\operatorname{Re}(\Spec(L))>0$ and $X_1\sim \N(0,\Sigma_2)$.
Then the velocity
% \[
% v_t(x) = \E[X_1 - X_0 \mid X_t = x]
% \]
is given by
$v_t(x) = (L-I)\bigl((1-t)I+tL\bigr)^{-1}x$,
and the corresponding ODE
\[
\partial_t Z_t = (L-I)\bigl((1-t)I+tL\bigr)^{-1}Z_t
\]
has the explicit solution
$Z_t = \bigl((1-t)I+tL\bigr)Z_0$.
In particular, one reflow step preserves the coupling map $L$.
\end{lemma}

\begin{proposition}[Non-convergence to optimal transport]\label{prop:non-ot}
For rectified flow on the Gaussian measures $\N(0,\Sigma_1)$ and $\N(0,\Sigma_2)$, initialized with the independent coupling, the first step produces the deterministic linear coupling
\[
X_1 = T X_0,
\qquad
T
=
\Sigma_1^{1/2}
\Bigl(\Sigma_1^{-1/2}\Sigma_2\Sigma_1^{-1/2}\Bigr)^{1/2}
\Sigma_1^{-1/2}.
\]
Every subsequent step preserves this coupling.

Moreover, let
\[
T_{\mathrm{OT}}
:=
\Sigma_1^{-1/2}
\Bigl(\Sigma_1^{1/2}\Sigma_2\Sigma_1^{1/2}\Bigr)^{1/2}
\Sigma_1^{-1/2}
\]
be the Gaussian optimal transport map with respect to the quadratic loss. Then
\[
T = T_{\mathrm{OT}}
\quad\Longleftrightarrow\quad
\Sigma_1\Sigma_2 = \Sigma_2\Sigma_1.
\]
\end{proposition}

Consequently, if $\Sigma_1$ and $\Sigma_2$ do not commute, rectified flow stabilizes after one step at a suboptimal coupling and does not recover the optimal transport map. Fundamentally, the optimal transport map $T_{\rm OT}$ is in a gradient space, which, when linear, is represented by a symmetric matrix. According to the Benamou–Brenier dynamic formulation of optimal transport, the optimal velocity field must also be a gradient field, which is represented by a symmetric matrix in the Gaussian case. However, the linear velocity fields of ordinary rectified flow in Lemmas~\ref{lem:rf-1} and \ref{lem:rf-2} need not be symmetric.
The lack of such a symmetric structure leads to the failure of vanilla rectified flow.
% Fundamentally, the optimal transport map $T_{\rm OT}$ is a symmetric linear map, and according to the Benamou–Brenier formula, the optimal velocity field must be a gradient field. In the Gaussian case, the velocity fields we got from Lemmas~\ref{lem:rf-1} and \ref{lem:rf-2} are linear and not symmetric, so equivalently, they do not belong to the gradient subspace. The lack of such a geometric structure leads to the failure of vanilla rectified flow.
% Fundamentally, the success of the \textit{c}-rectified flow is rooted in the geometry of Wasserstein space, as characterized by the Benamou–Brenier dynamic formulation of optimal transport; see the discussion in Appendix~\ref{app:background-details} for details. Specifically, Appendix~\ref{app:gaussian-convergence-speed-proof} precisely formulates the geometry of the Bures-Wasserstein case and rigorously establishes the connection above; when it is the case that $\Sigma_1$ and $\Sigma_2$ do not commute, the resulting matrix dynamics onto the projected gradient subspace will be significantly nontrivial.

\begin{remark}[Rectified flow fixed point does not characterize optimality]\label{rmk:rectflow-non-optimal}
Let $P=Q=\N(0,I_2)$, let $X\sim P$, and let $R_\theta$ denote the rotation operator by an angle
$0<|\theta|<\pi$. It can be shown that the coupling $(X,R_\theta X)$
is a fixed point of rectified flow. If $\varepsilon$ is an independent Rademacher
variable, we have
$R_\theta X-X\stackrel{d}{=}R_{\varepsilon\theta}X-X$ since both random vectors have the same covariance matrices.
Hence the distinct couplings $(X,R_\theta X)$ and
$(X,R_{\varepsilon\theta}X)$ induce the same displacement law and therefore
the same value of $\E[c(X_1-X_0)]$ for every integrable displacement cost
$c$. Thus the rectified flow fixed coupling cannot be the unique optimizer of any such
cost; in particular, it is not optimal whenever the optimal coupling is
unique. 
\end{remark}

Proofs of Proposition~\ref{prop:gaussian-joint-velocity},
Lemma~\ref{lem:rf-1}, Lemma~\ref{lem:rf-2}, and
Proposition~\ref{prop:non-ot} are given in
Appendix~\ref{app:rf-failure-proof}.
This failure motivates the quantitative analysis in Section~\ref{sec:gaussian-convergence-speed}: in the same
Gaussian setting, the \textit{c}-rectified update contracts the excess
quadratic transport cost.

\subsection{Failure of Rectified Flow: Log-concave Case}\label{sec:rf-failure-general}
The failure mode of the rectified flow can be beyond Gaussian.
In this section, we give a sufficient condition for a one-step rectified flow fixed point from a
Gaussian source to a log-concave target, as long as the negative log-density of the target is sufficiently smooth and has a uniformly bounded Hessian with a moderate condition number.
We first state the result for the isotropic Gaussian case, and then briefly mention the anisotropic case as follows.

\begin{theorem}[Relative log-curvature implies one-step straightness]\label{thm:relative-log-curvature}
Let $P=\N(0,I_d)$ and let $Q$ have density proportional to $e^{-V}$, where
$V\in C^2(\R^d)$ and
\[
mI_d\preceq\nabla^2V\preceq MI_d,
\qquad 0<m\le M<\infty.
\]
If $X_0\sim P$ and $X_1\sim Q$ are independent, then
$(Z_0,Z_1)=\rect((X_0,X_1))$
is well-defined. If $M/m<e^{4\pi}$, then $(Z_0,Z_1)$ is rectifiable and
straight; in particular,
\[
\rect((Z_0,Z_1))=(Z_0,Z_1).
\]
\end{theorem}

\begin{corollary}[A general Gaussian source]\label{cor:general-gaussian-source}
Let $X_0\sim\N(b,\Sigma)$ with $\Sigma\succ0$ and let $X_1$ be independent
of $X_0$. Suppose $\Sigma^{-1/2}(X_1-b)$ has density proportional to
$e^{-\widetilde V}$, where
\[
mI_d\preceq\nabla^2\widetilde V\preceq MI_d,
\qquad M/m<e^{4\pi}.
\]
Then one exact rectified-flow step from $(X_0,X_1)$ is rectifiable and straight.
\end{corollary}

Proofs of Theorem~\ref{thm:relative-log-curvature} and
Corollary~\ref{cor:general-gaussian-source} are given in
Appendix~\ref{app:rf-failure-proof}. We didn't verify that the fixed point of the rectified flow in these cases is suboptimal, but it is highly impossible that the rectified flow recovers the optimal coupling in this way, analogous to Proposition~\ref{prop:non-ot}.

\subsection{Convergence Rates for \textit{c}-Rectified Flow with Gaussian Marginals}
\label{sec:gaussian-convergence-speed}
In this section, we quantify the convergence rate of the \textit{c}-rectified flow in the Gaussian transport setting, still under the quadratic cost. Theorem~\ref{thm:convergence-square} shows that the iterates approach the optimal coupling under regularity and integrability assumptions, but it does not provide an explicit rate. For Gaussian marginals, the dynamics admit a more concrete description: the velocity fields remain linear, the gradient projection can be computed through a Gaussian Helmholtz decomposition, and the decrease in transport cost can be related to the non-gradient component of the initial velocity.

The main result is a one-step contraction estimate for the excess quadratic transport cost. The contraction is driven by the rotational part of the transport map. More precisely, writing a linear transport as \(T=T^*U\), where \(T^*\) is the Gaussian optimal transport map and \(U\) is \(\Sigma_1\)-orthogonal, the deviation from optimality is encoded by \(U-I\). A spectral gap excluding eigenvalues of \(U\) near \(-1\) ensures that the divergence-free component of the velocity controls this deviation, and hence that one step of \textit{c}-rectified flow removes a fixed fraction of the excess cost.

\subsubsection{One-step Contraction}

Let $\Sigma_1,\Sigma_2\in\R^{d\times d}$ be symmetric positive definite and
$P=\N(0,\Sigma_1)$, $Q=\N(0,\Sigma_2)$.

\begin{theorem}\label{thm:gaussian-contraction}
Let $X_0\sim P$, $X_1 = TX_0$, where $T$ is a (linear) transport map from $P$ to $Q$. Assume $(X_0, X_1)$ is \textit{c}-rectifiable. Then there exists a $\Sigma_1$-orthogonal matrix $U$ such that $T=T^* U$. Suppose $\dist(\Spec (U), -1) \ge \gamma > 0$, denote $(Z_0,Z_1)=\crect((X_0,X_1))$. Then
\[
    \E\|Z_1-Z_0\|^2 - W_2^2(P,Q) \le (1-\gamma^2 h) \left( \E\|X_1-X_0\|^2 - W_2^2(P,Q) \right).
    \]
The constant $h\in (0, 1/4)$ depends only on the spectra of \(\Sigma_1\) and \(\Sigma_2\).
\end{theorem}

\begin{remark}[Finite-step non-recovery]\label{rmk:finite-step-crect}
In dimension two, the first \textit{c}-rectified update from the independent
coupling recovers the Gaussian optimal transport map if and only if
$\Sigma_1$ and $\Sigma_2$ commute.  If they do not commute, the iterative outcome of the \textit{c}-rectified coupling is always \textit{c}-rectifiable, and no finite number
of \textit{c}-rectified updates recovers the optimal map; see
Proposition~\ref{prop:2d-finite-step} in
Appendix~\ref{sec:two-dimensional-gaussian-crect}.  In dimensions $d\ge3$,
finite-step recovery is nongeneric: for almost surely every pair of
positive definite covariance matrices, no finite iterate of the \textit{c}-rectified flow equals the optimal
map; see Proposition~\ref{prop:general-d-finite-step} in
Appendix~\ref{sec:general-dimensional-gaussian-crect}.
\end{remark}

\begin{remark}
When the spectral gap in Theorem~\ref{thm:gaussian-contraction} fails, the ODE \eqref{evo:crect} of rectified flow may not have a solution. To see this, we will enhance the following example that has been studied previously
\cite{liu2022rectified,mena2025statistical}: let $P=Q=\N(0,I_d)$ and set
$X_1=-X_0$ for $X_0\sim P$, so that $U=-I_d$. The interpolation and its
velocity are
\[
X_t=(1-2t)X_0,
\qquad
v_t(x)=-\frac{2}{1-2t}x \quad (t\ne\tfrac12).
\]
At $t=1/2$, all trajectories collapse at the origin, and the velocity field
becomes singular. The state at the collapse no longer determines a unique
continuation, so the coupling is not \textit{c}-rectifiable. 
Thus an
eigenvalue at minus one represents a genuine half-turn obstruction, and the
spectral gap quantifies the distance from this degeneracy.
\end{remark}

The proof of Theorem~\ref{thm:gaussian-contraction} is given in
Appendix~\ref{app:gaussian-convergence-speed-proof}.
The contraction follows from three comparisons. First, the spectral gap
ensures that the divergence-free projection of $T-I$ controls
$\gamma^2\|T-T^*\|_P^2$. Next, we have a variational inequality that converts this
projected component into a decrease in quadratic cost. Finally, some transport stability bounds the excess transport cost by a constant multiple of
$\|T-T^*\|_P^2$. Together, these estimates contribute to the factor
$1-\gamma^2h$ in Theorem~\ref{thm:gaussian-contraction}.

\subsubsection{Iterating the One-step Contraction}
By
Theorem~\ref{thm:convergence-square}, the Gaussian coupling iterates converge
to the unique optimal coupling, which forces their rotational factors to
converge to $I$. They are therefore uniformly separated from $-1$ after
finitely many steps. Combining this eventual gap with the finitely many
earlier steps gives a common contraction factor, and iterating the one-step
bound yields geometric decay of the excess cost.

\begin{theorem}[Exponential convergence of iterative \textit{c}-rectified flow]
\label{thm:exp-conv-crect}
Let $(Z_0^{(0)},Z_1^{(0)})=(X_0,X_1)$ be any jointly Gaussian initial coupling of
$P=\N(0,\Sigma_1)$,  $Q=\N(0,\Sigma_2)$, 
and for each $k\ge 0$, let
\[
(Z_0^{(k+1)},Z_1^{(k+1)})
:=\crect(Z_0^{(k)},Z_1^{(k)}).
\]
Assume that at every step, $(Z_0^{(k)},Z_1^{(k)})$ is rectifiable.
Then there exists a constant $h>0$ that depends only on %${\rm Spec}(U)$
$(Z_0^{(0)},Z_1^{(0)})$, such that the excess transport cost decays geometrically:
\[
\E\left\|Z_1^{(k)}-Z_0^{(k)}\right\|^2
- W_2^2(P,Q)
\le
% (1-c)^{k-1}
% \Big(
% \E\|X_1-X_0\|^2
% - W_2^2(P,Q)
% \Big) \le
e^{-h(k-1)}
\Big(
\E\|X_1-X_0\|^2
- W_2^2(P,Q)
\Big).
\]
% Equivalently, to obtain
% \[
% \E\left\|Z_1^{(k)}-Z_0^{(k)}\right\|^2
% - W_2^2(P,Q)
% \le \varepsilon,
% \]
% it suffices that
% \[
% k\ge \frac{1}{c}\log\left(
% \frac{
% \E\|X_1-X_0\|^2-W_2^2(P,Q)
% }{\varepsilon}
% \right) + 1.
% \]
\end{theorem}

The proof of Theorem~\ref{thm:exp-conv-crect} is also given in
Appendix~\ref{app:gaussian-convergence-speed-proof}.

The Gaussian calculation gives a fully explicit contraction mechanism.  We next
abstract the part of the argument that does not rely on Gaussian distribution: a
one-step decrease follows when the non-optimal component of the current map has
a stable projection onto divergence-free directions.

\subsection{A General Contraction Criterion for \textit{c}-Rectified Flow}
\label{sec:general-contraction-criterion}

Let \(P\) and \(Q\) be probability measures on \(\mathbb R^d\), and assume that
\(P\) is absolutely continuous, with density denoted by \(p\). Let
\(T:\mathbb R^d\to\mathbb R^d\) be a transport map such that \(T_\#P=Q\), and
let \(T^*\) denote the unique optimal transport map from \(P\) to \(Q\) with
respect to the quadratic cost.

Let \(X_0\sim P\) and set \(X_1=T(X_0)\). We study the one-step
\textit{c}-rectified coupling given by $(Z_0, Z_1) = \crect((X_0, T(X_0)))$. 
The goal of this section is to give a general criterion under which this step
contracts the excess quadratic transport cost.

Consider the linear interpolation
$T_t = tT + (1-t)\id$ for all $t\in [0,1]$.
We first record a simple local-in-time invertibility estimate, which allows us
to define the Eulerian velocity field along this deterministic interpolation.
\begin{proposition}[Invertibility of $T_t$ for small $t$]\label{prop:Tt-invertible}
Assume that $T$ is $L$-Lipschitz, 
then for every
$0\le t<{(1+L)^{-1}}$,
$T_t$
is injective and bi-Lipschitz onto its image.
In particular, $T_t^{-1}$ is well-defined on $T_t(\R^d)$ and $(1-(1+L)t)^{-1}$-Lipschitz.
\end{proposition}
The proof of Proposition~\ref{prop:Tt-invertible} is given in
Appendix~\ref{app:general-contraction-criterion-proof}.

By Proposition~\ref{prop:Tt-invertible}, for every $t<(2+2L)^{-1}$ the inverse map $T_t^{-1}$ is well-defined on $T_t(\R^d)$, and hence in particular on the support of $P_t=(T_t)_\#P$. Therefore we may define the velocity field
\[
v_t=(T-\id)\circ T_t^{-1}
\qquad\text{on }T_t(\R^d).
\]
This definition aligns with the marginal flow velocity $v_t(x) = \E[X_1 - X_0 | X_t = x]$ because the coupling is deterministic.

Let $p_t$ denote the probability density of the interpolant $X_t = T_t(X)$, and let $p$ denote the density of $P$. We define the subspace of divergence-free vector fields with respect to $p_t$ as 
\bbb
\label{eq:div-free-St-def}
{\mathcal S}_t := \left\{ h \in L^2(P_t; \mathbb{R}^d) : \nabla \cdot (p_t h) = 0  \text{ in the weak sense} \right\}.
\eee
Equivalently,
\(h\in {\mathcal S}_t\) if and only if
\[
-\int_{\mathbb R^d}\varphi\nabla \cdot (p_t h) \diff x=\int_{\mathbb R^d}\langle h,\nabla\varphi\rangle p_t\diff x=0
\qquad
\forall \varphi\in C_c^\infty(\mathbb R^d).
\]
We also let
\bbb
\label{eq:grad-Gt-def}
{\mathcal G}_t:=\overline{\{\nabla \varphi:\varphi\in C_c^\infty(\R^d)\}}^{\,L^2(P_t;\R^d)},
\eee
denote the closure of gradient fields in $L^2(P_t;\R^d)$. 
By this weak formulation, \({\mathcal S}_t={\mathcal G}_t^\perp\), 
hence
\(L^2(P_t;\mathbb R^d)={\mathcal G}_t\oplus {\mathcal S}_t\).
Our analysis proceeds in the weighted Hilbert space $L^2(P; \mathbb{R}^d)$ equipped with the inner product $\langle a, b \rangle_P = \int_{\mathbb{R}^d} a(x)^\top b(x) p(x)  \diff x$.

\subsubsection{Assumptions and Main Results}

We impose three assumptions. The first controls the interpolation and the
transported test fields. The second converts control of \(\|T-T^*\|_P^2\) into
control of the excess quadratic transport cost. The third is the structural
projection stability condition that drives the contraction estimate.

% \begin{assumption}\label{ass:rectifiable}
%     The coupling $(X_0, T(X_0))$ is c-rectifiable.
% \end{assumption}

\begin{assumption}[Lipschitz regularity]\label{ass:lipschitz}
    The map $T\in C^{1}$ is $L$-Lipschitz.
\end{assumption}

\begin{assumption}[Strong convexity of the Brenier potential]\label{ass:convexity}
    The optimal transport map $T^*$ is generated by a $\lambda$-strongly convex Brenier potential $\phi$, namely $T^* = \nabla \phi_0$ and $\nabla^2 \phi_0 \succeq \lambda I_d$.
\end{assumption}

 The key structural condition is
the following projection stability assumption: 
\begin{assumption}[Projection stability]\label{ass:projection_bound}
    There exists a constant $h > 0$ such that
    \[
    \|\proj_{{\mathcal S}_0}(T-T^*)\|_{P}^2 \ge h \|T-T^*\|_{P}^2,
    \]
    where ${\mathcal S}_0 = \{ u \in L^2(P;\R^d) : \nabla \cdot (p u) = 0 \text{  weakly}\}$. 
\end{assumption}

In the Gaussian setting, Assumption~\ref{ass:projection_bound} reduces to the
spectral separation condition $\dist(\Spec(U),-1)\ge\gamma>0$, as shown in
Appendix~\ref{app:gaussian-convergence-speed-proof}; it therefore provides a
natural extension of the Gaussian condition.  The projection bound is also
locally valid near $T^*$: if $T_\varepsilon=T^*\circ S_\varepsilon$ is
induced by a
$P$-measure-preserving map $S_\varepsilon$ that is close to the identity, with some regularity conditions on $\phi_0$, then for sufficiently
small $|\varepsilon|$, Assumption~\ref{ass:projection_bound} holds.  A precise statement and proof are given in
Proposition~\ref{prop:local-projection-stability} in
Appendix~\ref{app:general-contraction-criterion-proof}.  

Under these
assumptions, one-step of \textit{c}-rectified flow removes a definite fraction
of this deviation.

\begin{theorem}[One-step decrease of quadratic cost]
\label{thm:main_bound}
Let \(X_0\sim P\), \(X_1=T(X_0)\), and suppose that the coupling
\((X_0,X_1)\) is \textit{c}-rectifiable. Let
$(Z_0,Z_1)=\crect(X_0,X_1)$.
Then under Assumptions~\ref{ass:lipschitz}--\ref{ass:projection_bound},
\[
    \E \left\|T(X_0)-X_0\right\|^2
    -
    \E \left\|Z_1-Z_0\right\|^2
    \ge
    \frac{h}{8(L+1)}
    \|T-T^*\|_P^2.
\]
\end{theorem}

\begin{corollary}[One-step contraction of excess cost]\label{cor:one-step-contraction-general-L2}
Let \(X_0\sim P\), \(X_1=T(X_0)\), and suppose that the coupling
\((X_0,X_1)\) is \textit{c}-rectifiable. Let
$(Z_0,Z_1)=\crect(X_0,X_1)$.
Assume Assumptions~\ref{ass:lipschitz}--\ref{ass:projection_bound}. Then
\[
    \E \left\|Z_1-Z_0\right\|^2-W_2^2(P,Q)
    \le
    \left(1-\frac{\lambda h}{8(L+1)}\right)
    \left(
    \E \left\|T(X_0)-X_0\right\|^2-W_2^2(P,Q)
    \right).
\]
\end{corollary}

The proof of Theorem~\ref{thm:main_bound} and Corollary~\ref{cor:one-step-contraction-general-L2} are given in
Appendix~\ref{app:general-contraction-criterion-proof}.
Thus, under local regularity and projection stability, a single
\textit{c}-rectified step decreases the excess quadratic transport cost by a
fixed multiplicative factor.

\subsubsection{Iterating the One-step Contraction}

The one-step contraction can be iterated provided that the hypotheses of the
one-step assumption hold uniformly along the sequence of \textit{c}-rectified
couplings. We make this uniformity explicit in the following statement.

\begin{theorem}[Exponential convergence of iterative \textit{c}-rectified flow]
\label{thm:iterated-general-contraction}
Let % \(X_0\sim P\), % let \(T_0:\mathbb R^d\to\mathbb R^d\) be a transport map
% from \(P\) to \(Q\), and set
    $(Z_0^{(0)},Z_1^{(0)})$ be a initial coupling of $P$ and $Q$.
For each \(k\ge0\), define the iterates by
\[
    (Z_0^{(k+1)},Z_1^{(k+1)})
    :=
    \crect(Z_0^{(k)},Z_1^{(k)}).
\]
Assume that for every \(k\ge 1\), the coupling
\((Z_0^{(k)},Z_1^{(k)})\) is \textit{c}-rectifiable and is induced by a transport map
\(T_k\), namely
\[
    Z_0^{(k)}\sim P,
    \qquad
    Z_1^{(k)}=T_k(Z_0^{(k)}),
    \qquad
    (T_k)_\#P=Q.
\]
Assume further that Assumptions~\ref{ass:lipschitz}--\ref{ass:projection_bound}
hold for every \(T_k\) with the same constants \(L,h\). Then, 
% with
% \[
%     \theta:=\frac{\lambda c}{8(L+1)},
% \]
the excess quadratic transport cost satisfies
% \[
%     \E \left\|Z_1^{(k)}-Z_0^{(k)}\right\|^2
%     -
%     W_2^2(P,Q)
%     \le
%     \left(1-\frac{\lambda c}{8(L+1)}\right)^{k-1}
%     \left(
%     \E \left\|T_0(X_0)-X_0\right\|^2
%     -
%     W_2^2(P,Q)
%     \right).
% \]
% In particular,
\[
    \E \left\|Z_1^{(k)}-Z_0^{(k)}\right\|^2
    -
    W_2^2(P,Q)
    \le
    \exp{\left(-\frac{\lambda h}{8(L+1)} (k-1)\right)}
    \left(
    \E \left\|T_0(X_0)-X_0\right\|^2
    -
    W_2^2(P,Q)
    \right).
\]
% Equivalently, to guarantee
% \[
%     \E \left\|Z_1^{(k)}-Z_0^{(k)}\right\|^2
%     -
%     W_2^2(P,Q)
%     \le \varepsilon,
% \]
% it suffices that
% \[
%     k
%     \ge
%     \frac{8(L+1)}{\lambda c}
%     \log\left(
%     \frac{
%     \E \left\|T_0(X_0)-X_0\right\|^2-W_2^2(P,Q)
%     }{\varepsilon}
%     \right).
% \]
\end{theorem}
The proof of Theorem~\ref{thm:iterated-general-contraction} is given in
Appendix~\ref{app:general-contraction-criterion-proof}.

\begin{remark}
If the initial coupling is not induced by a transport map satisfying the uniform
assumptions, the same argument applies from the first index \(k_0\) at which the
iterate \((Z_0^{(k_0)},Z_1^{(k_0)})\) enters this admissible class. This is
expected to occur after sufficiently many iterations, since the
convergence result implies that the \textit{c}-rectified flow iterates approach the optimal transport
coupling, for which the required structure is satisfied.
In that case, for every \(k\ge k_0\),
    $\mathcal E_k
    \le
    \exp\left(-\frac{\lambda h(k-k_0)}{8(L+1)}\right)\mathcal E_{k_0}$.
Thus, the only additional issue for arbitrary initial couplings is to justify
that the iterates enter, and remain in, the class of map-induced couplings for
which the one-step contraction criterion applies.
\end{remark}

The results in this section assume access to the population
marginals and the associated population update.  We next supply the statistical
input needed to estimate the marginals from samples, and then 
connect to the estimation of the optimal transport map.

\section{Statistical Guarantee of \textit{c}-Rectified Flow}

This section combines marginal distribution estimation with the preceding population convergence theory to establish statistical guarantees for optimal transport estimation using \textit{c}-rectified flow.  The marginals \(P\) and \(Q\) are now unknown
and observed through independent samples.  Our approach proceeds from
score estimation to marginal estimation, and finally to optimal transport
estimation through iterative \textit{c}-rectification.

Following the score-based diffusion framework of
\cite{dou2024optimal}, we estimate the scores optimally and use reverse-time diffusion to obtain marginal estimators \(\widehat P_n\) and
\(\widehat Q_n\) which can be shown to be rate optimal under the Wasserstein distance. 

Finally, let \(T^*\) and \(\widetilde T_n^*\) denote the optimal transport maps
for \((P,Q)\) and their estimators \((\widetilde P_n,\widetilde Q_n)\), respectively.
Then, a stability result of optimal transport (e.g.~in  \cite{balakrishnan2025stability}) gives
$\E\|\widetilde T_n^*-T^*\|_{L^2(P)}^2
\lesssim
\E W_2^2(P,\widetilde P_n)
+
\E W_2^2(Q,\widetilde Q_n)$, which can be rate-optimal for the optimal transport map estimation.

\subsection{Score Estimation}
\label{sec:statest}

This subsection develops the statistical input used in the optimal-transport
map estimation result of Section~\ref{sec:optimal-estimation}.  We study score
estimation for a \(d\)-dimensional Gaussian-ratio model and derive matching
upper and lower bounds across the low-, intermediate-, and large-noise
regimes.  Projecting the resulting estimator onto a class with suitable slope bounds, we can get a stable reverse-time diffusion and
an explicit \(W_2\) risk bound in every fixed dimension.

\subsubsection{Setup and Main Statistical Results}

Fix an integer \(d\ge1\), treated as constant throughout, and constants
\(0<C_1<1<C_2<\infty\), \(L<\infty\), and \(\alpha>1\).  We assume that \(L\) is
large enough that the constant ratio \(r\equiv1\) is an interior point of the
smoothness constraint below.  Let \(\varphi_d\) be the standard Gaussian
density on \(\R^d\), let \(\varphi_t^{(d)}\) be the density of
\(N(0,tI_d)\), and define
\[
 \mathcal F_{\alpha,d}
 =
 \left\{
 f=r\varphi_d:
 \int_{\R^d}r\varphi_d=1,\quad
 C_1\le r\le C_2,\quad
 \|r\|_{C^\alpha(\R^d)}\le L
 \right\}.
\]
Here \(C^\alpha(\R^d)\) denotes the usual isotropic H\"older class
\(C^{m_\alpha,\alpha-m_\alpha}(\R^d)\), where
\(m_\alpha:=\lceil\alpha\rceil-1\).  Thus, when \(\alpha\) is an integer,
\(C^\alpha\) is interpreted as \(C^{\alpha-1,1}\).
Write
\[
\|g\|_{L^2(q)}^2:=\int_{\R^d}\|g(x)\|^2q(x)\diff x,
\qquad
\mathcal D_n:=\sigma(X_1,\ldots,X_n).
\]
For \(t\ge0\), define
\[
p_t=f*\varphi_t^{(d)},
\qquad
\gamma_t=\varphi_d*\varphi_t^{(d)}=\varphi_{1+t}^{(d)},
\qquad
R_t=\frac{p_t}{\gamma_t}.
\]
For \(t>0\), let \(s_t=\nabla\log p_t\).  Then \(C_1\le R_t\le C_2\) and
\[
s_t(x)=-\frac{x}{1+t}+\frac{\nabla R_t(x)}{R_t(x)}.
\]
Throughout this section, \(a\lesssim b\) and \(a\gtrsim b\) mean that the
corresponding inequality holds up to a constant depending only on
\((C_1,C_2,L,\alpha,d)\). We write \(a\asymp b\) when both relations hold. Denote $h_n=n^{-1/(2\alpha+d)}$, and
the multivariate score-risk envelope 
\begin{equation}\label{eq:multi-score-envelope}
 \mathfrak r_{n,d}(t)
 =
 \begin{cases}
 \displaystyle
 n^{-\frac{2(\alpha-1)}{2\alpha+d}}
 \wedge
 \frac{1\vee \log_+ (nt^{\alpha+d/2})^{d/2}}{nt^{1+d/2}}
 \wedge
 \frac1{nt^{d+1}},
 &0<t\le1,\\[4mm]
 \displaystyle
 \frac1{nt^2},
 &1<t\le n.
 \end{cases}
\end{equation}

\begin{theorem}[Multivariate score minimax risk]
\label{thm:multi-score-minimax}
For all sufficiently large \(n\) and every \(0<t\le n\),
\[
 \inf_{\hat s_t}
 \sup_{f\in\mathcal F_{\alpha,d}}
 \E_f\int_{\R^d}
 \|\hat s_t(x)-s_t(x)\|^2p_t(x)\diff x
 \asymp
 \mathfrak r_{n,d}(t),
\]
where the infimum is over all estimators based on
\(X_1,\ldots,X_n\overset{\rm i.i.d.}{\sim}f\).  Moreover, there exists an aggregate estimator \(\hat s_t^{\rm agg}\) such that the lower bound is
attained.
\end{theorem}

The proof of Theorem~\ref{thm:multi-score-minimax} is given in
Appendix~\ref{app:statest-proof}.  We next define the projection that turns
the aggregate estimator into a stable reverse drift.  Put
\(\bar\alpha=\alpha\wedge2\).  With \(K_0\) as in
Lemma~\ref{lem:multi-slope-projection}, let
\[
 \lambda_d(t)
 =
 K_0t^{\bar\alpha/2-1}\1_{\{t\le1\}}
 +
 \left\{\frac{K_0}{t(1+t)}-\frac1t\right\}\1_{\{t>1\}},
\]
\[
 \underline\lambda_d(t)
 =
 K_0t^{\bar\alpha/2-1}\1_{\{t\le1\}}
 +t^{-1}\1_{\{t>1\}},
\]
and define the closed convex set
\[
 \mathcal C_{t,d}
 =
 \left\{
 b=\nabla V\in L^2(\gamma_t;\R^d):
 -\underline\lambda_d(t)I_d
 \preceq\nabla^2V
 \preceq\lambda_d(t)I_d
 \text{ weakly}
 \right\}.
\]
Lemma~\ref{lem:multi-slope-projection} shows that
\(s_t\in\mathcal C_{t,d}\).  Consequently, Hilbert projection and the
uniform equivalence \(p_t\asymp\gamma_t\) give
\[
 \sup_{f\in\mathcal F_{\alpha,d}}
 \E_f\int_{\R^d}\|\widetilde s_t-s_t\|^2p_t
 \lesssim\mathfrak r_{n,d}(t),
 \qquad
 \widetilde s_t
 =\operatorname{Proj}^{L^2(\gamma_t;\R^d)}_{\mathcal C_{t,d}}
 \hat s_t^{\rm agg}.
\]
The resulting marginal estimator is summarized in
Algorithm~\ref{algo:projected-reverse-sampler}.

\begin{algorithm}[th]
\caption{Projected reverse sampler for \(\widehat\mu_n\)}
\label{algo:projected-reverse-sampler}
\begin{algorithmic}[1]
\State \textbf{Input:} \(X_1,\ldots,X_n\overset{\rm i.i.d.}{\sim}f\in
\mathcal F_{\alpha,d}\).
\State Construct \(\hat s_t^{\rm agg}\) from the three noise regimes.
\State Project the aggregate score:
\[
\widetilde s_t
=\operatorname{Proj}^{L^2(\gamma_t;\R^d)}_{\mathcal C_{t,d}}
\hat s_t^{\rm agg},
\qquad 0<t\le n.
\]
\State Sample \(\widehat Y_0\sim\gamma_n\) and, with a \(d\)-dimensional
Brownian motion \(B\), solve the reverse SDE
\[
\diff\widehat Y_u=\widetilde s_{n-u}(\widehat Y_u)\diff u+\diff B_u,
\qquad 0\le u\le n.
\]
\State \textbf{Output:} \(\widehat\mu_n=\mathcal L(\widehat Y_n\mid\mathcal D_n)\).
\end{algorithmic}
\end{algorithm}

\begin{theorem}[Multivariate \(W_2\) risk]\label{thm:w2-projected-sampler}
For all sufficiently large \(n\),
\[
 \sup_{f\in\mathcal F_{\alpha,d}}
 \left\{\E_f W_2^2(\widehat\mu_n,f)\right\}^{1/2}
 \lesssim
 \int_0^n\frac{\sqrt{\mathfrak r_{n,d}(t)}}{1+t}\diff t+n^{-1}
 \lesssim \epsilon_{n,d}:=
 \begin{cases}
 n^{-\frac12}\log\log n, & d=1,\\[1mm]
 n^{-\frac12}(\log n)^{\frac32}, & d=2,\\[1mm]
 n^{-\frac{\alpha+1}{2\alpha+d}}, & d\ge3 .
 \end{cases}
\]
Here we use \(\epsilon_{n,d}\) denotes the resulting global Wasserstein
rate after integration along the Gaussian interpolation.
\end{theorem}

The proof of Theorem~\ref{thm:w2-projected-sampler} is given in
Appendix~\ref{app:statest-proof}.
Theorem~\ref{thm:w2-projected-sampler} converts the time-indexed score risk
into a global Wasserstein guarantee for one unknown marginal.  Applying the
same construction separately to the samples from \(P\) and \(Q\) gives some high-accuracy samplers
$\E W_2^2(P,\widehat P_n)\lesssim\epsilon_{n,d}^2$,
$\E W_2^2(Q,\widehat Q_n)\lesssim\epsilon_{n,d}^2$.
These are precisely the marginal perturbation bounds required by the
optimal-transport stability argument.  The next subsection shows that they
transfer, without a loss in the polynomial rate, to the limiting \textit{c}-rectified coupling between
the estimated marginals.

\subsection{Application: Optimal Estimation of Optimal Transport}
\label{sec:optimal-estimation}
This section combines marginal estimation with stability of optimal transport
maps under perturbations of the source and target measures.  We treat the
marginal \(W_2\) error as the statistical input and first derive a plug-in
risk bound for the quadratic-cost optimal transport map.  The convergence results for
\(c\)-rectified flow established earlier provide a constructive procedure for
approximating the corresponding plug-in target.  We then formulate an
analogous stability bound for general displacement costs in the next section.

We work with the class of \(\lambda\)-log-concave and
\(\Lambda\)-log-smooth probability
measures
\[
{\cal C}_{\lambda}^{\Lambda}:=\left\{P \in {\cal P}_2(\R^d): \diff P(x) = f(x)\diff x, \quad \lambda I_d \preceq -\nabla^2 \log f \preceq \Lambda I_d \right\},
\]
and let \(T^*\) denote the optimal transport map from \(P\) to \(Q\), where
\(P,Q\in {\cal C}_{\lambda}^{\Lambda}\). Assume that, from the samples
$X_1,\dots,X_n \overset{\mathrm{i.i.d.}}{\sim} P$, 
$Y_1,\dots,Y_n \overset{\mathrm{i.i.d.}}{\sim} Q$,
we construct estimators \(\widehat P_n\) and \(\widehat Q_n\) as in
Section~\ref{sec:statest}, satisfying
\[
\E W_2^2(P,\widehat P_n)\le \rho_n^2,\qquad
\E W_2^2(Q,\widehat Q_n)\le \rho_n^2 .
\]
The stability argument below only uses these two marginal risk bounds.  To
instantiate them with Algorithm~\ref{algo:projected-reverse-sampler} and the
explicit rate \(\epsilon_{n,d}\), we additionally assume that the densities of
\(P\) and \(Q\) belong to \(\mathcal F_{\alpha,d}\), with common model
constants.

The projection step below only enforces measures in
\({\cal C}_{\lambda}^{\Lambda}\), the class on which the regularity theory for
optimal maps is available.  It does not itself construct the optimal map; the
map is obtained afterward as the exact transport between the projected
marginals.
The estimator is summarized in Algorithm~\ref{algo:plugin-ot-estimator}.

\begin{algorithm}[th]
\caption{Plug-in estimator of the optimal transport map}
\label{algo:plugin-ot-estimator}
\begin{algorithmic}[1]
\State \textbf{Input:} samples \(X_1,\dots,X_n\overset{\mathrm{i.i.d.}}{\sim}P\) and \(Y_1,\dots,Y_n\overset{\mathrm{i.i.d.}}{\sim}Q\).
\State Construct marginal estimators \(\widehat P_n,\widehat Q_n\) as in Algorithm~\ref{algo:projected-reverse-sampler}.
\State Project the marginal estimators onto the structured class:
\[
\widetilde P_n \in \arg\min_{\mu \in {\cal C}_{\lambda}^{\Lambda}} W_2(\mu,\widehat P_n),
\qquad
\widetilde Q_n \in \arg\min_{\nu \in {\cal C}_{\lambda}^{\Lambda}} W_2(\nu,\widehat Q_n).
\]
\State Compute the exact optimal transport map from \(\widetilde P_n\) to \(\widetilde Q_n\), equivalently the limiting map obtained by Algorithm~\ref{algo:crect}.
\State \textbf{Output:} the plug-in map \(\widetilde T^*_n\) from \(\widetilde P_n\) to \(\widetilde Q_n\).
\end{algorithmic}
\end{algorithm}

The output \(\widetilde T^*_n\) is therefore an exact population plug-in
target.  The deterministic contraction estimates above describe how
\textit{c}-rectified flow approximates this target by iteration.

We use the following stability estimate for optimal transport maps under
squared Euclidean cost, which has been studied in
\cite{balakrishnan2025stability}.

\begin{theorem}[Stability of smooth Brenier maps; Theorem 3 in \cite{balakrishnan2025stability}]
\label{thm:stability-smooth-ot}
Let $T^*=\nabla\varphi_0$ be the optimal transport map
from $P$ to $Q$. Let \(\Omega\subset\R^d\) be a convex set containing the
supports of \(P\) and \(\widetilde P\). Assume that the Brenier potential
\(\varphi_0\) is \(m_\varphi\)-strongly convex and \(M_\varphi\)-smooth on
\(\Omega\), in the sense that
\(m_\varphi I \preceq \nabla^2 \varphi_0 \preceq M_\varphi I\).
Let $\widetilde\pi$ be an optimal
coupling of $\widetilde P$ and $\widetilde Q$. Then
\[
\frac{1}{M_\varphi}
\E_{(X,Y)\sim \widetilde\pi}\|Y-T^*(X)\|_2^2
\le
\frac{1}{m_\varphi}W_2^2(Q,\widetilde Q)
+M_\varphi W_2^2(P,\widetilde P)
+2W_2(P,\widetilde P)W_2(Q,\widetilde Q).
\]
In particular, if $\widetilde P$ is absolutely continuous and $\widetilde T^*$ denotes the
optimal transport map from $\widetilde P$ to $\widetilde Q$, then
\[
\E_{\widetilde P}\|\widetilde T^*-T^*\|_2^2
\lesssim_{m_\varphi,M_\varphi}
W_2^2(P,\widetilde P)+W_2^2(Q,\widetilde Q).
\]
\end{theorem}
Theorem~\ref{thm:stability-smooth-ot} is quoted from
\cite{balakrishnan2025stability}; the plug-in argument below is proved in
Appendix~\ref{app:optimal-estimation-proof}.

\begin{proposition}[Plug-in estimation of the Brenier map]
\label{prop:plugin-brenier-map-rate}
Assume that \(P,Q\in {\cal C}_{\lambda}^{\Lambda}\), and that there exist marginal
estimators satisfying
$\E W_2^2(P,\widehat P_n)\le \rho_n^2,
 \E W_2^2(Q,\widehat Q_n)\le \rho_n^2$.
Let
\(T^*=\nabla\varphi_0\) be the optimal transport map from \(P\) to \(Q\), and let
\(\widetilde T^*_n\) be the output of Algorithm~\ref{algo:plugin-ot-estimator}.
Then there exists a constant \(C'=C'(\lambda,\Lambda)>0\) such that
\[
\E\|\widetilde T^*_n-T^*\|_{L^2(P)}^2
\le C'
\rho_n^2 .
\]
In particular, if the two marginal densities also belong to
\(\mathcal F_{\alpha,d}\), the samplers
\(\widehat P_n,\widehat Q_n\) from
Algorithm~\ref{algo:projected-reverse-sampler} yield a plug-in estimator with
squared \(L^2(P)\) risk
\(O(\epsilon_{n,d}^2)\).  This matches the marginal-estimation scale inherited
from Section~\ref{sec:statest}, up to the logarithmic factors appearing there
in dimensions \(d=1,2\).
\end{proposition}

The proof of Proposition~\ref{prop:plugin-brenier-map-rate} is given in
Appendix~\ref{app:optimal-estimation-proof}.
Together with the marginal rates from Section~\ref{sec:statest}, this gives a
rate-optimal plug-in estimator of the optimal transport map using
\textit{c}-rectified flow to reach the population transport between the
estimated marginals.

\section{General Cost}\label{sec:general-cost}
The preceding sections focused on the quadratic cost.  We now extend the same
chain of results to displacement costs \(c(y-x)\): qualitative convergence,
local contraction, and stability of the induced cost-optimal map under marginal
perturbations.  The quadratic squared distance is replaced by the dual conjugate 
and the \(c\)-transform geometry.

\subsection{Convergence of \textit{c}-Rectified Flow under General Cost}
\label{sec:general-cost-convergence}
The following theorem is the general-cost analogue of
Theorem~\ref{thm:convergence-square}. The squared error is replaced by the
Bregman-type discrepancy \(b_c\), and the velocity field is expressed through
the nonlinear map \(\nabla c^*\circ \nabla f_s\). The additional assumptions
below ensure that this nonlinear loss is stable under weak convergence of the
couplings and local convergence of the potentials.
\begin{theorem}[Convergence of \textit{c}-rectified flow]
\label{thm:convergence-general-cost}
Let \((Z_0^{(0)},Z_1^{(0)})=(X_0,X_1)\) be an initial coupling of \(P\) and \(Q\), and let
\(\{\vv Z^{(k)}=(Z_0^{(k)},Z_1^{(k)})\}_{k=1}^\infty\)
be a sequence of couplings of \((P,Q)\), where
\((Z_0^{(k+1)},Z_1^{(k+1)})=\crect(Z_0^{(k)},Z_1^{(k)})\).
For brevity, for each $k\ge 1$, 
define
$Z_t^{(k)}=tZ_1^{(k)}+(1-t)Z_0^{(k)}$, $t\in[0,1]$,
and
\[
g_k(x,s):=(\nabla c^*\circ \nabla f_s^{\vv Z^{(k)}})(x).
\]
Assume that for each \(k\), the function \(f_\cdot^{\vv Z^{(k)}}\) minimizes
\[
L_{\vv Z^{(k)}}(f)
:=
\int_0^1
\E \Bigl[
b_c\bigl(Z_1^{(k)}-Z_0^{(k)};(\nabla c^*\circ \nabla f_s)(Z_s^{(k)})\bigr)
\Bigr]\diff s.
\]

Suppose also there exists an $\alpha>0$, such that for every $R>0$, there exists a constant $C_R>0$, 
    for all $k \ge 1$, 
\[
\|f^{\vv Z^{(k)}}\|_{C^{2+\alpha,1+\alpha}(\overline{B(0,R)}\times[0,1])}\le C_R;
\]
Assume moreover that there exist a constant \(C>0\) and nonnegative measurable functions
\(\Phi,\Psi:\mathbb R^d\to[0,\infty)\)
such that
\[
b_c(x;y)\le C\bigl(1+\Phi(x)+\Psi(y)\bigr)
\qquad\text{for all }x,y\in\mathbb R^d,
\]
and that the three families
\[
\{\Phi(Z_1^{(k)}-Z_0^{(k)})\}_{k=1}^\infty,\qquad
\left\{\int_0^1 \Psi\bigl(g_k(Z_s^{(k)},s)\bigr)\diff s\right\}_{k=1}^\infty,\qquad
\{c(Z_1^{(k)}-Z_0^{(k)})\}_{k=1}^\infty
\]
are uniformly integrable.

Then
\[
\lim_{k\to\infty}\E [c(Z_1^{(k)}-Z_0^{(k)})]
=
\inf_{X\sim P,\;Y\sim Q}\E [c(Y-X)].
\]
Moreover, $ (Z^{(k)}_0,Z^{(k)}_1)\xrightarrow{w} (Z_0^*,Z_1^*)$, where $(Z_0^*,Z_1^*)$ is the unique optimal coupling of $(P,Q)$ under the cost function $c(x-y)$.
\end{theorem}
The proof of Theorem~\ref{thm:convergence-general-cost} is given in
Appendix~\ref{app:beyond-quadratic-contraction-proof}.
The conclusion says that \textit{c}-rectified flow still selects the
cost-optimal coupling once the quadratic geometry is replaced by the
corresponding \(c\)-geometry.  

\begin{remark}
A convenient sufficient condition for the uniform integrability assumptions in the \(c\)-loss theorem is the following. Suppose there exist \(p\ge 1\), \(\epsilon>0\), and a constant \(C>0\) such that 
for all \(x,y\in\mathbb R^d\),
\[
\max\{c(x),\Phi(x), \Psi(x)\}\le C(1+|x|^p),
\] and $|g_k(x,s)|\le C(1+|x|^q)$ for all $(x,s)\in\mathbb R^d\times[0,1],\ k\ge 1$.
If \(P,Q\) have finite \(p\max\{q,1\}+\epsilon\) moments, then the three families
\[
\{\Phi(Z_1^{(k)}-Z_0^{(k)})\}_{k\ge 1},\qquad
\left\{\int_0^1 \Psi\bigl(g_k(Z_s^{(k)},s)\bigr)\,\diff s\right\}_{k\ge 1},\qquad
\{c(Z_1^{(k)}-Z_0^{(k)})\}_{k\ge 1}
\]
are uniformly integrable.

Indeed, 
we have
\[
c(Z_1^{(k)}-Z_0^{(k)})+\Phi(Z_1^{(k)}-Z_0^{(k)})
\le C\bigl(1+|Z_1^{(k)}-Z_0^{(k)}|^p\bigr)
\le C\bigl(1+|Z_0^{(k)}|^p+|Z_1^{(k)}|^p\bigr).
\]
Next, 
since
$|Z_s^{(k)}|\le |Z_0^{(k)}|+|Z_1^{(k)}|$,
and by the growth assumption on \(g_k\) and H\"older's inequality,
\[
|g_k(Z_s^{(k)},s)|^p
\le C\bigl(1+|Z_s^{(k)}|^q\bigr)^p
\le C\bigl(1+|Z_0^{(k)}|^{pq}+|Z_1^{(k)}|^{pq}\bigr).
\]
Using the growth assumption on \(\Psi\), we obtain
\[
\Psi\bigl(g_k(Z_s^{(k)},s)\bigr)
\le C\bigl(1+|g_k(Z_s^{(k)},s)|^p\bigr)
% \le C\bigl(1+|Z_0^{(k)}|^{pq}+|Z_1^{(k)}|^{pq}\bigr)
\implies
\int_0^1 \Psi\bigl(g_k(Z_s^{(k)},s)\bigr)\,\diff s
\le
C\bigl(1+|Z_0^{(k)}|^{pq}+|Z_1^{(k)}|^{pq}\bigr).
\]

Since \(Z_0^{(k)}\sim P\) and \(Z_1^{(k)}\sim Q\) for all \(k\), and \(P,Q\) have finite \({p\max\{q,1\}}+\epsilon\) moments, the two families
\[
\{1+|Z_0^{(k)}|^{pq}+|Z_1^{(k)}|^{pq}\}_{k\ge 1}, \qquad \{1+|Z_0^{(k)}|^{p}+|Z_1^{(k)}|^{p}\}_{k\ge 1}
\]
are uniformly integrable. Hence each of the three families above is uniformly integrable.
\end{remark}

\subsection{Contraction Beyond Quadratic Loss}
\label{sec:beyond-quadratic-contraction}

Qualitative convergence does not by itself give a rate.  We therefore state a
local contraction criterion for a broader class of strictly convex cost
functions. Let $c: \mathbb{R}^d \to \mathbb{R}$ be a strictly convex cost function. Let \(T:\R^d\to\R^d\) be a transport map from \(P\) to \(Q\), and let
\(T^*\) denote the unique optimal transport map for the cost \(c(y-x)\). Such \(T^*\) exists under suitable regularity conditions (see, e.g., \cite{villani2021topics}).
% We consider the transport cost $\E[c(T(X)-X)]$.
As before, set \(X_0\sim P\) and \(X_1=T(X_0)\), then we consider the straight interpolation
\[
    T_t=(1-t)\id+tT,
    \qquad
    X_t=T_t(X_0),
    \qquad
    P_t=(T_t)_\#P.
\]
Whenever \(T_t\) is invertible on its image, the velocity along this
deterministic interpolation is
$v_t=(T-\id)\circ T_t^{-1}$.
We use the same notation for the $P_t$-divergence-free space and gradient space as defined before at \eqref{eq:div-free-St-def}.

We impose the following assumption on the cost function mainly by constraining its  Hessian.

\begin{assumption}\label{ass:cost_regularity}
    The cost function $c\in C^2(\R^d)$, and there exist constants $0 < \mu_c \le L_c < \infty$ such that for all $z \in \mathbb{R}^d$,
    $\mu_c I_d \preceq \nabla^2 c(z) \preceq L_c I_d$.
    
\end{assumption}

Under this assumption, we generalize the projection assumption. In the quadratic case, the relevant vector field was $T-\id$. For a general cost, the driving force is the gradient of the cost applied to the displacement.

\begin{assumption}\label{ass:stability_general}
    There exists a constant $h > 0$ such that:
    \[
    \|\proj_{{\mathcal S}_0}(\nabla c(T-\id))\|_{P}^2 \ge h \|\nabla c(T-\id)-\nabla c(T^*-\id)\|_{P}^2.
    \]
\end{assumption}

\noindent The next lemma compares the excess \(c\)-transport cost with the
force discrepancy \(\nabla c(T-\id)-\nabla c(T^*-\id)\).  This comparison is the
general-cost replacement for the quadratic cost-difference estimate used
earlier.

\begin{lemma}[Cost deficit controlled by force discrepancy]\label{lemma:cost_upper_bound}
Let \(c:\R^d\to\R\) satisfy Assumption \ref{ass:cost_regularity}. 
Let \(T^*\) be an optimal transport map from \(P\) to \(Q\) for the 
cost \(c(y-x)\), and let \(T\) be any transport map from \(P\) to \(Q\). Assume
that there exist optimal Kantorovich potentials  \((\varphi,\psi)\)  given by \eqref{eq:Kantorovich-general-cost}, such that
\(\psi\in C^2(\R^d)\) and 
    $\nabla^2 c(y-x)-\nabla^2\psi(y) \preceq L_y I_d$,
    $\forall \, (x,y)\in\R^d\times \R^d$.
% for some \(\mu_\psi\ge0\). 
Then
\[
    \E[c(T(X)-X)]-W_c(P,Q)
    \le
    \frac{L_y}{2\mu_c^2}
    \|\nabla c(T-\id)-\nabla c(T^*-\id)\|_{P}^2.
\]
\end{lemma}
The proof of Lemma~\ref{lemma:cost_upper_bound} is given in
Appendix~\ref{app:beyond-quadratic-contraction-proof}.

\begin{theorem}[One-step Contraction for General Cost]\label{thm:general-cost-one-step-decrease}
Let \(X_0\sim P\), \(X_1=T(X_0)\sim Q\). Suppose $(X_0,T(X_0))$ is \textit{c}-rectifiable, and let
$(Z_0,Z_1)=\crect(X_0,T(X_0))$.
Assume Assumptions \ref{ass:lipschitz}, \ref{ass:cost_regularity}, \ref{ass:stability_general}, together with the hypotheses of Lemma \ref{lemma:cost_upper_bound}, we have
    \[
    \E [c(T(X_0)-X_0)] - \E [c(Z_1-Z_0)] \ge \frac{h\mu_c^2}{8L_c L_y(L+1)} \left( \E[c(T(X_0)-X_0)] - W_c(P,Q) \right).
    \]
    Equivalently,
    \[
    \E[c(Z_1-Z_0)]-W_c(P,Q)
    \le
    \left(
    1-
    \frac{h\mu_c^2}{8L_cL_y(L+1)}
    \right)
    \left(
    \E[c(T(X_0)-X_0)]-W_c(P,Q)
    \right).
\]
\end{theorem}

The proof of Theorem~\ref{thm:general-cost-one-step-decrease} is given in
Appendix~\ref{app:beyond-quadratic-contraction-proof}.
Thus, under the general-cost projection stability condition, one
\textit{c}-rectified step removes a fixed fraction of the excess \(c\)-transport
cost.  The constants now depend on the curvature of \(c\) and on the local
curvature of the dual slack.

The one-step estimate can be iterated once the assumptions hold uniformly along
the sequence of \(c\)-rectified couplings.

\begin{theorem}[Exponential convergence under uniform general-cost stability]
\label{thm:general-cost-iterated-contraction}
Let \((Z_0^{(0)},Z_1^{(0)})\) be an initial coupling of \(P\) and \(Q\), and
define
    $(Z_0^{(k+1)},Z_1^{(k+1)})
    =
    \crect(Z_0^{(k)},Z_1^{(k)})$,
    $k\ge0$.
Assume that for every \(k\ge0\), the coupling
\((Z_0^{(k)},Z_1^{(k)})\) is \(c\)-rectifiable and induced by a transport map
\(T_k\), so that
    $Z_0^{(k)}\sim P$,
    $Z_1^{(k)}=T_k(Z_0^{(k)})$,
    $(T_k)_\#P=Q$.
Assume further that
Assumptions \ref{ass:lipschitz}, \ref{ass:cost_regularity}, \ref{ass:stability_general}, and in Lemma \ref{lemma:cost_upper_bound}
hold uniformly along \(T_k\), with the same constants
\(L,h\). Then
\[
    \E[c(Z_1^{(k)}-Z_0^{(k)})]-W_c(P,Q)
    \le
%     \left(
%     1-
%     \frac{\gamma\mu_c^2}{8L_c(M_c+\mu_\psi)(L+1)}
%     \right)^{k-1}
%     \left(
%     \E[c(Z_1^{(0)}-Z_0^{(0)})]-W_c(P,Q)
%     \right).
% \]
% Consequently,
% \[
%     \E[c(Z_1^{(k)}-Z_0^{(k)})]-W_c(P,Q)
%     \le
    \exp\left(
    -\frac{h\mu_c^2 (k-1)}{8L_cL_y(L+1)}
    \right)
    \left(
    \E[c(Z_1^{(0)}-Z_0^{(0)})]-W_c(P,Q)
    \right).
\]
\end{theorem}

The proof of Theorem~\ref{thm:general-cost-iterated-contraction} is given in
Appendix~\ref{app:beyond-quadratic-contraction-proof}.
This gives the algorithmic part of the general-cost theory.  To obtain a
statistical statement, it remains to control how perturbing the source and
target marginals changes the induced \(c\)-optimal map.

\subsection{Estimation for General Cost}
The final ingredient is a stability estimate for general displacement costs.
The assumptions below play the same role as strong convexity and smoothness of
the Brenier potential in the quadratic case.
\begin{theorem}[Stability bound for displacement costs]
\label{thm:displacement-cost-stability}
Suppose Assumption~\ref{ass:cost_regularity} holds, and let
\((\varphi,\psi)\) be an optimal Kantorovich dual pair for \((P,Q)\) under
cost \(c\), where \(T^*\) is the \(c\)-OT map from \(P\) to \(Q\).
Define the slack function
    $D(x,y):=c(y-x)-\varphi(x)-\psi(y)$.
Assume further that there exist constants \(\mu_y,L_x,L_y>0\) such that the
following Hessian bounds hold for all \((x,y)\):
\[
    \mu_y I\preceq \nabla^2_{yy}D(x,y)
    =
    \nabla^2 c(y-x)-\nabla^2\psi(y)\preceq L_yI,
\]
\[
    \nabla^2_{xx}D(x,y)=
    \nabla^2 c(y-x)-\nabla^2\varphi(x)\preceq L_xI,
\]
Finally, let \(\widetilde\pi\) be an optimal coupling between \(\widetilde P\) and
\(\widetilde Q\) for the cost \(c(y-x)\).
Then
\[
    \mu_y\int \frac{\|y-T^*(x)\|^2}{2}\,\diff\widetilde\pi(x,y)
    \le
    L_x W_2^2(P,\widetilde P)
    +
    L_y W_2^2(Q,\widetilde Q)
    +
    2L_c W_2(P,\widetilde P)W_2(Q,\widetilde Q).
\]
\end{theorem}

The proof of Theorem~\ref{thm:displacement-cost-stability} is given in
Appendix~\ref{app:optimal-estimation-proof}.
This bound transfers marginal Wasserstein perturbations into an \(L^2\)-type
control of the induced cost-optimal map.  It is the general-cost analogue of the
smooth Brenier-map stability used above.

Combining this stability estimate with the marginal estimators from
Section~\ref{sec:statest} gives the corresponding plug-in rate.

\begin{corollary}[Plug-in stability for displacement costs]
\label{cor:plugin-displacement-cost}
Let \(\widetilde T_n^c\) be the \(c\)-optimal map from \(\widetilde P_n\) to
\(\widetilde Q_n\), and let \(T^c\) be the \(c\)-optimal map from \(P\) to
\(Q\).  Suppose that the Kantorovich potentials associated with
\((\widetilde P_n,\widetilde Q_n)\) satisfy the assumptions of
Theorem~\ref{thm:displacement-cost-stability}, with constants independent of
\(n\). Then
$\E_P \|T^c(X)-\widetilde T_n^c(X)\|^2
\lesssim
\rho_n^2$.
\end{corollary}

The proof of Corollary~\ref{cor:plugin-displacement-cost} is given in
Appendix~\ref{app:optimal-estimation-proof}.
The main body has therefore established the full chain: \textit{c}-rectified
flow targets the correct optimal coupling, contracts under local stability
conditions, and, when combined with score-based marginal estimation, yields
rate-optimal estimation of the optimal transport map and its general-cost
analogue.

\begin{acks}[Acknowledgments]
The authors would like to thank Aram-Alexandre Pooladian for helpful discussions, and Shayan Hundrieser and Jiadong Liang for feedback on the manuscript. 
\end{acks}

\subsection*{Statement on AI use}
The authors were interacting with ChatGPT 5.6 Sol to assist in proving Theorem~\ref{thm:relative-log-curvature} and Corollary~\ref{cor:general-gaussian-source}, developing a proof idea for Proposition~\ref{prop:non-ot}, and verifying some statements in Remark~\ref{rmk:finite-step-crect}. 
\newpage

\begin{appendix}

\section{Additional Related Work}
\label{app:additional-related-work}

\subsection{Flow-based Estimation and Bridge Methods}
There is a broader theory of convergence rates for flow-based transport
estimators that is adjacent to our objective. For general flow matching,
near-minimax convergence guarantees under Wasserstein metrics are established
\cite{fukumizu2025flow,kunkel2025minimax}. For rectified flow,
\cite{sahoo2026sample} prove order-optimal sample complexity under smoothness
assumptions. These results concern statistical estimation of distributions or
transport dynamics from samples, whereas our main algorithmic results concern
the population rectification operator and its convergence to the optimal
transport coupling.

Schr\"{o}dinger bridges and entropic optimal transport provide regularized
counterparts of the transport problem. Classical references interpret
Schr\"{o}dinger bridges as entropy-regularized optimal transport on path space
\cite{leonard2014survey,dimarino2020optimal}. In the small-noise limit,
Schr\"{o}dinger potentials converge to Kantorovich potentials
\cite{nutz2022entropic}, and the associated entropic optimizers admit
quantitative large-deviation asymptotics with local exponential rates
\cite{bernton2022entropic}. Diffusion Schr\"{o}dinger bridge matching introduces
iterative Markovian fitting as a bridge-matching procedure
\cite{shi2023diffusion}; related iterative methods have also been studied
\cite{peluchetti2023diffusionbridge,kholkin2024diffusion}; and
\cite{silveri2025exponential} establish non-asymptotic exponential
convergence guarantees for iterative Markovian fitting. Our results instead
concern unregularized rectified-flow iteration and convergence to classical
optimal transport.

\subsection{Score-based Estimation and Sampling}
Classical score matching estimates unnormalized models through the score of the
log density \cite{hyvarinen2005estimation}, and denoising score matching
provides a practical way to estimate scores after Gaussian smoothing
\cite{vincent2011connection}. These ideas became central in score-based
generative modeling, where reverse-time diffusions use estimated scores to
sample from an unknown distribution \cite{song2021scorebased}.
\cite{dou2024optimal} study how optimal score-estimation
rates can be converted into optimal sampling guarantees, and the algorithmic
structure of their method is closely related to the projected reverse sampler used in our
statistical section. In our setting, we go beyond the density class with compact support and get a similar but different minimax-optimal score estimation rate, which leads to optimal sampling and optimal transport estimation based on \textit{c}-rectified flow. Relatedly, \cite{han2026one} use Wasserstein gradient flows to train
a one-step generator, whereas we study the convergence of repeated
\textit{c}-rectification to an optimal coupling between fixed
endpoint marginals.

\subsection{Entropic Statistical Optimal Transport}
Sinkhorn-type estimators are computationally efficient and statistically more
stable than unregularized empirical optimal transport. The sample complexity
of Sinkhorn divergences was studied by \cite{genevay2019sinkhorn}, while
\cite{mena2019entropic} obtained statistical bounds and central limit theorems
for entropic optimal transport. More recently, \cite{rigollet2025sample}
proved dimension-free parametric rates for several entropic optimal transport
quantities. In the semi-discrete setting, \cite{pooladian2023minimax} proved a
minimax-optimal dimension-independent \(n^{-1/2}\) rate for an entropic
estimator of discontinuous optimal transport maps. These estimators target
regularized quantities, so their statistical benefits must be balanced against
regularization bias.

The barycentric projection of the empirical entropic plan was analyzed by
\cite{pooladian2021entropic}, who obtained finite-sample guarantees for a
computationally tractable Sinkhorn-based map estimator. The role of
regularization bias was further studied by \cite{pooladian2022debiaser}, while
stability results for entropic Brenier maps in \cite{divol2025tight} quantify
how regularized maps vary with the target measure. This literature gives
the estimation theory for regularized plug-in transport map estimators.

\section{Additional Details for Section~\ref{sec:background}}
\label{app:background-details}

\subsection{Additional Optimal-transport Facts}
The Kantorovich problem can be regarded as an infinite-dimensional linear program, hence admits a dual formulation:
\bbb
\label{eq:kanto-dual}
\frac{1}{2}W_2^2(P,Q)
=
\sup_{(f,g)\in\mathcal F}
\left\{
\int f \diff P
+
\int g \diff Q
\right\},
\eee
where
\bb
\mathcal F
:=
\left\{
(f,g): f\in L^1(P),\ g\in L^1(Q),\
f(x)+g(y)\leq \frac{1}{2}\|x-y\|^2
\right\}.
\ee
The functions $(f,g)$ are called Kantorovich potentials. Under the finite second moment assumption, there exists a maximizing pair of Kantorovich potentials $(f_0,g_0)$ solving \eqref{eq:kanto-dual} (see, e.g., \cite{villani2009optimal}).

The following fundamental theorem shows that, under absolute continuity of the source measure, the Kantorovich and Monge formulations are unified by a unique optimal map.

\begin{theorem}[Brenier's theorem; see {\cite{knott1984optimal} and \cite{villani2009optimal}}]
Let $P\in\mP_2(\R^d)$ be absolutely continuous with respect to the Lebesgue measure, and let $Q\in\mP_2(\R^d)$. Then there exists a $P$-a.e. unique optimal transport map $T^*$ from $P$ to $Q$. Moreover, there exists a convex function $\varphi_0:\R^d\to\R$ such that
\bb
T^*=\nabla \varphi_0
\quad
P\text{-a.e.}
\ee
The unique optimal coupling in \eqref{p:1} is induced by the map
$\diff\pi_0(x,y)
=
\diff P(x)\cdot \delta_{T^*(x)}(y)$.
\end{theorem}

Thus, when the source distribution has a density, the optimal coupling is concentrated on the graph of a deterministic map. In particular, the solution of the Kantorovich problem \eqref{p:1} induces the solution of the Monge problem \eqref{p:2}.

The Brenier potential can also be obtained from the Kantorovich potentials. Let $(f_0,g_0)$ be optimal potentials for \eqref{eq:kanto-dual}, and define
\bbb
\label{eq:brenier-potentials}
\varphi_0(x)
:=
\frac{1}{2}\|x\|^2-f_0(x),
\qquad
\psi_0(y)
:=
\frac{1}{2}\|y\|^2-g_0(y).
\eee
Then $\varphi_0$ and $\psi_0$ are convex functions, and $\varphi_0$ is the Brenier potential whose gradient gives the optimal transport map from $P$ to $Q$:
$T^*=\nabla\varphi_0$,
$P$-a.e.
If $Q$ is also absolutely continuous, then the inverse optimal transport map from $Q$ to $P$ is given by $(T^*)^{-1}=\nabla\psi_0$,
$Q$-a.e.
Moreover, $\varphi_0$ and $\psi_0$ are convex conjugates of one another, that is, $\psi_0 = \varphi_0^*$.

The Kantorovich dual can be rewritten more concisely in terms of the Brenier potentials. Specifically, if $\varphi^*$ denotes the Legendre--Fenchel conjugate of $\varphi$,
then we have the semi-dual formulation
\bbb
\label{eq:dual}
\frac{1}{2}W_2^2(P,Q)
=
\frac{1}{2}\int \|x\|^2 \diff P(x)
+
\frac{1}{2}\int \|y\|^2 \diff Q(y)
-
\inf_{\varphi}
\left\{
\int \varphi \diff P
+
\int \varphi^* \diff Q
\right\},
\eee
where the infimum is taken over suitable functions $\varphi$ for which the two integrals are well-defined. Under the assumptions of Brenier's theorem, the minimizer of \eqref{eq:dual} is precisely the Brenier potential $\varphi_0$, and the optimal transport map is recovered as $T^*=\nabla\varphi_0$. This semi-dual perspective is frequently used for statistical estimation of optimal transport maps; for example, \cite{hutter2021minimax} study \eqref{eq:dual} over smooth function classes such as functions with bounded H\"older norms.

We also recall the dynamic formulation of optimal transport. Let $(\rho_t)_{t\in[0,1]}$ be a curve of probability measures on $\R^d$ such that $\rho_0=P$ and $\rho_1=Q$. If particles evolve according to the velocity field $v_t:\R^d\to\R^d$, equivalently, if $X_t\sim \rho_t$ follows the ODE
\bb
\diff X_t=v_t(X_t)\diff t,
\ee
then the evolution of their law is described by the continuity equation
\bb
\partial_t\rho_t+\nabla\cdot(\rho_t v_t)=0,
\ee
understood in the weak sense. The Benamou--Brenier formula gives a dynamic characterization of the $2$-Wasserstein distance, which states that the optimal coupling is related to the minimizer of the total kinetic energy along the path space:
\bbb
\label{eq:bb}
\frac{1}{2}W_2^2(P,Q)
=
\inf_{\rho_t,v_t}
\int_0^1
\int_{\R^d}
\frac{1}{2}\|v_t(x)\|^2
\,\diff\rho_t(x)\diff t = \inf_{\rho_t,v_t}
\int_0^1 \frac{1}{2}\E\|v_t(X_t)\|^2\diff t,
\eee
where the infimum is taken over all pairs $(\rho_t,v_t)$ satisfying the continuity equation and the boundary conditions $\rho_0=P$, $\rho_1=Q$; we refer to \cite{benamou2000computational,villani2009optimal} for details.

Under the assumptions of Brenier's theorem, let $T$ be the optimal transport map from $P$ to $Q$. Then the minimizing path in \eqref{eq:bb} is the constant-speed displacement interpolation
\bbb
\label{eq:displacement-interpolation}
X_t=(1-t)X_0+tT(X_0),
\qquad t\in[0,1], \qquad X_0\sim P.
\eee
Equivalently, the intermediate law is
$\rho_t = \big((1-t)\mathrm{Id}+tT\big)_\# P$ for all $t\in[0,1]$.
The curve $(\rho_t)_{t\in[0,1]}$ is a constant-speed geodesic in the Wasserstein space $(\mP_2(\R^d),W_2)$.

\subsection{Statistical Property for Optimal Transport}
We briefly recall the statistical estimation problem for optimal transport. 
In statistical estimation problems, $P$ and $Q$ are typically unknown and are observed through
finite given samples. For instance, given i.i.d. samples $X_1,\ldots,X_n\sim P$ and
$Y_1,\ldots,Y_m\sim Q$, one forms the empirical measures
$\widehat P_n=n^{-1}\sum_{i=1}^n\delta_{X_i}$ and
$\widehat Q_m=m^{-1}\sum_{j=1}^m\delta_{Y_j}$. Natural plug-in estimators of
transport costs are then $W_p(\widehat P_n,Q)$, in the one-sample case, or
$W_p(\widehat P_n,\widehat Q_m)$, in the two-sample case. A central question
is to quantify the error between these empirical quantities and their
population counterparts, for example
$|W_p(\widehat P_n,\widehat Q_m)-W_p(P,Q)|$ or
$|W_p^p(\widehat P_n,\widehat Q_m)-W_p^p(P,Q)|$.

The first major phenomenon in this literature is the curse of dimensionality.
For general distributions on $\mathbb R^d$, the empirical measure converges in
Wasserstein distance at dimension-dependent rates. Nonasymptotic bounds of this
type were obtained by \cite{fournier2015rate}, while sharp finite-sample and
asymptotic rates were further developed by \cite{weed2019sharp}. In a
representative compactly supported setting, the empirical Wasserstein error satisfies
dimension-dependent rates of the form
\[
    \E W_p^p(\widehat\mu_n,\mu)
    \lesssim
    \begin{cases}
        n^{-1/2}, & d<2p,\\
        n^{-1/2}\log n, & d=2p,\\
        n^{-p/d}, & d>2p,
    \end{cases}
\]
up to assumptions and constants depending on the support and moment conditions;
(see \cite{fournier2015rate,weed2019sharp}).
\cite{niles2022estimation} construct a case in which the minimax rate (of the estimation of $W_2(P,Q)$) is at least $C n^{-\frac{1}{d}}$ for large $d$ and a constant $C$.
Thus, when $d>2p$,
$W_p(\widehat\mu_n,\mu)$ typically behaves like $n^{-1/d}$, displaying the
classical curse of dimensionality for empirical optimal transport.
There is also a line of work that studies statistical estimation of optimal
transport maps under additional structural assumptions on the Brenier map. For
example, \cite{hutter2021minimax} consider optimal transport
problem with quadratic loss, 
and studies the minimax rate for estimating $T^*$. In contrast to the
distribution-free empirical Wasserstein theory, this setting imposes structural
regularity assumptions on the transport map and on the underlying measures. 
The smoothness of the
transport map is measured through $\|T^*\|_{C^\alpha}\le R$ for a fixed
constant $\alpha$ and $R$. The estimation error is measured in the loss $\|\hat T - T^*\|_{P}^2 = \E_P\|\hat T(X)-T^*(X)\|^2$,
where $\hat T$ is an estimator constructed from the samples. Under these
assumptions, \cite{hutter2021minimax} prove the minimax lower bound
\bb
\inf_{\hat T} \sup_{P,T^*}
\|\hat T - T^* \|_{P}^2
\gtrsim
n^{-\frac{2\alpha}{2\alpha-2+d}} \vee {n}^{-1},
\ee
where the supremum is taken over the admissible pairs $(P,T^*)$ satisfying the
above regularity conditions. They also construct an estimator $\hat T$ using an estimator constructed on a wavelet basis such
that
\bb
\sup_{P,T^*}
\|\hat T - T^* \|_{P}^2
\lesssim
n^{-\frac{2\alpha}{2\alpha-2+d}}(\log n)^2 \vee {n}^{-1}.
\ee
Thus, up to the logarithmic factor in the upper bound, the minimax rate is
$n^{-\frac{2\alpha}{2\alpha-2+d}} \vee n^{-1}$. The remaining logarithmic gap is
conjectured by \cite{hutter2021minimax} to be removable. These results show how
smoothness of the Brenier map can partially offset the statistical curse of
dimensionality: when $T^*$ is very smooth, namely when $\alpha$ is large, the
rate approaches the parametric order $n^{-1}$; however, when $\alpha$ is small
and $d$ is large, the rate remains slow and is close to the high-dimensional
empirical-OT behavior.

Another line of work studies the estimation of entropic optimal transport, where the objective is
\[
\inf_{\pi\in\Pi(P,Q)}
\left\{\int_{\R^d\times\R^d}
\|x-y\|^2
\,\diff\pi(x,y) + \eta \cdot \mathrm{KL}\left(\pi\,\|\,P\otimes Q\right) \right\}.
\]
We discuss this line of work in the related work
section and omit further details here.

\subsection{Additional Rectified-flow Facts}
The basic properties of rectification are the following. First, since $\vv X$ and $\vv Z$ have the same expected velocity field, they have the same time marginals; in particular, $\mL(Z_1)=\mL(X_1)$, and $\mL(Z_t)=\mL(X_t)$ for all $t\in [0,1]$. 
Second, for any convex function $c:\R^d\to\R$,
$\E[c(Z_1-Z_0)]\le \E[c(X_1-X_0)]$.
For a detailed proof of the above facts, see Theorems~3.3 and 3.5 in \cite{liu2022flow}.

Therefore, the rectified flow can be iterated. Starting from an initial coupling $(Z_0^{(0)},Z_1^{(0)})=(X_0,X_1)$, define recursively
\bb
(Z_0^{(k)},Z_1^{(k)})=\rect((Z_0^{(k-1)},Z_1^{(k-1)})),
\ee
Then $\mL(Z_0^{(k)})=\mL(X_0)=P$, $\mL(Z_1^{(k)})=\mL(X_1)=Q$, and  for every convex $c$, $\E[c(Z_1^{(k)}-Z_0^{(k)})]$ is non-increasing in \(k\).
In practice, one estimates \(v^{\vv X}\) from samples by solving \eqref{eq:v} over a parameterized class, solves the ODE \eqref{eq:rect} with the fitted field, and then uses the resulting pair \((Z_0^{(k)},Z_1^{(k)})\) to fit a transport map. This gives a learned transport from \(P\) to \(Q\), although the limiting coupling is not necessarily optimal for a prescribed transport cost.

If we view the rectified flow as an optimization algorithm on the path space of time-differentiable stochastic processes with path-wise $c$-transport cost, then one-step rectified flow will be a one-step steepest descent.
For a time-differentiable stochastic process $\vv Y$, 
define
\bbb
\label{eq:def_F_c}
F_c(\vv Y)=\int_0^1 \E\left[c(\dot Y_s)\right]\diff s.
\eee
\begin{lemma}[Lemma 3.3 in \cite{liu2022rectified}] \label{thm:rectdual}
For every convex functions $c\colon \R^d\to \R$, 
 \bbb \label{equ:pathopt}
 \vv Z = \rect(\vv X) \quad\Longleftrightarrow\quad \vv Z = 
 \arg\min_{\vv Y}  
 \left\{ F_c(\vv Y)  \quad\text{s.t.}\quad 
 v^{\vv Y}  = v^{\vv X} \right\}. 
 \eee  

\end{lemma}

\subsection{Additional \textit{c}-Rectified-flow Facts}\label{sec:appendix-c-rectflow-intro}

Intuitively, the distinction between rectified flow and \textit{c}-rectified flow can be
understood through the dynamic formulation of optimal transport.  The
Benamou--Brenier problem minimizes the kinetic action jointly over a curve of
marginals and a velocity field satisfying the continuity equation; see
\eqref{eq:bb}.  For the quadratic cost, its optimal velocity field is a
gradient field and, under suitable regularity, is generated by a potential
satisfying a Hamilton--Jacobi equation; for example, see \cite[Section 5.5]{liu2022rectified}. 

Ordinary rectified flow solves a different projection problem as proposed in \eqref{eq:def_F_c}.  Its minimizer is
the conditional velocity
$v_t(x)=\E[X_1-X_0\mid X_t=x]$,
which preserves the interpolation marginals but need not be a gradient.  It
may therefore contain a non-gradient component (such as rotational parts).
Consequently, a deterministic straight coupling may be a fixed
point of ordinary rectified flow without being the optimal coupling, as we have illustrated in Remark~\ref{rmk:rectflow-non-optimal}.

For the quadratic cost, \textit{c}-rectified flow instead performs the
least-squares projection onto gradient velocity fields.  This removes the
specific non-gradient component that ordinary rectified flow cannot detect
and makes the update compatible with the velocity structure in the
Benamou--Brenier formula \eqref{eq:bb}.  

Now we introduce more properties to understand the \textit{c}-rectfied flow better.

\subsubsection{Marginal-Preserving Vector Fields}
We first define the notion of marginal-preserving vector fields. Let \(\vv X=\{X_t:t\in[0,1]\}\) be a stochastic process.

\begin{definition}[Marginal-preserving vector field]
A vector field \(r:\R^d\times[0,1]\to\R^d\) is called \(\vv X\)-marginal-preserving if, for every \(t\in[0,1]\) and every \(h\in C_c^1(\R^d)\),
\[
    \int_0^t
    \E[\nabla h(X_s)^\top r_s(X_s)]
    \diff s
    =
    0.
\]
\end{definition}

This condition means that \(r\) contributes zero to the weak continuity equation along the marginal curve \(\mL(X_t)\). The following characterization is basic in \(c\)-rectified flow.

\begin{theorem}[Lemma 4.2 in \cite{liu2022rectified}]
\label{thm:mar}
Let \(\vv X\) and \(\vv Y\) be time-differentiable stochastic processes such that \(\mL(X_0)=\mL(Y_0)\). Assume that \(\vv X\) is rectifiable and that \(v^{\vv Y}\) exists and is locally bounded. Then
$\mL(X_t)=\mL(Y_t)$, for all $t\in[0,1]$,
if and only if \(v^{\vv X}-v^{\vv Y}\) is \(\vv Y\)-marginal-preserving.
\end{theorem}

\begin{theorem}[Theorem 5.2(ii) in \cite{liu2022rectified}]
\label{thm:crect-marginal-preservation}
For a $c$-rectifiable $\vv X$ and $c^* \in C^1(\R^d)$, let $\vv Z=\crect(\boldsymbol{X})$, we have $\mL(Z_t) = \mL(X_t)$ for $t \in [0,1]$.
\end{theorem}

\subsubsection{\textit{c}-Rectified Duality Formulation}
Back to the viewpoint of the rectified flow algorithm as an optimization on path space, A natural relaxation of \eqref{equ:pathopt}  would be  
 \bbb \label{equ:pathoptmlaw}
 \vv Z = 
 \arg\min_{\vv Y}  
 \left\{ F_c(\vv Y) , \quad \text{s.t.}\quad 
 \mL(Y_t)  = \mL(X_t),\, \, \forall t\in[0,1]\right\}, 
 \eee 

For each continuously differentiable $f:\R^d \times [0,1] \to \R^d$,  Fenchel's inequality and Theorem~\ref{thm:mar} yield
\bbb
 F_c(\vv X)-F_c(\vv Y)
&=\int_0^1 \E [c( \dot X_s)-c( \dot Y_s)] \diff s
\le \int_0^1 \E [c( \dot X_s)+c^*(\nabla f_s(Y_s))-\langle \nabla f_s(Y_s) , \dot Y_s \rangle] \diff s \label{ineq:dual} \\
&=\int_0^1 \E [c( \dot X_s)+c^*(\nabla f_s(Y_s))-\langle \nabla f_s(Y_s) , v^{\vv Y}_s(Y_s) \rangle] \diff s \notag\\
\label{eq:mar}
&=\int_0^1 \E [c( \dot X_s)+c^*(\nabla f_s(X_s))-\langle 
\nabla f_s(Y_s) , v^{\vv X}_s(Y_s) \rangle] \diff s =\int_0^1 \E [m_c(\dot X_s,\nabla f_s(X_s))] \diff s, 
\eee
where (\ref{ineq:dual}) is because $c(x)+c^*(y) \ge \langle x,y\rangle$ and (\ref{eq:mar}) is because $v^{\vv X}-v^{\vv Y}$ is $Y$-marginal-preserving as stated in Theorem~\ref{thm:mar}. Therefore, (\ref{ineq:dual}) takes equality if and only if 
$\dot Y_s = \nabla c^*(\nabla f_s(Y_s))$
for almost every $s$. 
This suggests the strong duality, as presented in \cite[Theorem~5.3]{liu2022rectified}, that
\bbb
\label{primal:crect}
\min_{f}\left\{ L_{\vv X, c}(f) := \int_0^1 \E [m_c(\dot X_s, \nabla f_s(X_s))] \diff s \right\} = \sup_{\vv Y}\Big\{F_c(\vv X)-F_c(\vv Y): \mL(X_t)=\mL(Y_t),\ \forall t\in[0,1]\Big\},
\eee

% Therefore, 
% \bb
% \min_f L_{\vv X,c}(f)\ge
% \sup_{\vv Y}\Big\{F_c(\vv X)-F_c(\vv Y): \mL(X_t)=\mL(Y_t),\ \forall t\in[0,1]\Big\}.
% \ee

% If $L_{\vv X,c}(f)$ attains its minimum at $f^{\vv X,c}$, then under suitable regularity the Euler--Lagrange equation implies that (\ref{ \label{evo:crect}}) generates a process $\vv Y$ satisfying $\mL(X_t)=\mL(Y_t)$ for all $t\in[0,1]$. In that case,
% \bbb
% \label{dual}
% L_{\vv X,c}(f^{\vv X,c})
% =
% \sup_{\vv Y}\Big\{F_c(\vv X)-F_c(\vv Y): \mL(X_t)=\mL(Y_t),\ \forall t\in[0,1]\Big\}.
% \eee

\begin{remark}
For the rectified flow, the analogous dual identity is
\bb
\inf_v \int_0^1 \E[b_c(\dot X_s;v_s(X_s))]\diff s
=
\sup_{\vv Y}\Big\{F_c(\vv X)-F_c(\vv Y): v^{\vv X}=v^{\vv Y}\Big\}.
\ee
\end{remark}

\subsubsection{Stationary Points of \textit{c}-Rectified Flow}
The next result identifies whether a stationary coupling of the \textit{c}-rectified flow
must be optimal. As observed by \cite[Section~4.1]{hertrich2025relation}, the
fixed-point characterization requires a connected-support condition.

\begin{theorem}[Theorem 5.6 in \cite{liu2022rectified}]
\label{thm:mini}
Let $c$ be convex with $c,c^* \in C^1(\R^d)$, and let
$\vv X=\{X_t=tX_1+(1-t)X_0:\ t\in[0,1]\}$
for a coupling $(X_0,X_1)$. Assume that $\vv X$ is $c$-rectifiable, that
$\operatorname{supp}\mL(X_t) = \R^d$ for every $t\in[0,1]$, and that
$f^{\vv X,c}\in C^{2,1}(\R^d\times[0,1])$ minimizes $L_{\vv X,c}$. Then the following are equivalent:
\[
(X_0,X_1)=\crect((X_0,X_1)) \,\,\Longleftrightarrow\,\, L_{\vv X,c}(f^{\vv X,c})=0 \,\,\Longleftrightarrow\,\, (X_0,X_1) \text{ is a \textit{c}-optimal coupling}.
\]
% \begin{itemize}
%     \item $(X_0,X_1)=\crect(X_0,X_1)$;
%     \item $L_{\vv X,c}(f^{\vv X,c})=0$;
%     \item $(X_0,X_1)$ is a \textit{c}-optimal coupling.
% \end{itemize}
\end{theorem}

\subsubsection{Decrease of transport cost}
Define
$S_c(\vv X)=F_c(\vv X)-\E[c(X_1-X_0)]$, where recall $F_c$ is defined at \eqref{eq:def_F_c}. By the convexity of $c$ and Jensen's inequality, for every time-differentiable process $\vv X$, we always have $S_c(\vv X) \ge 0$.
Also,
If $\vv X$ is the linear interpolation process between $X_0$ and $X_1$, then $\dot X_t=X_1-X_0$ for all $t\in [0,1]$, we have
$S_c(\vv X)=0$.

Now let $\vv Y=\crect(\boldsymbol X)$, so that $(Y_0,Y_1)=\crect(X_0,X_1)$. By the dual identity (\ref{primal:crect}),
\bbb
F_c(\vv X)-F_c(\vv Y)=L_{\vv X,c}(f^{\vv X,c}).
\eee
Therefore, we have the following inequality
\bbb
\label{eq:c-rect-variational}
\E[c(X_1-X_0)]-\E[c(Y_1-Y_0)]
&=F_c(\vv X)-F_c(\vv Y)+S_c(\vv Y) =L_{\vv X,c}(f^{\vv X,c})+S_c(\vv Y)\ge 0.
\eee
Hence $c$-rectification does not increase the transport cost. In particular, the sequence of couplings generated by Algorithm~\ref{algo:crect} has non-increasing transport cost.

Finally, by \eqref{eq:c-rect-variational}, for every $K\ge 1$,
\bb
\sum_{k=1}^K L_{\vv Z^{(k)},c}(f^{\vv Z^{(k)},c})&\le  \sum_{k=1}^K  \left(\E[c(Z^{(k-1)}_0-Z^{(k-1)}_1)] -\E[c(Z^{(k)}_0-Z^{(k)}_1)]\right) \\
&= \E[c(Z_1^{(0)}-Z_0^{(0)})] - \E[c(Z^{(K)}_0-Z^{(K)}_1)]
\le\E[c(Z_1^{(0)}-Z_0^{(0)})]
< +\infty.
\ee
Therefore, as stated in Corollary 5.7 in \cite{liu2022rectified},
\bbb
\label{eq:1K}
\min_{k\le K}L_{\vv Z^{(k)},c}(f^{\vv Z^{(k)},c}) =O(1/K).
\eee

\section{Proof of Section~\ref{sec:crect-qualitative-convergence}}
\label{app:crect-qualitative-proof}
\begin{lemma}[Continuity of integrated losses]
\label{lem:continuity-integrated-loss}
Let
$\vv Z^{(k)}=(Z_0^{(k)},Z_1^{(k)})
\xrightarrow{w}
\vv Z^*=(Z_0^*,Z_1^*)$
in $\R^d\times\R^d$, and define
$Z_t^{(k)}=(1-t)Z_0^{(k)}+tZ_1^{(k)}$,
$Z_t^*=(1-t)Z_0^*+tZ_1^*$.
Let $h_k,h_*:\R^d\times[0,1]\to\R^m$ be continuous functions such that
$h_k\to h_*$
locally uniformly on $\R^d\times[0,1]$, and let
$\ell:\R^d\times\R^m\to[0,\infty)$ be continuous. Define
\[
V_k
:=
\int_0^1
\ell\bigl(
Z_1^{(k)}-Z_0^{(k)};
h_k(Z_t^{(k)},t)
\bigr)\diff t
,\qquad
V_*
:=
\int_0^1
\ell\bigl(
Z_1^*-Z_0^*;
h_*(Z_t^*,t)
\bigr)\diff t.
\]
If $\{V_k\}_{k\ge1}$ is uniformly integrable, then
$\E[V_k]\longrightarrow \E[V_*]$.
\end{lemma}

\begin{proof}[Proof of Lemma~\ref{lem:continuity-integrated-loss}]
By the Skorokhod representation theorem, after replacing the random
variables by copies with the same laws on a common probability space, we
may assume that
\[
(Z_0^{(k)},Z_1^{(k)})
\longrightarrow
(Z_0^*,Z_1^*)
\qquad\text{almost surely}.
\]
Fix an event $\omega$ for which this convergence holds. Then
\[
\sup_{t\in[0,1]}
\bigl\|
Z_t^{(k)}(\omega)-Z_t^*(\omega)
\bigr\|
\le
\|Z_0^{(k)}(\omega)-Z_0^*(\omega)\|
+
\|Z_1^{(k)}(\omega)-Z_1^*(\omega)\|
\longrightarrow 0.
\]
In particular, there exists $R=R(\omega)<\infty$ such that
$Z_t^{(k)}(\omega),Z_t^*(\omega)\in\overline{B(0,R)}$
for every $t\in[0,1]$ and every $k$.

Using the local uniform convergence of $h_k$ and the uniform continuity
of $h_*$ on the compact set
$\overline{B(0,R)}\times[0,1]$, we obtain
\begin{align*}
&\sup_{t\in[0,1]}
\bigl\|
h_k(Z_t^{(k)}(\omega),t)
-
h_*(Z_t^*(\omega),t)
\bigr\|\\
&\quad\le
\sup_{(x,t)\in\overline{B(0,R)}\times[0,1]}
\|h_k(x,t)-h_*(x,t)\|+
\sup_{t\in[0,1]}
\bigl\|
h_*(Z_t^{(k)}(\omega),t)
-
h_*(Z_t^*(\omega),t)
\bigr\|
\longrightarrow 0.
\end{align*}
Moreover,
$Z_1^{(k)}(\omega)-Z_0^{(k)}(\omega)
\longrightarrow
Z_1^*(\omega)-Z_0^*(\omega)$.
By the continuity of $\ell$, which implies uniform
continuity on this compact set,
\[
\sup_{t\in[0,1]}
\left|
\ell\bigl(
Z_1^{(k)}-Z_0^{(k)};
h_k(Z_t^{(k)},t)
\bigr)
-
\ell\bigl(
Z_1^*-Z_0^*;
h_*(Z_t^*,t)
\bigr)
\right|
\longrightarrow 0
\]
almost surely. Consequently, we have almost surely that 
$V_k\longrightarrow V_*$.
Since $\{V_k\}_{k\ge1}$ is uniformly integrable, Vitali's theorem yields
$\E[V_k]\longrightarrow\E[V_*]$.
Because the expectations involved depend only on the laws of the random
variables, the conclusion also holds on the original probability spaces.
\end{proof}

\begin{proof}[Proof of Theorem~\ref{thm:convergence-square}]
Since $\E[\|Z^{(k)}_1-Z^{(k)}_0\|^2]$ is decreasing in $k$ by the cost-decrease inequality \eqref{eq:c-rect-variational}, applied with the quadratic cost, it is enough to show that there exists a subsequence $\{{k_j}\}_{j \in \mathbb N}$ such that $\lim_{j\to\infty}\E[\|Z^{(k_j)}_1-Z^{(k_j)}_0\|^2]=W_2^2(P,Q)$. 
By \eqref{eq:1K}, we already know that for every $K>0$, there exists $k\leq K$ such that
\bb
L_{\vv Z^{(k)}} (f^{\vv Z^{(k)}}) = O\left(\frac{1}{K}\right),
\ee
therefore, 
we may choose an increasing subsequence, still denoted by $k$, such that
$L_{\vv Z^{(k)}} (f^{\vv Z^{(k)}}) \to 0$ as $k \to \infty$. 

Fix an arbitrary $\epsilon>0$. Since $P$ and $Q$ are tight, there exist compact sets $A_\epsilon, B_\epsilon \in \R^d$ such that
$P(A_\epsilon^c) < \epsilon,Q(B_\epsilon^c) < \epsilon$.
Let $C_\epsilon = A_\epsilon \times B_\epsilon$, then for every $k \in \mathbb N$,
\bb
\bP((Z^{(k)}_0,Z^{(k)}_1) \in C_\epsilon^c) &\le \bP(Z^{(k)}_0 \in A_\epsilon^c) + \bP(Z^{(k)}_1 \in B_\epsilon^c) 
= P(A_\epsilon^c) + Q(B_\epsilon^c)
< 2\epsilon.
\ee
Hence $\{(Z_0^{(k)},Z_1^{(k)})\}_{k\in\mathbb N}$ is tight. By Prokhorov's Theorem, there exists a subsequence, still denoted by $\{(Z_0^{(k)},Z_1^{(k)})\}_{k\in\mathbb N}$, and a random vector $(Z_0^*,Z_1^*)$ such that 
as $ k\to\infty$,
\[
(Z_0^{(k)},Z_1^{(k)})\xrightarrow{w}(Z_0^*,Z_1^*).
\]
Since $Z_0^{(k)}\sim P$ and $Z_1^{(k)}\sim Q$ for all $k$, we also have $Z_0^*\sim P$ and $Z_1^*\sim Q$.

For each $m\in\mathbb N$, let
$K_m:=\overline{B(0,m)}\times[0,1]$.
By our assumption, for every $m$ there exists $C_m>0$ such that
\[
\sup_k \|f^{\vv Z^{(k)}}\|_{C^{2+\alpha,1+\alpha}(K_m)}\le C_m.
\]
Hence, for every multi-index $\beta$ with $|\beta|\le 2$, the family
$\{D_x^\beta f^{\vv Z^{(k)}}\}_{k\ge 1}$
is uniformly bounded and equicontinuous on $K_m$, and the same is true for
$\{\partial_t f^{\vv Z^{(k)}}\}_{k\ge 1}$.
By the Arzel\`a--Ascoli theorem, for each fixed $m$ one may extract a subsequence such that
$f^{\vv Z^{(k)}},D_x^\beta f^{\vv Z^{(k)}}\ (|\beta|\le 2),\partial_t f^{\vv Z^{(k)}}$
all converge uniformly on $K_m$.

Applying this successively for $m=1,2,\dots$ and using a diagonal argument, we obtain a subsequence, still denoted by $\{f^{\vv Z^{(k)}}\}$, such that for every $m$,
$f^{\vv Z^{(k)}},D_x^\beta f^{\vv Z^{(k)}}\ (|\beta|\le 2),\partial_t f^{\vv Z^{(k)}}$
converge uniformly on $K_m$ as $k\to\infty$. Denote the corresponding limits by
$f^*, g_\beta\ (|\beta|\le 2), h$.

We verify that $g_\beta=D_x^\beta f^*$ for arbitrary $|\beta|\le 2$, and $h=\partial_t f^*$. Fix $m\in\mathbb N$. Since
$f^{\vv Z^{(k)}}\to f^*$,
$\partial_{x_i}f^{\vv Z^{(k)}}\to g_{e_i}$
uniformly on $K_m$, for any $(x,t)\in K_m$ and any $\delta\in\mathbb R$ such that $(x+se_i,t)\in K_m$ for all $s\in[0,\delta]$, we have
\[
f^{\vv Z^{(k)}}(x+\delta e_i,t)-f^{\vv Z^{(k)}}(x,t)
=
\int_0^\delta \partial_{x_i}f^{\vv Z^{(k)}}(x+se_i,t)\,\diff s.
\]
Passing to the limit as $k\to\infty$ yields
\[
f^*(x+\delta e_i,t)-f^*(x,t)
=
\int_0^\delta g_{e_i}(x+se_i,t)\,\diff s.
\]
Since $g_{e_i}$ is continuous on $K_m$, dividing by $\delta$ and letting $\delta\to 0$ gives
$\partial_{x_i}f^*(x,t)=g_{e_i}(x,t)$.
Hence $f^*$ is $C^1$ in the spatial variables on $K_m$, and
$\nabla_x f^*=(g_{e_1},\dots,g_{e_d})$.

Next, fix $i,j\in\{1,\dots,d\}$. Since
$\partial_{x_i}f^{\vv Z^{(k)}}\to g_{e_i}$, 
$
\partial_{x_j}\partial_{x_i}f^{\vv Z^{(k)}}\to g_{e_i+e_j}$
uniformly on $K_m$, the same argument applied to $\partial_{x_i}f^{\vv Z^{(k)}}$ shows that  on $K_m$,
$\partial_{x_j}g_{e_i}=g_{e_i+e_j}$.
Because $g_{e_i}=\partial_{x_i}f^*$, this implies that on $K_m$,
$\partial_{x_j}\partial_{x_i}f^*=g_{e_i+e_j}$.
Therefore, $f^*$ is twice continuously differentiable in $x$ on $K_m$ with
$D_x^\beta f^*=g_\beta$, for arbitrary $|\beta|\le 2$.

Similarly, 
since
$f^{\vv Z^{(k)}}\to f^*$,
$
\partial_t f^{\vv Z^{(k)}}\to h$
uniformly on $K_m$, we have
$\partial_t f^*(x,t)=h(x,t)$.
Thus $f^*\in C^{2,1}(K_m)$. Since $m$ is arbitrary, we conclude that
$f^*\in C^{2,1}(\R^d\times[0,1])$.
Moreover, for each $m$,
$f^{\vv Z^{(k)}}\to f^*$
in $C^{2,1}(K_m)$,
and in particular,
$\nabla_x f^{\vv Z^{(k)}}\to \nabla_x f^*$
uniformly on every compact subset of $\R^d\times[0,1]$.

We apply Lemma~\ref{lem:continuity-integrated-loss} with
$h_k(x,t):=\nabla f_t^{\vv Z^{(k)}}(x)$,
$h_*(x,t):=\nabla f_t^*(x)$,
$\ell(a;b):=\|a-b\|^2$.
The required local uniform convergence of $h_k$ to $h_*$ follows from the
preceding compactness argument.

It remains only to verify uniform integrability of the corresponding
integrated losses
\[
V_k
:=
\int_0^1
\left\|
Z_1^{(k)}-Z_0^{(k)}
-
\nabla f_t^{\vv Z^{(k)}}(Z_t^{(k)})
\right\|^2
\diff t.
\]
By the Cauchy-Schwarz inequality and the assumption
$\|\nabla f_t^{\vv Z^{(k)}}(x)\|\le Cg(x)$,
we have
\begin{align*}
V_k
&\le
2\|Z_1^{(k)}-Z_0^{(k)}\|^2
+
2\int_0^1
\left\|
\nabla f_t^{\vv Z^{(k)}}(Z_t^{(k)})
\right\|^2
\diff t
\le
4\|Z_0^{(k)}\|^2
+
4\|Z_1^{(k)}\|^2
+
2C^2\int_0^1 g(Z_t^{(k)})^2\diff t.
\end{align*}
Since $Z_0^{(k)}\sim P$ and $Z_1^{(k)}\sim Q$ for every $k$, and
$P,Q\in\mathcal P_2(\R^d)$, the families
$\{\|Z_0^{(k)}\|^2\}_{k\ge1}$
and
$\{\|Z_1^{(k)}\|^2\}_{k\ge1}$
are uniformly integrable. By assumption,
$\left\{
\int_0^1 g(Z_t^{(k)})^2\diff t
\right\}_{k\ge1}$
is also uniformly integrable. Hence $\{V_k\}_{k\ge1}$ is uniformly
integrable.
Lemma~\ref{lem:continuity-integrated-loss} therefore gives
\[
L_{\vv Z^{(k)}}(f^{\vv Z^{(k)}})
=
\E[V_k]
\longrightarrow
\E\left[
\int_0^1
\left\|
Z_1^*-Z_0^*
-
\nabla f_t^*(Z_t^*)
\right\|^2
\diff t
\right]
=
L_{\vv Z^*}(f^*).
\]
Since
$L_{\vv Z^{(k)}}(f^{\vv Z^{(k)}})\longrightarrow0$,
we conclude that
$L_{\vv Z^*}(f^*)=0$.
Therefore, by Theorem~\ref{thm:mini},
$(Z_0^*,Z_1^*)$ is an optimal coupling of $(P,Q)$.

It remains to prove the claimed full weak convergence under uniqueness.
Since the family $\{(Z_0^{(k)},Z_1^{(k)})\}_{k\ge1}$ is tight, every subsequence
$\{(Z_0^{(k_j)},Z_1^{(k_j)})\}_{j\ge1}$ admits a further weakly convergent subsequence, and the limit is given by the optimal coupling $(Z_0^{*},Z_1^{*})$ for the same reason as above. This implies that the full sequence satisfies
$(Z_0^{(k)},Z_1^{(k)})\xrightarrow{w}(Z_0^*,Z_1^*)$.
\end{proof}
\begin{remark}\label{rmk:ui}
The uniform integrability assumption in the theorem is automatically satisfied if \(g\) has polynomial growth. More precisely, suppose there exist constants \(C>0\) and \(p\ge 0\) such that $g(x)\le C(1+\|x\|^p)$ for all $x\in\mathbb R^d$.
If \(P\) and \(Q\) have finite \(2p+\epsilon\)-moments, then the family
\[
\left\{
\int_0^1 g\bigl(Z^{(k)}_s\bigr)^2\,\diff s
\right\}_{k\ge 1}
\]
is uniformly integrable.

Indeed, since
$\|Z_t^{(k)}\|\le \|Z_0^{(k)}\|+\|Z_1^{(k)}\|$, using H\"older's inequality,
\[
g(Z_t^{(k)})^2
\le C(1+\|Z_t^{(k)}\|^{2p})
\le C\bigl(1+\|Z_0^{(k)}\|^{2p}+\|Z_1^{(k)}\|^{2p}\bigr).
% \implies
% \int_0^1 g(Z_t^{(k)})^2\,\diff t
% \le
% C\bigl(1+\|Z_0^{(k)}\|^{2p}+\|Z_1^{(k)}\|^{2p}\bigr).
\]

Since \(Z_0^{(k)}\sim P\) and \(Z_1^{(k)}\sim Q\) for all \(k\), and \(P,Q\) have finite \(2p+\epsilon\)-moments, the families
$\{\|Z_0^{(k)}\|^{2p}\}_{k\ge 1}$, $
\{\|Z_1^{(k)}\|^{2p}\}_{k\ge 1}$
are uniformly integrable. Therefore
$\left\{
\int_0^1 g(Z_t^{(k)})^2\,\diff t
\right\}_{k\ge 1}$
is uniformly integrable.

Especially, if $P$ and $Q$ follow a Gaussian distribution, then $g$ will be a linear function, satisfying all the above.
\end{remark}

\section{Proof of Section~\ref{sec:rf-failure}}
\label{app:rf-failure-proof}
\begin{proof}[Proof of Proposition~\ref{prop:gaussian-joint-velocity}]
Since $(X_0,X_1)$ is jointly Gaussian, the pair
$(X_1-X_0,X_t)$ is jointly Gaussian and centered. Therefore,
\[
\E[X_1-X_0\mid X_t=x]
=
\Cov(X_1-X_0,X_t)\Sigma_t^{-1}x.
\]
Moreover,
\begin{align*}
\Cov(X_1-X_0,X_t)
&=
\Cov\bigl(X_1-X_0,(1-t)X_0+tX_1\bigr)\\
&=
(1-t)\bigl(\Cov(X_1,X_0)-\Cov(X_0,X_0)\bigr)
+t\bigl(\Cov(X_1,X_1)-\Cov(X_0,X_1)\bigr)\\
&=
(1-t)(\Sigma_{01}^{\top}-\Sigma_1)
+t(\Sigma_2-\Sigma_{01}).
\end{align*}
Substituting this identity into the Gaussian conditional expectation formula
gives \eqref{eq:vt-joint}.
\end{proof}

\begin{proof}[Proof of Lemma~\ref{lem:rf-1}]
Denote
$M = \Sigma_1^{-1/2}\Sigma_2\Sigma_1^{-1/2}$,
then $M$ is symmetric positive definite. We factorize $\Sigma_t$ and the numerator in the drift as
\[
\Sigma_t
=
\Sigma_1^{1/2}\bigl((1-t)^2I + t^2M\bigr)\Sigma_1^{1/2},
\qquad
t\Sigma_2 - (1-t)\Sigma_1
=
\Sigma_1^{1/2}\bigl(tM - (1-t)I\bigr)\Sigma_1^{1/2}.
\]
Substituting these expressions into the ODE gives
\[
\partial_t Z_t
=
\Sigma_1^{1/2}
\Bigl[\bigl(tM-(1-t)I\bigr)\bigl((1-t)^2I+t^2M\bigr)^{-1}\Bigr]
\Sigma_1^{-1/2} Z_t.
\]

Now define the whitened variable
$Y_t := \Sigma_1^{-1/2} Z_t$
for all $t\in[0,1]$,
then
\[
\partial_t Y_t
=
\Sigma_1^{-1/2}\partial_t Z_t
=
\bigl(tM-(1-t)I\bigr)\bigl((1-t)^2I+t^2M\bigr)^{-1}Y_t
= 
\frac12 \Bigl(\partial_t\bigl((1-t)^2I+t^2M\bigl)\Bigr)\bigl((1-t)^2I+t^2M\bigr)^{-1}Y_t.
\]
Since each $(1-t)^2I+t^2M$ is a polynomial in $M$, they commute for all $t_1,t_2$.
Therefore, 
\[
\int_0^t \Bigl(\partial_s\bigl((1-s)^2I+s^2M\bigl)\Bigr)\bigl((1-s)^2I+s^2M\bigr)^{-1}\diff s
=
\bigl(\log \bigl((1-t)^2I+t^2M\bigr) - \log I\bigr)
=
\log \bigl((1-t)^2I+t^2M\bigr) ,
\]
% because $N_0 = I$.
Hence
\[
Y_t
=
\exp\Bigl(\frac12 \log \bigl((1-t)^2I+t^2M\bigr)\Bigr)Y_0
=
\bigl((1-t)^2I+t^2M\bigr)^{1/2}Y_0.
\]
Transforming back, we obtain
$Z_t
=
\Sigma_1^{1/2}Y_t
=
\Sigma_1^{1/2}\bigl((1-t)^2I+t^2M\bigr)^{1/2}\Sigma_1^{-1/2}Z_0$,
which is exactly the claimed formula. Evaluating at $t=1$ particularly yields that 
$Z_1
=
\Sigma_1^{1/2}M^{1/2}\Sigma_1^{-1/2}Z_0
=
TZ_0$.
\end{proof}

\begin{proof}[Proof of Lemma~\ref{lem:rf-2}]
Since $\operatorname{Re}(\Spec (L)) > 0$, we have $\operatorname{Re}(\Spec ((1-t)I+tL)) > 0$ for all $t\in [0,1]$, so $(1-t)I+tL$ is always invertible. Therefore, 
as
$X_t = \bigl((1-t)I+tL\bigr)X_0$, 
we have
$X_0 = \bigl((1-t)I+tL\bigr)^{-1}X_t$
almost surely. Therefore
$v_t(x)
=
\E[(L-I)X_0 \mid X_t=x]
=
(L-I)\bigl((1-t)I+tL\bigr)^{-1}x$.

Now consider
$\tilde Z_t := \bigl((1-t)I+tL\bigr)Z_0$.
Then
$\partial_t \tilde Z_t = (L-I)Z_0$.
On the other hand,
\[
(L-I)\bigl((1-t)I+tL\bigr)^{-1}\tilde Z_t
=
(L-I)\bigl((1-t)I+tL\bigr)^{-1}\bigl((1-t)I+tL\bigr)Z_0
=
(L-I)Z_0.
\]
Thus
$\partial_t \tilde Z_t = (L-I)\bigl((1-t)I+tL\bigr)^{-1}\tilde Z_t$,
so $\tilde Z_t$ indeed solves the ODE. By the existence and uniqueness of the linear ODE system, we have $Z_t = \tilde Z_t = \bigl((1-t)I+tL\bigr)Z_0$. 
Evaluating at $t=1$ gives
$Z_1 = LZ_0$.
Hence the deterministic coupling map $L$ is preserved by reflow.
\end{proof}

\begin{proof}[Proof of Proposition~\ref{prop:non-ot}]
By Lemma~\ref{lem:rf-1}, the first run of rectified flow produces the deterministic coupling $X_1 = TX_0$.
Since $\Spec T = \Spec ((\Sigma_1^{-1/2}\Sigma_2\Sigma_1^{-1/2})^{1/2}) > 0$,
applying Lemma~\ref{lem:rf-2} with $L=T$ shows that every subsequent run preserves this coupling.
Thus, the algorithm stabilizes at the map $T$ after one step.

Then, we compare $T$ with the Gaussian optimal transport map $T_{\mathrm{OT}}$.
First assume that $\Sigma_1$ and $\Sigma_2$ commute.
Then $\Sigma_1^{-1/2}$ and $\Sigma_2$ also commute, and hence
$\Sigma_1^{-1/2}\Sigma_2\Sigma_1^{-1/2}
=
\Sigma_1^{-1}\Sigma_2$.
Therefore
\[
T
=
\Sigma_1^{1/2}\bigl(\Sigma_1^{-1}\Sigma_2\bigr)^{1/2}\Sigma_1^{-1/2}
=
\Sigma_2^{1/2}\Sigma_1^{-1/2}.
\]
Similarly, since $\Sigma_1^{1/2}$ and $\Sigma_2$ commute,
$T_{\mathrm{OT}}
=
\Sigma_1^{-1/2}\bigl(\Sigma_1\Sigma_2\bigr)^{1/2}\Sigma_1^{-1/2}
=
\Sigma_2^{1/2}\Sigma_1^{-1/2}$.
Hence
$T = T_{\mathrm{OT}}$.

Conversely, assume that
$T = T_{\mathrm{OT}}
$.
Since $T_{\mathrm{OT}}$ is symmetric positive definite, $T$ must be symmetric. Write
\[
M := \Sigma_1^{-1/2}\Sigma_2\Sigma_1^{-1/2}.
\]
Then
$T = \Sigma_1^{1/2}M^{1/2}\Sigma_1^{-1/2}$.
The symmetry of $T$ gives $T=T^\top$, which is,
\[
\Sigma_1^{1/2}M^{1/2}\Sigma_1^{-1/2}
=
\Sigma_1^{-1/2}M^{1/2}\Sigma_1^{1/2}.
\]
Multiplying on the left and right by $\Sigma_1^{1/2}$ yields
$\Sigma_1 M^{1/2} = M^{1/2}\Sigma_1$.
Therefore, $\Sigma_1$ commutes with $M^{1/2}$, and hence also with
$M = M^{1/2}M^{1/2}$.

Since $M$ commutes with $\Sigma_1$, $M$ also commutes with every spectral function of $\Sigma_1$, in particular with $\Sigma_1^{1/2}$.
Using
$\Sigma_2 = \Sigma_1^{1/2}M\Sigma_1^{1/2}$,
we compute
\[
\Sigma_1\Sigma_2
=
\Sigma_1\Sigma_1^{1/2}M\Sigma_1^{1/2}
=
\Sigma_1^{3/2}M\Sigma_1^{1/2}
=
\Sigma_1^{1/2}M\Sigma_1^{3/2}
=
\Sigma_1^{1/2}M\Sigma_1^{1/2}\Sigma_1
=
\Sigma_2\Sigma_1,
\]
where we used the commutation of $\Sigma_1$ and $M$.
Thus
$\Sigma_1\Sigma_2 = \Sigma_2\Sigma_1$.
We have thus shown that $T = T_{\mathrm{OT}}$ implies $\Sigma_1\Sigma_2 = \Sigma_2\Sigma_1$. Together with the first implication, this proves the equivalence. Hence noncommuting covariances lead after one step to a non-optimal fixed point of rectified flow.
\end{proof}

\subsection{Proofs of Theorem~\ref{thm:relative-log-curvature} and
Corollary~\ref{cor:general-gaussian-source}}

In this section, we refer to the flow mapping generated by the velocity field $v_t$ as $\Phi_t: \R^d \to \R^d$, which means $\Phi_t(x)$ is the solution of the ODE $\frac{\diff}{\diff s} \Phi_s(x) = v_s(\Phi_s(x))$ with initial condition $\Phi_0(x) = x$.

\subsubsection{Technical Lemmas and General Criteria}

We will first use the following consequence of Plastock's line-lifting criterion for
global inversion \cite[Theorem~1.2]{plastock1974homeomorphisms}.

\begin{lemma}[Uniform conorm implies global inversion]\label{lem:global-inverse}
Let $F\in C^1(\R^d;\R^d)$ and suppose
\[
\sigma_{\min}(JF(x))\ge\beta>0\qquad\text{for every }x\in\R^d.
\]
Then $F$ is a global $C^1$ diffeomorphism and $\mathrm{Lip}(F^{-1})\le\beta^{-1}$.
\end{lemma}

\begin{proof}[Proof of Lemma~\ref{lem:global-inverse}]
The assumption implies that $JF(x)$ is invertible and
$\|JF(x)^{-1}\|_{\mathrm{op}}\le\beta^{-1}$ for every $x$. Thus $F$ is a
local $C^1$ diffeomorphism by the inverse function theorem.

By \cite[Theorem~1.2]{plastock1974homeomorphisms}, it remains to verify
condition~$(L)$. Let $p(r)=(1-r)y_0+ry_1$ be a line segment and let
$q:[0,b)\to\R^d$ be a lift, so $F(q(r))=p(r)$. Local inversion shows that $q$
is $C^1$, and
\[
\|\dot q(r)\|
=\|JF(q(r))^{-1}(y_1-y_0)\|
\le\beta^{-1}\|y_1-y_0\|.
\]
Consequently, for $0\le r<s<b$,
$\|q(s)-q(r)\|\le\beta^{-1}\|y_1-y_0\|(s-r)$.
Hence $q(r)$ has a limit as $r\uparrow b$, which is precisely condition~$(L)$.
This implies that $F$ is a global homeomorphism. Since its
local inverses are $C^1$, $F$ is a global $C^1$ diffeomorphism.
\end{proof}

Before going to the detailed setting, we state a general angular criterion for the dynamics led by general velocity fields.

\begin{lemma}[Dynamic angular criterion]\label{lem:dynamic-angular}
Let $v \in C^1([0,1]\times\R^d;\R^d)$. Suppose it generates a complete $C^1$
flow $\Phi_t$, with $\Phi_0=\id$, and let $T=\Phi_1$. Assume $T$ is globally Lipschitz and, for integrable
functions $r$ and $\Omega$, and for every $t\in[0,1]$ and $x\in\R^d$,
\begin{align*}
\lambda_{\min}(\Sym Jv_t(x))&\ge r(t),\qquad
\sup_{|u|=1}\|(I_d-uu^\top)Jv_t(x)u\|\le\Omega(t), \qquad
\int_0^1 \Omega(t)\,\diff t<\pi.
\end{align*}
Then every
$T_s=(1-s)\id+sT$ is a global bi-Lipschitz $C^1$ diffeomorphism and for every $x, h\in\R^d$,
\begin{align}
\|JT(x)h\|&\ge
\exp\!\left(\int_0^1r(t)\,\diff t\right)\|h\|,
\qquad
\angle(h,JT(x)h)\le\int_0^1\Omega(t)\,\diff t; \label{eq:dynamic-tangent-bounds}\\
\sigma_{\min}(JT_s(x))&\ge
\cos\!\left(\frac12\int_0^1\Omega(t)\,\diff t\right)
\min\!\left\{1,\exp\!\left(\int_0^1r(t)\,\diff t\right)\right\}.\label{eq:dynamic-chord-conorm}
\end{align}
Moreover, the map coupling $(X,T(X))$ is rectifiable and straight.

% If $\int_0^1\Omega(t)\,\diff t<\pi/2$, then $T$ is strongly monotone with constant
% $\exp\!\left(\int_0^1r(t)\,\diff t\right)
% \cos\!\left(\int_0^1\Omega(t)\,\diff t\right)$.
% At equality, $T$ is monotone.
\end{lemma}

\begin{proof}[Proof of Lemma~\ref{lem:dynamic-angular}]
Fix $x\in\R^d$ and $h\ne0$. 
For brevity within this proof, we will define
$\beta=\cos\!\left(\frac12\int_0^1\Omega(t)\,\diff t\right)\min\left\{1,\exp\!\left(\int_0^1r(t)\,\diff t\right)\right\}$. Since we assume $\int_0^1\Omega(t)\,\diff t<\pi$, we have $\beta>0$.

For each $x\in\R^d$, differentiating the flow ODE with respect to $x$ and applying the resulting Jacobian
equation to $h\ne0$, we find that
\begin{equation}\label{eq:tangent-ode}
\frac{\diff}{\diff t}\bigl(J\Phi_t(x)h\bigr)
=Jv_t(\Phi_t(x))J\Phi_t(x)h,
\qquad J\Phi_0(x)h=h.
\end{equation}
By Liouville's formula, we have
\[
\det J\Phi_t(x)=\exp\!\left(\int_0^t\tr\bigl(Jv_r(\Phi_r(x))\bigr)\,\diff r\right)>0.
\]
Thus $J\Phi_t(x)h\ne0$, and its normalized direction
$u_t:=J\Phi_t(x)h/\|J\Phi_t(x)h\|$ is well defined. Therefore,
\eqref{eq:tangent-ode} also yields
\begin{equation*}
\frac{\diff}{\diff t}\log\|J\Phi_t(x)h\|
=u_t^\top\Sym Jv_t(\Phi_t(x))u_t\ge r(t).
\end{equation*}
Integrating gives
$\|JT(x)h\|\ge \exp\!\left(\int_0^1r(t)\,\diff t\right)\|h\|$.
Differentiating the normalized direction $u_t$, we have
$\frac{\diff}{\diff t} u_t=(I_d-u_tu_t^\top)Jv_t(\Phi_t(x))u_t$.
Consequently $\|\frac{\diff}{\diff t} u_t\|\le\Omega(t)$, and the spherical length of the direction curve is bounded by $\int_0^1\Omega(t)\,\diff t$, hence
$\angle(h,JT(x)h)=\angle(u_0,u_1)\le\int_0^1\Omega(t)\,\diff t$.

Fix $s\in(0,1)$ and the above controls,
\begin{align*}
\|JT_s(x)h\|&=\|(1-s)h+sJT(x)h\|\ge\cos\!\left(\frac12\int_0^1\Omega(t)\,\diff t\right)\bigl((1-s)\|h\|+s\|JT(x)h\|\bigr)
\ge\beta\|h\|.
\end{align*}
At $s=0$ and $s=1$, the same conclusion follows directly from $JT_0=I_d$. Thus \eqref{eq:dynamic-chord-conorm} holds
for every $s\in[0,1]$.
Lemma~\ref{lem:global-inverse} shows that every
$T_s$ is a global $C^1$ diffeomorphism and
$\mathrm{Lip}(T_s^{-1})\le\beta^{-1}$uniformly in $s$.
Because $T$ is globally Lipschitz,
$\mathrm{Lip}(T_s)\le(1-s)+s\mathrm{Lip}(T)\le\max\{1,\mathrm{Lip}(T)\}$.
Thus every $T_s$ is bi-Lipschitz, with both bounds uniform in $s$.

Suppose that we are starting from the coupling $(X,T(X))$, denote $(X_0,X_1)=(X,T(X))$ and $X_s=(1-s)X_0+sX_1$ for $s\in[0,1]$.
Since $T_s$ is bijective,
$X_1-X_0 = T(X) - X=(T-\id)\bigl(T_s^{-1}((1-s)X_0 + sX_1)\bigr)$.
Therefore, the conditional velocity has the version
$v_s(x) := \E [X_1 - X_0 \mid X_s=x] = (T-\id)\bigl(T_s^{-1}(x)\bigr)$.
Furthermore,
$\mathrm{Lip}(v_s)
\le
(\mathrm{Lip}(T)+1)\beta^{-1}$
uniformly in $s$. Thus $v$ has uniform linear growth and the ODE
$\dot Z_s=v_s(Z_s)$
has a unique global solution on $[0,1]$. 

For every $x\in\R^d$,
\[
\frac{\diff}{\diff s}T_s(x)=T(x)-x
=(T-\id)\bigl(T_s^{-1}(T_s(x))\bigr) = v_s(T_s(x)).
\]
By uniqueness, $s\mapsto T_s(x)$ is exactly the solution to the ODE with initial condition $T_0(x)=x$. In particular, $T_1(x)=T(x)$ is the endpoint of
the rectified flow. Hence $(X,T(X))$ is rectifiable and straight, and its
rectification returns the same endpoint map.

% Finally, if $\theta\le\pi/2$, then
% \[
% h^\top JT(x)h=\|h\|\|JT(x)h\|\cos\angle(h,JT(x)h)\ge \exp\!\left(\int_0^1r(t)\,\diff t\right)\cos\theta\|h\|^2.
% \]
% The mean value theorem gives
% $(T(x)-T(y))^\top (x-y)\ge \exp\!\left(\int_0^1r(t)\,\diff t\right)\cos\theta\|x-y\|^2$.
% This proves strong monotonicity when $\theta<\pi/2$ and monotonicity at equality.
\end{proof}

\subsubsection{Proof of Theorem~\ref{thm:relative-log-curvature} and Corollary~\ref{cor:general-gaussian-source}}
Below, we will specifically focus on our setting of rectified flow starting from $X_0\sim \N(0,I_d)$ and $X_1\sim Q$ with $Q\propto \exp(-V)$. We will eventually show that the velocity field $v_t$ satisfies the conditions of Lemma~\ref{lem:dynamic-angular}.

\begin{lemma}[Posterior covariance Bound]\label{lem:posterior-covariance}
For $0\le t<1$, conditionally on $X_t=x$, 
\begin{equation}\label{eq:posterior-covariance}
\left(M+\frac{t^2}{(1-t)^2}\right)^{-1}I_d
\preceq \Cov(X_1\mid X_t=x)\preceq
\left(m+\frac{t^2}{(1-t)^2}\right)^{-1}I_d.
\end{equation}
For $0\le t<1$, the conditional velocity has a smooth version satisfying
\begin{equation}\label{eq:velocity-jacobian}
Jv_t(x)=-\frac{1}{1-t}I_d+
\frac{t}{(1-t)^3}\Cov(X_1\mid X_t=x).
\end{equation}
It is symmetric and
\begin{equation}\label{eq:velocity-jacobian-bounds}
\frac{t-M(1-t)}{t^2+M(1-t)^2}I_d
\preceq Jv_t(x)\preceq
\frac{t-m(1-t)}{t^2+m(1-t)^2}I_d.
\end{equation}
\end{lemma}

\begin{proof}[Proof of Lemma~\ref{lem:posterior-covariance}]
From Bayes' rule and
$X_0=(X_t-tX_1)/(1-t)$, we get the posterior density of $X_1$ given $X_t=x$ is propotional to
$\exp\!\left(-V(y)-\frac{|x-ty|^2}{2(1-t)^2}\right)$.
Therefore, 
we have
\[
\left(m+\frac{t^2}{(1-t)^2}\right)I_d
\preceq\nabla^2\left(V(y)+\frac{|x-ty|^2}{2(1-t)^2}\right)\preceq
\left(M+\frac{t^2}{(1-t)^2}\right)I_d.
\]
Then Lemma~2 in
\cite{chewi2023entropic} gives \eqref{eq:posterior-covariance}.
Now, differentiating the posterior density of $X_1$ given $X_t=x$ gives
\[
J_x\E[X_1\mid X_t=x]
=\frac{t}{(1-t)^2}\Cov(X_1\mid X_t=x).
\]
Since $X_0=(x-tX_1)/(1-t)$ conditionally on $X_t=x$,
$v_t(x)=\frac{\E[X_1\mid X_t=x]-x}{1-t}$,
and \eqref{eq:velocity-jacobian} follows. Substitution of
\eqref{eq:posterior-covariance} yields
\eqref{eq:velocity-jacobian-bounds}. 
\end{proof}

\begin{lemma}[Global flow and endpoint regularity]\label{lem:global-flow}
Using the same notation as in Lemma~\ref{lem:posterior-covariance}, 
The maps $(t,x)\mapsto v_t(x)$ and $(t,x)\mapsto Jv_t(x)$ extend
continuously to $[0,1]\times\R^d$ by
$v_1(x)=x,
Jv_1(x)=I_d$. Furthermore, $v$ is uniformly globally
Lipschitz in space, has uniform linear growth, and generates a unique complete
$C^1$ flow $\Phi_t$ on $[0,1]$; $T=\Phi_1$ is $C^1$ and
globally Lipschitz.
\end{lemma}

\begin{proof}[Proof of Lemma~\ref{lem:global-flow}]
For $t<1$, Lemma~\ref{lem:posterior-covariance} shows that $v$ and $Jv$ are
jointly continuous on $[0,1)\times\R^d$.
Both scalar bounds in \eqref{eq:velocity-jacobian-bounds} are bounded on
$[0,1)$ and converge to one as $t\uparrow1$. Since $Jv_t(x)$ is symmetric, we have
$\sup_{0\le t<1}\sup_{x\in\R^d}\|Jv_t(x)\|_{\mathrm{op}}<\infty$ and 
$\sup_{x\in\R^d}\|Jv_t(x)-I_d\|_{\mathrm{op}}\longrightarrow0$.

It remains to prove that $v_t(x)\to x$. At $x=0$, the conditional density of $X_1$ given $X_t=0$ is proportional to
$\exp(-V(y)-\frac{t^2}{2(1-t)^2}\|y\|^2)$. Integration by parts and the lower Hessian bound of \(V\)
give
\[
\begin{aligned}
d
&=\E[X_1^\top\nabla V(X_1)\mid X_t=0]
+\frac{t^2}{(1-t)^2}\E[\|X_1\|^2\mid X_t=0]\\
&\ge \left(m+\frac{t^2}{(1-t)^2}\right)\E[\|X_1\|^2\mid X_t=0]
-\|\nabla V(0)\|\E[\|X_1\|\mid X_t=0].
\end{aligned}
\]
By Cauchy--Schwarz inequality, we can solve
\[
\Bigl(\E[\|X_1\|^2\mid X_t=0]\Bigr)^{1/2}
\le \frac{\|\nabla V(0)\|}{m+\frac{t^2}{(1-t)^2}}+\sqrt{\frac{d}{m+\frac{t^2}{(1-t)^2}}}.
\]
Therefore, $\E[\|X_1\|\mid X_t=0]$ is uniformly bounded as $t \uparrow 1$. Also, since $\nabla V$ is $M$-Lipschitz, we have
$\E[\|\nabla V(X_1)\|\mid X_t=0] \le\|\nabla V(0)\|+M\E[\|X_1\|\mid X_t=0]$ is also uniformly bounded.
Another integration by parts gives
\[
\E[\nabla V(X_1)\mid X_t=0]+\frac{t^2}{(1-t)^2}\E[X_1\mid X_t=0]=0.
\]
Hence, for $t>0$,
\[
v_t(0)
=-\frac{1-t}{t^2}\E[\nabla V(X_1)\mid X_t=0].
\]
Thus $v_t(0)\to0$.
For every $R<\infty$, 
$\sup_{|x|\le R}|v_t(x)-x|
\le |v_t(0)|
+R\sup_{x\in\R^d}\|Jv_t(x)-I_d\|_{\mathrm{op}}\to 0$.
Together with the uniform convergence of $Jv_t$ to $I_d$, this proves both
claimed continuous extensions.

The uniform derivative bound and the boundedness of $v_t(0)$ on $[0,1]$
give a constant $C<\infty$ such that, for all $t\in[0,1]$ and $x\in\R^d$,
$\|v_t(x)\|\le C(1+\|x\|)$ and $v_t$ is a $C$-Lipschitz function in $x$.
Therefore, there exists a unique
global solution 
$\Phi_t$ generated by the vector field $v_t$.
Since $Jv_t$ is uniformly continuous, 
the $C^1$ dependence theorem for ODEs gives the following ODE
\[
\frac{\diff}{\diff t}J\Phi_t(x)
=Jv_t(\Phi_t(x))J\Phi_t(x),\qquad J\Phi_0(x)=I_d.
\]
Thus $\Phi_t$ is $C^1$ in space, and $T = \Phi_1$ is globally Lipschitz. 
\end{proof}

\begin{lemma}[Endpoint geometry]\label{lem:endpoint-geometry}
Denote $T=\Phi_1$. For every $x,h\in\R^d$ and $h\ne0$,
\begin{equation}\label{eq:endpoint-tangent-bounds}
M^{-1/2}\|h\|\le \|JT(x)h\|\le m^{-1/2}\|h\|,
\qquad
\angle(h,JT(x)h)\le\frac14\log\frac{M}{m}.
\end{equation}
If $M/m<e^{4\pi}$, then every $T_s=(1-s)\id+sT$ is a global
bi-Lipschitz $C^1$ diffeomorphism with
\[
\sigma_{\min}(JT_s(x))\ge
\cos\!\left(\frac18\log\frac{M}{m}\right)\min\{1,M^{-1/2}\},
\]
and $(X_0,T(X_0))$ is rectifiable and straight.
\end{lemma}

\begin{proof}[Proof of Lemma~\ref{lem:endpoint-geometry}]
Fix $x\in\R^d$ and $h\ne0$. Lemma~\ref{lem:global-flow} makes each
$\Phi_t$ a $C^1$ diffeomorphism, so $J\Phi_t(x)h\ne0$. Therefore,
\[
\frac{\diff}{\diff t}\log\|J\Phi_t(x)h\|
=\frac{(J\Phi_t(x)h)^\top Jv_t(\Phi_t(x))J\Phi_t(x)h}
{\|J\Phi_t(x)h\|^2}.
\]
Integrating the two bounds in \eqref{eq:velocity-jacobian-bounds} yields
\[
-\frac12\log M
\le \log\frac{\|JT(x)h\|}{\|h\|}
\le-\frac12\log m,
\]
which proves the length bounds in \eqref{eq:endpoint-tangent-bounds}.

For a symmetric matrix $A$ with spectrum in $[a,b]$ and an arbitrary $\|u\|=1$, 
we have 
\[
\|(I_d-uu^\top)Au\|^2
=\|Au\|^2-(u^\top Au)^2
\le\frac{(b-a)^2}{4}.
\]
Therefore, by \eqref{eq:velocity-jacobian-bounds}, we have
\begin{align*}
2\sup_{x\in\R^d,\,|u|=1}\|(I_d-uu^\top)Jv_t(x)u\|
&\le \frac{t-m(1-t)}{t^2+m(1-t)^2}-\frac{t-M(1-t)}{t^2+M(1-t)^2} \,\,\implies \\
\int_0^1\sup_{x\in\R^d,\,|u|=1}
\|(I_d-uu^\top)Jv_t(x)u\|\,\diff t
&\le\frac12\left(-\frac12\log m+\frac12\log M\right)
=\frac14\log\frac{M}{m}.
\end{align*}

Lemma~\ref{lem:global-flow} verifies the regularity, completeness, and global
Lipschitz hypotheses of Lemma~\ref{lem:dynamic-angular}. Combining with the bound in \eqref{eq:velocity-jacobian-bounds}, we can take $r$ as the lower bound
and $\Omega$ as the difference between the upper and lower bounds. 
Their integrals are
$-\frac12\log M$ and $\frac14\log(M/m)$, respectively. If $M/m<e^{4\pi}$, we are able to apply Lemma~\ref{lem:dynamic-angular} to conclude.
\end{proof}

\begin{proof}[Proof of Theorem~\ref{thm:relative-log-curvature}]
By Lemma~\ref{lem:global-flow}, the velocity field $v$ given by the straight coupling generates a unique complete
flow $\Phi_t$ on $[0,1]$. Hence, the first step rectified flow is well-defined, and
with $T:=\Phi_1$, its endpoint coupling can be written as
$(Z_0,Z_1)=(Z_0,T(Z_0))$.
If $M/m<e^{4\pi}$, Lemma~\ref{lem:endpoint-geometry} shows that this map
coupling is fixed under rectified flow; that is,
$\rect((Z_0,Z_1))=(Z_0,Z_1)$.
\end{proof}

\begin{proof}[Proof of Corollary~\ref{cor:general-gaussian-source}]
Let $A=\Sigma^{1/2}$ and shift the coupling by setting
$\widetilde X_i=A^{-1}(X_i-b)$ for $i=0,1$.
Then $\widetilde X_0\sim\N(0,I_d)$, and $\widetilde X_1$ is independent of
$\widetilde X_0$ and has density proportional to $e^{-\widetilde V}$.
Theorem~\ref{thm:relative-log-curvature} therefore applies to
$(\widetilde X_0,\widetilde X_1)$.

It remains to transfer its conclusion back to the original coordinates. More
generally, let
\[
\widetilde Y_t=(1-t)\widetilde Y_0+t\widetilde Y_1,
\qquad
Y_t=b+A\widetilde Y_t,
\]
and denote their conditional velocity fields by $\widetilde v_t$ and $v_t$,
respectively. Since $A$ is invertible,
$v_t(x)
=A\widetilde v_t\bigl(A^{-1}(x-b)\bigr)$.
Thus, if $\widetilde Z_t$ solves the rectified-flow ODE driven by
$\widetilde v_t$, then $Z_t=b+A\widetilde Z_t$ solves the ODE driven by
$v_t$. Uniqueness of the two ODEs shows that rectification commutes with this
invertible affine transformation:
\[
\rect\bigl((b+A\widetilde Y_0,b+A\widetilde Y_1)\bigr)
=
\bigl(b+A\widetilde Z_0,b+A\widetilde Z_1\bigr),
\quad
(\widetilde Z_0,\widetilde Z_1)
=\rect((\widetilde Y_0,\widetilde Y_1)).
\]

Let $(\widetilde Z_0,\widetilde Z_1)=\rect(\widetilde X_0,\widetilde X_1)$. By
Theorem~\ref{thm:relative-log-curvature}, $(\widetilde Z_0,\widetilde Z_1)$ is a fixed point of the rectified flow.
Applying the preceding affine-equivariance identity first to the input
coupling and then to $(\widetilde Z_0,\widetilde Z_1)$ shows that $\rect(X_0,X_1)$ is also a fixed point of the rectified flow.
\end{proof}

\section{Proof of Section~\ref{sec:gaussian-convergence-speed}}
\label{app:gaussian-convergence-speed-proof}
\subsection{Technical Preparation}

We begin with two elementary structural facts that will be used throughout the proof. 

\paragraph*{Factorization of linear transports}
Every linear transport can be regarded as a matrix
$T\in\R^{d\times d}$ and $T(x)=Tx$,
\[
T_\#P=Q \iff T\Sigma_1T^\top=\Sigma_2.
\]
The quadratic-cost optimal transport map is also linear:
\begin{equation}\label{eq:gauss-opt-map-linear}
T^*(x)=\Sigma_1^{-1/2}\big(\Sigma_1^{1/2}\Sigma_2\Sigma_1^{1/2}\big)^{1/2}\Sigma_1^{-1/2}x:=T^* x,
\end{equation}
so that $T^*$ is symmetric positive definite and $T^*\Sigma_1 T^*=\Sigma_2$.

\begin{lemma}[Factorization by a $\Sigma_1$-rotation]\label{lem:AU-factorization}
Let $T$ satisfy $T\Sigma_1T^\top=\Sigma_2$, and let $T^*$ be as in \eqref{eq:gauss-opt-map-linear}.
Then there exists a unique matrix $U$ such that
\[
T=T^* U,\qquad U\Sigma_1U^\top=\Sigma_1.
\]
In particular, $U$ is orthogonal with respect to the inner product
$\langle x,y\rangle_{\Sigma_1^{-1}}:=x^\top\Sigma_1^{-1}y$. 
When $\Sigma_1=I$, $U$ is an ordinary orthogonal matrix.
\end{lemma}

\begin{proof}[Proof of Lemma~\ref{lem:AU-factorization}]
Define $U:=(T^*)^{-1}T$. Using $(T^*)^\top = T^*$, $T^*\Sigma_1 T^*=\Sigma_2$ and $T\Sigma_1T^\top=\Sigma_2$,
\[
U\Sigma_1U^\top
=(T^*)^{-1}T\Sigma_1T^\top (T^*)^{-1}
=(T^*)^{-1}\Sigma_2(T^*)^{-1}
% =(T^*)^{-1}(T^*\Sigma_1 T^*)(T^*)^{-1}
=\Sigma_1.
\]
Thus $T=T^* U$ with $U\Sigma_1U^\top=\Sigma_1$. Uniqueness is immediate since $T^*$ is invertible.
\end{proof}

Henceforth we write
$T(x)=Tx=T^* Ux, T^*(x)=T^* x$,
with $U\Sigma_1U^\top=\Sigma_1$.

\paragraph*{ Projection of linear fields}

We work in the Hilbert space $L^2(\mu;\R^d)$ with inner product
\[
\langle f,g\rangle_{\mu}:=\E[f(X)^\top g(X)],\qquad X\sim\N(0,\Sigma).
\]
For matrices $M,N\in\R^{d\times d}$, by a slight abuse of notation, we also write
\begin{equation}\label{eq:matrix-inner-product}
\langle M,N\rangle_{\mu}:=\langle Mx,Nx\rangle_{\mu}=\E[X^\top M^\top N X]=\tr(M\Sigma N^\top).
\end{equation}

The next lemma is the finite-dimensional form of the weighted Helmholtz decomposition for Gaussian measures. It shows that, when the input vector field is linear, both its gradient and divergence-free components are again linear. This is the key reason that the Gaussian-linear class is preserved by one-step \textit{c}-rectified flow.
\begin{lemma}[Projection of a linear field under a Gaussian measure]
\label{lem:linear-proj-gaussian}
Let \(\mu=\N(0,\Sigma)\) on \(\R^d\), where \(\Sigma\succ 0\), and \(p\) is the density of \(\mu\).  Denote \({\mathcal G}_\mu:=\overline{\{\nabla \varphi:\varphi\in C_c^\infty(\R^d)\}}^{\,L^2(\mu;\R^d)}\), and
\[
{\mathcal S}_\mu:=\left\{u\in L^2(\mu;\R^d): \nabla\cdot(pu)=0 \text{ in the weak sense, i.e.}~\int_{\R^d} u^\top \nabla \varphi\,p(\diff x)=0
\text{ for every } \varphi\in C_c^\infty(\R^d)
\right\}. 
\]
Then,
\begin{enumerate}[label = (\alph*)]
    \item  \({\mathcal G}_\mu={\mathcal S}_\mu^\perp\) in \(L^2(\mu;\R^d)\),
    and for every linear map $x\mapsto Mx$, $Mx \in {\mathcal S}_\mu$ is equivalent to $M = \Sigma K$ for some skew-symmetric $K$, $Mx \in {\mathcal G}_\mu$ is equivalent to $M$ is symmetric;
    \item For every matrix \(M\in\R^{d\times d}\), the orthogonal projections
$\proj_{{\mathcal S}_\mu}(Mx),  \proj_{{\mathcal G}_\mu}(Mx)$
are both linear functions;
\item 
Suppose \(K\) is the unique skew-symmetric matrix satisfying the continuous Lyapunov equation
\[
\Sigma K + K \Sigma = M - M^\top,
\]
then 
$\proj_{{\mathcal G}_\mu}(Mx)=(M - \Sigma K)x$, $
\proj_{{\mathcal S}_\mu}(Mx)=\Sigma K x$.
\end{enumerate}

% Let \({\mathcal G}_\mu:={\mathcal S}_\mu^\perp\)be the closure in \(L^2(\mu;\R^d)\) of gradient fields.

\end{lemma}

\begin{proof}[Proof of Lemma~\ref{lem:linear-proj-gaussian}]
We argue in three parts. 

\medskip
\noindent
\textbf{(a): Characterization of linear fields in ${\mathcal S}_\mu$.}
First, ${\mathcal G}_\mu = {\mathcal S}_\mu^\perp$ comes from the definition, and for those $x\mapsto Mx$ in ${\mathcal G}_\mu$, \(Mx=\nabla \phi\) for some potential \(\phi\). Hence 
$M=\nabla(Mx)=\nabla^2\phi$ is symmetric. On the other hand, for symmetric $M$, $Mx = \nabla (\frac{1}{2}x^\top M x) \in {\mathcal G}_\mu$. 

We then characterize those linear fields $x\mapsto Mx$ that belong to ${\mathcal S}_\mu$. Define \(K = \Sigma^{-1} M\).
The condition $Mx\in {\mathcal S}_\mu$ means
$\nabla\cdot(p(x)Mx)\equiv 0$.
Using the product rule and the identity
$\nabla p(x)=-p(x)\Sigma^{-1}x$,
we compute
\[
\nabla\cdot(p(x)Mx)
=
p(x)\tr(M)+\nabla p(x)\cdot Mx
=
p(x)\bigl[\tr(M)-x^\top\Sigma^{-1}Mx\bigr].
\]
Since \(p(x) > 0\), this implies \(x^\top\Sigma^{-1}Mx=\tr(M)\) for all \(x \in \R^d\). 
The quadratic form can be a constant function only if \(K = \Sigma^{-1} M\) satisfies \(K = -K^\top\). 

Conversely, if \(K\) is skew-symmetric, then \(\tr(M) = \tr(\Sigma K) = 0\) by the symmetry of $\Sigma$, and $x^\top\Sigma^{-1}Mx=x^\top Kx=0$ for all $x$, satisfying the equation. So indeed $Mx\in {\mathcal S}_\mu$.

\medskip
\noindent
\textbf{(b): Reduction to linear fields via adapted Hermite polynomials.}
Let \(\{{\rm He}_k(z)\}_{k \in \mathbb{N}^d}\) be the standard multivariate Hermite polynomials, which form an orthogonal basis for \(L^2(\gamma;\R^d)\) where \(\gamma = \N(0, I_d)\). Under the affine transformation \(z = \Sigma^{-1/2}x\), the functions \(\Phi_k(x) := {\rm He}_k(\Sigma^{-1/2} x)\) form an orthogonal basis for the scalar space \(L^2(\mu;\R^d)\). We can orthogonally decompose \(L^2(\mu;\R^d) = \bigoplus_{n \ge 0} \mathcal{H}_n(\mu)\), where the \(n\)-th chaos \(\mathcal{H}_n(\mu)\) is the span of \(\{\Phi_k(x) : |k| = n\}\).

By the property of the standard Hermite polynomial, \(\nabla_z {\rm He}_k(z)\in \mathcal{H}_{|k|-1}(\mu)^{\oplus d}\). By the chain rule, \(\nabla_x \Phi_k(x) = \Sigma^{-1/2} \nabla_z {\rm He}_k(z)\). Therefore, \(\nabla_x \Phi_k(x) \in \mathcal{H}_{|k|-1}(\mu)^{\oplus d}\).
This implies that for all \(\phi \in \mathcal{H}_n(\mu)\),\(\nabla \phi\in\mathcal{H}_{n-1}(\mu)^{\oplus d}\). Consequently, if \(\phi \in \mathcal{H}_n(\mu)\) and \(\psi \in \mathcal{H}_m(\mu)\) with \(n \neq m\), \(\nabla \phi \in \mathcal{H}_{n-1}(\mu)^{\oplus d}\) and \(\nabla\psi\in\mathcal{H}_{m-1}(\mu)^{\oplus d}\) respectively. By the orthonogonal decomposition of \(L^2(\mu;\R^d)\),
$\langle \nabla \phi, \nabla \psi \rangle_{\mu} = 0$.
Thus, the gradient subspaces \(\nabla \mathcal{H}_n(\mu)\) are mutually orthogonal in \(L^2(\mu;\R^d)\).

Now, consider the linear vector field \(f(x) = Mx\). Its coordinate functions are linear combinations of the degree-1 adapted Hermite polynomials, meaning \(Mx \in \mathcal{H}_1(\mu)^{\oplus d}\). Because \(\nabla \mathcal{H}_n(\mu) \subset \mathcal{H}_{n-1}(\mu)^{\oplus d}\), for all \(n-1 \neq 1\), \(Mx\perp \nabla \mathcal{H}_n(\mu)\). 
Therefore, \(\proj_{{\mathcal G}_\mu}(Mx)\in \nabla \mathcal{H}_2(\mu)\). Because \(\mathcal{H}_2(\mu)\) consists of scalar quadratic polynomials, \(\nabla \mathcal{H}_2(\mu)\) only have linear vector fields. Therefore there exist matrices $M_G$ and $M_S$ such that
$\proj_{{\mathcal G}_\mu}(Mx)=M_Gx$, $
\proj_{{\mathcal S}_\mu}(Mx)=M_Sx$.
Since $\proj_{{\mathcal G}_\mu}(Mx)\in G_{\mu}$, $M_G$ must be symmetric.

\medskip
\noindent
\textbf{(c): 
Identification of the two projections.}
Since
$Mx=\proj_{{\mathcal G}_\mu}(Mx)+\proj_{{\mathcal S}_\mu}(Mx)=M_Gx+M_Sx$,
we have
$M_G=M-M_S=M-\Sigma K$.
Because $M_G$ is symmetric,
$M-\Sigma K=(M-\Sigma K)^\top=M^\top-K^\top\Sigma^\top$.
Using $\Sigma^\top=\Sigma$ and $K^\top=-K$, we obtain
$M-\Sigma K = M^\top + K\Sigma$,
hence
\[
\Sigma K + K\Sigma = M-M^\top.
\]
It remains to prove uniqueness. Since $\Sigma\succ0$, the linear map
$K\mapsto \Sigma K+K\Sigma$
is invertible on the space of skew-symmetric matrices. Therefore, the above Lyapunov equation has a unique skew-symmetric solution $K$. This uniquely determines
$M_S=\Sigma K,
M_G=M-\Sigma K$.
Thus
$\proj_{{\mathcal S}_\mu}(Mx)=\Sigma Kx$,     $
\proj_{{\mathcal G}_\mu}(Mx)=(M-\Sigma K)x$,
as claimed.
\end{proof}
 
We finally recall a standard variational lower bound for projections in Hilbert spaces, which will be the key tool for proving the lower bound in our cases.

\begin{lemma}\label{lemma:variational-lb}
    Let $\cH$ be a Hilbert space and ${\mathcal S} \subseteq \cH$ a closed subspace. For any $v \in \cH$ and nonzero $w \in {\mathcal S}$,
    \[
    \|\proj_{{\mathcal S}} v\|_{\cH} \ge \frac{|\langle v, w \rangle_{\cH}|}{\|w\|_{\cH}}.
    \]
\end{lemma}
\begin{proof}[Proof of Lemma~\ref{lemma:variational-lb}]
    Let $v_{{\mathcal S}} = \proj_{{\mathcal S}} v$. Since $v - v_{{\mathcal S}} \perp {\mathcal S}$, we have $\langle v, w \rangle_{\cH} = \langle v_{{\mathcal S}}, w \rangle_{\cH}$. The result follows immediately from the Cauchy-Schwarz inequality: $|\langle v_{\mathcal S}, w \rangle_{\cH}| \le \|v_{\mathcal S}\|_{\cH} \|w\|_{\cH}$.
\end{proof}

\subsection*{Explicit \textit{c}-Rectified Flow Framework in the Gaussian Case}\label{sec:crect-gaussian-framework}

We apply the preceding projection formula to the velocity field generated by a linear coupling. This gives an explicit matrix-valued description of one-step \textit{c}-rectified flow. In particular, the projected velocity remains linear at every time, so the resulting flow map is again linear.

Since the joint Gaussian velocity field is always linear as suggested by Proposition~\ref{prop:gaussian-joint-velocity}, independent of the initial coupling (as long as the joint distribution is jointly Gaussian), one-step \textit{c}-rectified flow will always produce a linear transport map. 
Consider the initial linear transport coupling of the Gaussian marginals
$X_1=TX_0$, $X_0\sim P$, $T_\# P = Q$. The straight interpolation associated with this coupling is
\[
X_t=(1-t)X_0+tX_1=((1-t)I+tT)X_0:=T_tX_0.
\]
% Denote $T_t = (1-t) I + t T$, then 
The marginal distribution of the straight interpolation is given by 
$X_t\sim P_t=(T_t)_\#P=\N(0,\Sigma_t)$,
$\Sigma_t=T_t\Sigma_1T_t^\top$, for all $t\in [0,1]$.

According to Lemma~\ref{lem:AU-factorization}, $T=T^*U$, so $\|T\|_{\rm op}\le \|T^*\|_{\rm op}\|U\|_{\rm op}<+\infty$. Hence for every $t\in[0,\|T-I\|_{\rm op}^{-1})$, the matrix $T_t=(1-t)I+tT$ is invertible. In what follows we work on the time interval $[0,\|T-I\|_{\rm op}^{-1}\wedge 1)$. On this interval,
the velocity field along the interpolation is the linear field
$v_t(y)=(T-I)T_t^{-1}y$.

The one-step quadratic \textit{c}-rectified flow is obtained by projecting this velocity onto the gradient subspace ${\mathcal G}_t={\mathcal S}_t^\perp\subset L^2(P_t;\R^d)$, where ${\mathcal S}_t$ denotes the $p_t$-divergence-free subspace. By Lemma~\ref{lem:linear-proj-gaussian}, projections of linear fields under Gaussian measures remain linear. Hence there exist matrices $G_t$ and $S_t$ such that
\[
v_t(z)=g_t(z)+s_t(z),
\qquad
 g_t(z):=\proj_{{\mathcal G}_t}v_t(z) = G_t z ,
\qquad
 s_t(z):=\proj_{{\mathcal S}_t}v_t(z) = S_t z,
\]
and define the \textit{c}-rectified flow equation by $Z_0 = X_0$, and
$\dot Z_t=g_t(Z_t) = G_t Z_t$, $t\in [0,1]$.
Consequently, the flow itself is linear: if a matrix flow ${\cal T}_t$ solves
$\dot {\cal T}_t=G_t{\cal T}_t$, $
{\cal T}_0=I$,
then $Z_t={\cal T}_tX_0$. In particular, one step of \textit{c}-rectified flow again produces a linear Gaussian coupling
\[
\crect((X_0,X_1)) = (Z_0,Z_1)=(X_0,{\cal T}_1X_0).
\]
Thus, in the Gaussian-linear setting, one step of \textit{c}-rectified flow can be implemented entirely at the matrix level. The construction is summarized below.

\begin{algorithm}[th]
\caption{One-step Gaussian \textit{c}-rectified flow}
\label{algo:gaussian-crect-onestep}
\begin{algorithmic}[1]
\State \textbf{Input:} 
% Gaussian marginals $P=\N(0,\Sigma_1)$, $Q=\N(0,\Sigma_2)$, and a linear transport map 
$T$ such that $T\Sigma_1T^\top=\Sigma_2$.
% \State Sample $X_0\sim P$ and set $X_1=TX_0$.
\For{$t\in[0,1]$}
\State % Define 
$T_t=(1-t)I+tT$, %and the interpolating state $X_t=T_tX_0$.
% \State 
compute %the marginal law $P_t=(T_t)_\#P=\N(0,\Sigma_t)$, where 
$\Sigma_t=T_t\Sigma_1T_t^\top$, and $V_t = (T-I)T_t^{-1}$.
% \State Form the linear velocity field $v_t(z)=(T-I)T_t^{-1}z$.
\State $V_t=G_t+S_t$ in $L^2(\N(0, \Sigma_t);\R^d)$ according to Lemma \ref{lem:linear-proj-gaussian}.
\State %Equivalently, write $g_t(z)=G_tz$ and ev
Solve 
$\dot{\cal T}_t=G_t{\cal T}_t$, with ${\cal T}_0=I$.
\EndFor
\Return ${\cal T}_1$
\end{algorithmic}
\end{algorithm}

\subsection{Proof of Remark~\ref{rmk:finite-step-crect}: Gaussian Case Study for \textit{c}-Rectified Flow}

\subsubsection{Two-dimensional Gaussian Case Study}
\label{sec:two-dimensional-gaussian-crect}

We now specialize the Gaussian calculation to dimension two.  In this case,
the rotational part of a linear transport is described by a single angle.  We
first study the iterates obtained from the independent coupling and then
compute one \textit{c}-rectified step starting from a deterministic linear
coupling.

Let $P=\N(0,\Sigma_1)$ and $Q=\N(0,\Sigma_2)$ on $\R^2$, where
$\Sigma_1,\Sigma_2\succ0$.  Throughout this section, we take the square loss, and for $\theta\in\R$ let
$U_\theta:=\begin{pmatrix}
\cos\theta&-\sin\theta\\
\sin\theta&\cos\theta
\end{pmatrix}$, and set
$J:=\begin{pmatrix}0&-1\\1&0\end{pmatrix}$.

We first determine the angle produced by the first \textit{c}-rectified step
from the independent coupling.  The calculation uses a convenient
description of a two-dimensional covariance matrix.  For $\Sigma=(\sigma_{ij})\succ0$,
write
\begin{equation}
\label{eq:2d-shape-coordinates}
\Sigma=\frac{\tr \Sigma}{2}
\begin{pmatrix}1+\beta&\alpha\\\alpha&1-\beta\end{pmatrix},
\qquad
\alpha=\frac{2\sigma_{12}}{\tr \Sigma},
\qquad
\beta=\frac{\sigma_{11}-\sigma_{22}}{\tr \Sigma},
\qquad
\gamma=\sqrt{1-\alpha^2-\beta^2}.
\end{equation}
Thus $(\alpha,\beta)$ belongs to the open unit disk and
$\gamma=2\sqrt{\det \Sigma}/\tr \Sigma$.

\begin{proposition}[Independent initialization and deterministic updates]
\label{prop:2d-finite-step}
We initialize Algorithm~\ref{algo:crect} from the
independent coupling
$(Z_0^{(0)},Z_1^{(0)})\sim P\otimes Q$. Then for $k\ge1$, $(Z_0^{(k-1)},Z_1^{(k-1)})$ is
\textit{c}-rectifiable, and define $(Z_0^{(k)},Z_1^{(k)}) = \crect (Z_0^{(k-1)},Z_1^{(k-1)})$.  Let $T^{(k)}$ be the linear transport
such that $Z_1^{(k)}=T^{(k)}Z_0^{(k)}$.  Then the following statements hold.
\begin{enumerate}[label=(\roman*)]
    \item The first iterate satisfies
    $T^{(1)}=T^*
    \,\,\Leftrightarrow\,\,
    \Sigma_1\Sigma_2=\Sigma_2\Sigma_1$.

    \item For every $k\ge1$, there is a unique
    $\theta^{(k)}\in(-\pi,\pi)$ such that
    $T^{(k)}\Sigma_1^{1/2}
    =T^*\Sigma_1^{1/2}U_{\theta^{(k)}}$.
    If $\Sigma_1$ and $\Sigma_2$ do not commute, then $T^{(k)}\ne T^*$, and the angles satisfy
    $0<|\theta^{(k+1)}|<|\theta^{(k)}|$.
\end{enumerate}
\end{proposition}

\begin{lemma}[A sufficient condition for Gaussian rectifiability]
\label{lem:2d-gaussian-rectifiability}
The independent coupling $P\otimes Q$ is $c$-rectifiable.
The same holds for a deterministic coupling $(X_0,TX_0)$ with
$X_0\sim P$, $TX_0\sim Q$, $\det T>0$, and $\tr T>0$.
In both cases, the conditional velocity and its gradient projection
are smooth linear fields with coefficients bounded on $[0,1]$.
The $c$-rectified ODE has a unique solution on $[0,1]$ for every
initial point and induces
a linear transport map with $\det \mathcal T_1 >0$.
\end{lemma}

\begin{proof}
Recall our notations $X_t=(1-t)X_0+tX_1$, $\Sigma_t=\Var(X_t)$, and
$v_t(x)=V_tx$.  For the independent coupling,
$\Sigma_t=(1-t)^2\Sigma_1+t^2\Sigma_2\succ0$ and
$V_t=\bigl(t\Sigma_2-(1-t)\Sigma_1\bigr)\Sigma_t^{-1}$.
For the deterministic coupling, let $T_t=(1-t)I+tT$.
Since for all $t\in[0,1]$,
$\det T_t=(1-t)^2+t(1-t)\tr T+t^2\det T>0$,
we have $\Sigma_t=T_t\Sigma_1T_t^\top\succ0$ and
$V_t=(T-I)T_t^{-1}$.  In either case, $\Sigma_t$ and $V_t$
are smooth on $[0,1]$, and $\Sigma_t$ is uniformly positive definite.

By Lemma~\ref{lem:linear-proj-gaussian}, the gradient projection
is $g_t(x)=G_tx$, where $G_t$ is symmetric.  Since
$\dot\Sigma_t=V_t\Sigma_t+\Sigma_tV_t^\top$ and the
divergence-free component contributes zero to this expression,
$G_t$ is the unique symmetric solution of
\begin{equation}
\label{eq:2d-projected-lyapunov}
G_t\Sigma_t+\Sigma_tG_t=\dot\Sigma_t.
\end{equation}
The Lyapunov operator is invertible for $\Sigma_t\succ0$, and
its inverse depends smoothly on $\Sigma_t$.  Thus $G_t$ is
smooth and bounded on $[0,1]$.  In particular, $g_t$ is globally
Lipschitz in $x$, uniformly in $t$.
Consequently, $\dot{\mathcal T}_t=G_t\mathcal T_t$ with
$\mathcal T_0=I$ has a unique solution on $[0,1]$, and
$Z_t=\mathcal T_tX_0$ is the unique $c$-rectified flow.
Liouville's formula gives
$\det\mathcal T_t=\exp(\int_0^t\tr G_s\,\diff s)>0$ for all $t\in[0,1]$, specifically, $\det \mathcal T_1>0$.
\end{proof}

\begin{lemma}[Positive trace after independent initialization]
\label{lem:2d-independent-positive-trace}
Let $T^{(1)}$ be the endpoint map of the first $c$-rectified
update from $P\otimes Q$.  Then $\det T^{(1)}>0$ and
$\tr T^{(1)}>0$.
\end{lemma}

\begin{proof}
Use the notation of Lemma~\ref{lem:2d-gaussian-rectifiability}
with $\Sigma_t=(1-t)^2\Sigma_1+t^2\Sigma_2$ and
$T^{(1)}=\mathcal T_1$.  That lemma already gives
$\det T^{(1)}>0$, so it remains to prove $\tr T^{(1)}>0$.

Let $(\alpha_t,\beta_t)$ be the shape coordinates of $\Sigma_t$
defined in \eqref{eq:2d-shape-coordinates}.
Without loss of generality, since the trace is invariant under conjugation, we may assume that there exists a common rotation $U_\varphi$ that rotates $G_t$, $\mathcal T_t$,
and $T^{(1)}$ simultaneously, such that $\beta_0=\beta_1$ and $\alpha_0\le \alpha_1$.
Normalizing $\Sigma_t$ by its trace shows that its shape
lies on the chord joining the endpoint shapes.  In particular,
$\beta_t=\beta_0$ for every $t$, while
\begin{equation*}
\alpha_t
=\frac{(1-t)^2\tr\Sigma_1\cdot \alpha_0+t^2\tr\Sigma_2\cdot \alpha_1}
{\tr\Sigma_t}\quad\implies\quad
\dot \alpha_t
=\frac{2t(1-t)\tr\Sigma_1\cdot\tr\Sigma_2}
{(\tr\Sigma_t)^2}(\alpha_1-\alpha_0)
\ge0.
\end{equation*}
Substituting the representation of $\Sigma_t$ into
\eqref{eq:2d-projected-lyapunov}, solving entrywise, and since $\dot \beta_t=0$, we obtain
\begin{equation*}
(G_t)_{12}=(G_t)_{21}
=\frac{(1-\beta_0^2)\dot \alpha_t}
{2(1-\alpha_t^2-\beta_0^2)}
\ge0.
\end{equation*}
The diagonal entries of $G_t$ are continuous
on $[0,1]$ and hence bounded below.  We may therefore take
$M=\max_{t\in[0,1],\,i=1,2}\max\{0,-(G_t)_{ii}\}$, so that
every entry of $G_t+MI$ is nonnegative.
Differentiate $e^{Mt}\mathcal T_t$ and use
$\dot{\mathcal T}_t=G_t\mathcal T_t$, we have
$\frac{\diff}{\diff t}(e^{Mt}\mathcal T_t)
=(G_t+MI)e^{Mt}\mathcal T_t$, and we can solve this integral equation by Picard iteration, which gives the series
\begin{equation*}
e^{Mt}\mathcal T_t
=I+\sum_{n=1}^{\infty}
\int_{0<s_n<\cdots<s_1<t}
(G_{s_1}+MI)\cdots(G_{s_n}+MI)
\,\diff s_n\cdots\diff s_1.
\end{equation*}
Every matrix product in this series is entrywise nonnegative.
Consequently,
the series after $I$ is an entrywise nonnegative matrix, so
$(e^{Mt}\mathcal T_t)_{ii}\ge1$ for $i=1,2$.
It follows that $(\mathcal T_t)_{ii}\ge e^{-Mt}>0$ and hence
$\tr T^{(1)}=\tr\mathcal T_1\ge2e^{-M}>0$, as claimed.
\end{proof}

\begin{proof}[Proof of Proposition~\ref{prop:2d-finite-step}]
By Lemma~\ref{lem:2d-gaussian-rectifiability}, the first update is
well-defined and is induced by an orientation-preserving linear
transport $T^{(1)}$.  Lemma~\ref{lem:2d-independent-positive-trace}
also gives $\tr T^{(1)}>0$.
We first prove part~(i), and then establish the existence and
rectifiability of all subsequent iterates together with part~(ii)
by induction.

\medskip
\noindent
\textit{Proof of part~(i).}
We first identify the angle produced by the initial update.

\medskip
\noindent
\textit{The commuting case.}
Let $X_0\sim P$ and $X_1\sim Q$ be independent, and set
$X_t=(1-t)X_0+tX_1$.  Then $X_t\sim\N(0,\Sigma_t)$, where
$\Sigma_t=(1-t)^2\Sigma_1+t^2\Sigma_2$.
The projected velocity has the form $g_t(x)=G_tx$, where $G_t$ is symmetric.
Since the divergence-free part does not change the marginal curve, $G_t$ is
the unique symmetric solution of 
$G_t\Sigma_t+\Sigma_tG_t=\dot\Sigma_t$.
The rectified flow is therefore the matrix flow
$\dot{\mathcal T}_t=G_t\mathcal T_t$, $\mathcal T_0=I$.

If $\Sigma_1$ and $\Sigma_2$ commute, they can be simultaneously
orthogonally diagonalized.  In this basis, write
$\Sigma_1=\diag(a_1,a_2)$ and $\Sigma_2=\diag(b_1,b_2)$. Then we can solve
$G_{t,i}=\frac12\frac{\diff}{\diff t}
\log\bigl((1-t)^2a_i+t^2b_i\bigr)$.
It follows that
$\mathcal T_{1,i}=\exp(\int_0^1G_{t,i}\diff t)=\sqrt{b_i/a_i}$, which is
the Gaussian optimal map in \eqref{eq:gauss-opt-map-linear}.  Thus
$T^{(1)}=T^*$ whenever $\Sigma_1$ and $\Sigma_2$ commute.

\medskip
\noindent
\textit{Rotation along a covariance path.}
We now derive the converse.  Consider any smooth positive definite covariance
path $S_t$ and an orientation-preserving factor $R_t$ satisfying
$R_tR_t^\top=S_t$.  Write
$R_t=S_t^{1/2}U_{\theta_t}$.  For a quadratic \textit{c}-rectified matrix
flow, the projected velocity is a gradient linear field $x\mapsto G_tx$,
so $G_t$ is symmetric.  The factor evolution satisfies
$\dot R_t=G_tR_t$, and hence $\dot R_tR_t^{-1}=G_t$.  It follows that
$R_t^\top\dot R_t=R_t^\top G_tR_t$ is symmetric.

Since
$R_t^\top\dot R_t=U_{\theta_t}^\top
\left(S_t^{1/2}\frac{\diff}{\diff t}S_t^{1/2}
+\dot\theta_tS_tJ\right)U_{\theta_t}$,
the symmetry condition is equivalent to
\begin{equation}
\label{eq:2d-path-symmetry}
\left[S_t^{1/2},\frac{\diff}{\diff t}S_t^{1/2}\right]
+\dot\theta_t(S_tJ+JS_t)=0.
\end{equation}
For a positive $2\times2$ matrix $S$, by the Cayley-Hamilton formula, we have 
$S^2-(\tr S)S+d^2I=0$, and hence
$S^{1/2}=\frac{S+\sqrt{\det S}\,I}
{\sqrt{\tr S+2\sqrt{\det S}}}$. By the chain rule, one can check that
$[S^{1/2},\dot S^{1/2}]
=\frac{[S,\dot S]}{\tr S+2\sqrt{\det S}}$.
Let $(\alpha_t,\beta_t)$ denote the coordinates of $S_t$ in
\eqref{eq:2d-shape-coordinates}.  Direct calculation gives
\[
[S_t,\dot S_t]
=-\frac{(\tr S_t)^2}{2}(\beta_t\dot \alpha_t-\alpha_t\dot \beta_t)J,
\qquad
S_tJ+JS_t=(\tr S_t)J.
\]
Substituting these identities into \eqref{eq:2d-path-symmetry} yields
\begin{equation}
\label{eq:2d-rotation-form}
\dot\theta_t
=\frac{\beta_t\dot \alpha_t-\alpha_t\dot \beta_t}{2\left(1+\sqrt{1-\alpha_t^2-\beta_t^2}\right)}.
\end{equation}

We briefly recall the differential-form notation used below.  On $\R^3$ with the canonical coordinate chart, a
one-form has the form
$\eta=a\,\diff \alpha+b\,\diff \beta+c\,\diff \gamma$.  Along a smooth curve
$q_t=(\alpha_t,\beta_t,\gamma_t)$, its value on the tangent vector is
$\eta_{q_t}(\dot q_t)=a(q_t)\dot \alpha_t+b(q_t)\dot \beta_t+c(q_t)\dot \gamma_t$,
and its line integral is the time integral of this quantity.  Its exterior
derivative is the two-form
\begin{equation*}
\diff\eta={}(\partial_\alpha b-\partial_\beta a)\,\diff \alpha\wedge\diff \beta
+(\partial_\alpha c-\partial_\gamma a)\,\diff \alpha\wedge\diff \gamma +(\partial_\beta c-\partial_\gamma b)\,\diff \beta\wedge\diff \gamma.
\end{equation*}
Here $\wedge$ is the antisymmetric wedge product; for example,
$\diff \gamma\wedge\diff \alpha=-\diff \alpha\wedge\diff \gamma$.  In the present problem,
$\gamma=\sqrt{1-\alpha^2-\beta^2}$ is determined by the independent coordinates $(\alpha,\beta)$.
Thus a one-form $\eta=a\,\diff \alpha+b\,\diff \beta$ on the open disk satisfies
$\diff\eta=(\partial_\alpha b-\partial_\beta a)\,\diff \alpha\wedge\diff \beta$.

Along the curve $t\mapsto(\alpha_t,\beta_t)$, the one-form
$\omega=\frac{\beta\,\diff \alpha-\alpha\,\diff \beta}
{2(1+\sqrt{1-\alpha^2-\beta^2})}$ satisfies
$\omega_{(\alpha_t,\beta_t)}(\dot \alpha_t,\dot \beta_t)=\dot\theta_t$ by
\eqref{eq:2d-rotation-form}.  Integrating in time therefore expresses the
accumulated rotation as the line integral of $\omega$ along this curve.

To compute its exterior derivative, write
$\gamma=\sqrt{1-\alpha^2-\beta^2}$.  Then
$\diff \gamma=-(\alpha\,\diff \alpha+\beta\,\diff \beta)/\gamma$,
$\diff\!\left(1/(2(1+\gamma))\right)
=(\alpha\,\diff \alpha+\beta\,\diff \beta)/(2\gamma(1+\gamma)^2)$, and
$\diff(\beta\,\diff \alpha-\alpha\,\diff \beta)=-2\,\diff \alpha\wedge\diff \beta$.
The product rule, followed by $\alpha^2+\beta^2=1-\gamma^2$, gives
\begin{equation}
\label{eq:2d-curvature}
\begin{aligned}
\diff\omega
&=-\left(
\frac{\alpha\,\diff \alpha+\beta\,\diff \beta}{2\gamma(1+\gamma)^2}+\frac{1}{1+\gamma}
\right)\diff \alpha\wedge\diff \beta =-\frac{\diff \alpha\wedge\diff \beta}{2\gamma}
=-\frac{\diff \alpha\wedge\diff \beta}
{2\sqrt{1-\alpha^2-\beta^2}}.
\end{aligned}
\end{equation}

\medskip
\noindent
\textit{The two covariance paths.}
Multiplying a covariance matrix by a positive scalar does not change its
$(\alpha,\beta)$-coordinates.  Here $\Sigma_t$ denotes the covariance of the
time-$t$ marginal along the first \textit{c}-rectified flow initialized from
the independent coupling.  Since
\[
\Sigma_t
=\bigl((1-t)^2+t^2\bigr)
\left[
\frac{(1-t)^2}{(1-t)^2+t^2}\Sigma_1
+\frac{t^2}{(1-t)^2+t^2}\Sigma_2
\right],
\]
the normalized covariance is a convex
combination of the two endpoints.
Thus the path $\Sigma_t$ traces the line segment joining the
two endpoint shapes in the $(\alpha,\beta)$-disk.

To describe the path that generates $T^*$, set
$R_0^*=\Sigma_1^{1/2}$, $R_1^*=T^*\Sigma_1^{1/2}$, and
$R_s^*=(1-s)R_0^*+sR_1^*$,  then this factor path
generates the optimal map and satisfies the same rotation rule
\eqref{eq:2d-rotation-form}.
Because
$(R_0^*)^\top R_1^*=\Sigma_1^{1/2}T^*\Sigma_1^{1/2}$ is symmetric positive
definite, $(R_s^*)^\top\dot R_s^*$ is symmetric. 

Next, we lift each shape point to the upper hemisphere by
$\widehat p(S)=(\alpha,\beta,\gamma)$. To make it precise, if
$R=\left(\begin{smallmatrix}a&b\\c&d\end{smallmatrix}\right)$ and
$S=RR^\top$, then the calculation gives
\begin{equation}
\label{eq:2d-shape-lift}
\widehat p(S)
=\frac{1}{a^2+b^2+c^2+d^2}
\bigl(2(ac+bd),\,2(ad-bc),\,a^2+b^2-c^2-d^2\bigr).
\end{equation}
We claim that the lifted optimal path is the great-circle arc joining
$\widehat p(\Sigma_1)$ and $\widehat p(\Sigma_2)$.  To verify the claim,
apply fixed left and right orthogonal transformations\footnote{Right multiplication does not change $RR^\top$, while left multiplication
rotates all lifted shape vectors by the same orthogonal transformation; thus
neither operation affects.} to the two factors so
that $R_0^*=\begin{pmatrix}a&0\\0&b\end{pmatrix}$, where $a,b>0$.
  The symmetry of
$(R_0^*)^\top R_1^*$ allows us to get the constraint that
$R_1^*=\begin{pmatrix}p&bq\\aq&r\end{pmatrix}$, and the positive definiteness of $(R_0^*)^\top R_1^*$ gives us $ap+br=\tr((R_0^*)^\top R_1^*)>0$.
Denote $P(R):=\|R\|_F^2\widehat p(RR^\top)$, then simple algebra gives
\[
P(R_s^*)
=(1-s)^2P(R_0^*)+2s(1-s)C+s^2P(R_1^*),
\]
where $C=\bigl(q(a^2+b^2),\,ar+bp,\,ap-br\bigr)$.
Notice that
\[
C
=\frac{\|R_1^*\|_F^2}{2(ap+br)}P(R_0^*)
+\frac{\|R_0^*\|_F^2}{2(ap+br)}P(R_1^*).
\]
Therefore, for $0<s<1$,
$P(R_s^*)$ is a positive linear combination of $P(R_0^*)$ and
$P(R_1^*)$, so $\widehat P(R_s^*)$ lies on the intersection of the upper hemisphere and the space $\operatorname{Span}({\widehat p(\Sigma_1),\widehat p(\Sigma_2)})$. In other words, the lifted path $\widehat P(R_s^*)$
is the shorter great-circle arc between the two distinct
endpoint lifts ${\widehat p(\Sigma_1),\widehat p(\Sigma_2)}$ on the upper hemisphere.

The projection that drops the $\gamma$-coordinate maps this circle to a centered
ellipse in the $(\alpha,\beta)$-plane.  It degenerates to a line precisely when its
restriction to the plane of the circle has rank one, or equivalently when
that plane contains the $\gamma$-axis.  This is the same as
$\alpha_1\beta_2-\beta_1\alpha_2=0$, namely, the two endpoint vectors $(\alpha_i,\beta_i)$ are
collinear with the origin.

For the endpoint coordinates $(\alpha_i,\beta_i)$ of $\Sigma_i$, direct multiplication
gives $[\Sigma_1,\Sigma_2]
=\frac{(\tr\Sigma_1)(\tr\Sigma_2)}{2}(\alpha_1\beta_2-\beta_1\alpha_2)J$.
Since $\Sigma_1$ and $\Sigma_2$ do not commute, we have then $[\Sigma_1,\Sigma_2]\neq 0$, and as a consequence, $\alpha_1\beta_2-\beta_1\alpha_2\neq 0$.  

Let $\gamma_{\mathrm{ind}}$ and $\gamma_*$ denote the line segment of the path $\Sigma_t$ and
the projected optimal arc of $R^*_t$, respectively, both oriented from $(\alpha_1,\beta_1)$ to
$(\alpha_2,\beta_2)$.  They bound a region $\Omega$ with nonempty interior in the
disk, and as an oriented curve we write
$\partial\Omega=\gamma_{\mathrm{ind}}-\gamma_*$.  Green's theorem and
\eqref{eq:2d-curvature} give
\begin{equation}
\label{eq:2d-first-angle-gap}
\begin{aligned}
\bigl|\theta_{\mathrm{ind}}-\theta_*\bigr|
&=\left|
\int_{\gamma_{\mathrm{ind}}}\omega-\int_{\gamma_*}\omega
\right| =\left|\int_{\partial\Omega}\omega\right|
=\left|\int_\Omega\diff\omega\right| =\int_\Omega
\frac{\diff x\,\diff z}{2\sqrt{1-x^2-z^2}}>0.
\end{aligned}
\end{equation}
The integral of the same positive density over the entire unit disk is
$\pi$, so the absolute value in \eqref{eq:2d-first-angle-gap} is strictly
less than $\pi$.  The endpoint factors therefore cannot differ by a full
rotation, and $T^{(1)}\ne T^*$ whenever the covariances do not commute.  This
proves the equivalence.

\medskip
\noindent
\textit{Proof of part~(ii).}
If the covariance matrices commute, part~(i) gives $T^{(1)}=T^*$.
Since $T^*$ is symmetric positive definite, $(1-t)I+tT^*$ is
invertible on $[0,1]$, and its conditional velocity matrix
$(T^*-I)((1-t)I+tT^*)^{-1}$ is symmetric.
The gradient projection therefore leaves this velocity unchanged,
and the resulting flow is $Z_t=((1-t)I+tT^*)X_0$.
Thus $T^{(k)}=T^*$ and $\theta^{(k)}=0$ for every $k\ge1$.

Assume now that the covariance matrices do not commute.
By \eqref{eq:2d-first-angle-gap}, the first iterate satisfies
$T^{(1)}\Sigma_1^{1/2}=T^*\Sigma_1^{1/2}U_{\theta^{(1)}}$
for some $0<|\theta^{(1)}|<\pi$.
Since $U_\theta=(\cos\theta)I+(\sin\theta)J$, the trace
$\tr(T^*\Sigma_1^{1/2}U_\theta\Sigma_1^{-1/2})$ is a
sinusoid in $\theta$.  At $\theta=0$ its value is $\tr T^*>0$,
and at $\theta=\theta^{(1)}$ its value is $\tr T^{(1)}>0$
by Lemma~\ref{lem:2d-independent-positive-trace}.
A nonzero sinusoid has simple zeros separated by $\pi$.
Since $|\theta^{(1)}|<\pi$, positivity at both endpoints
therefore implies
$\tr(T^*\Sigma_1^{1/2}U_\theta\Sigma_1^{-1/2})>0$
for $\theta$ between $0$ and $\theta^{(1)}$.

We now prove the claim by induction, starting from the
already constructed first iterate.
Suppose that $T^{(k)}$ has been constructed and satisfies
$T^{(k)}\Sigma_1^{1/2}
=T^*\Sigma_1^{1/2}U_{\theta^{(k)}}$,
where $\theta^{(k)}$ has the same sign as $\theta^{(1)}$ and
$0<|\theta^{(k)}|\le|\theta^{(1)}|$.
Then, we have
$\tr T^{(k)}>0$, while the factorization gives
$\det T^{(k)}=\det T^*>0$.
Lemma~\ref{lem:2d-gaussian-rectifiability} shows that this
coupling is $c$-rectifiable and its straight interpolation
is nonsingular throughout $[0,1]$.  Hence the next update exists.
For the following calculation, write $T=T^{(k)}$,
$T^+=T^{(k+1)}$, and $\theta=\theta^{(k)}$.
We will show that the updated angle lies strictly between
$0$ and $\theta$, which completes the induction.

Define
\begin{equation*}
H:=\Sigma_1^{1/2}T^*\Sigma_1^{1/2}
=\bigl(\Sigma_1^{1/2}\Sigma_2\Sigma_1^{1/2}\bigr)^{1/2},
\qquad
\rho:=\frac{\tr H}{\sqrt{(\tr\Sigma_1)(\tr\Sigma_2)}}.
\end{equation*}
Suppose $G \sim \mathcal{N}(0, I_2)$, and denote $X_0 = \Sigma_1^{1/2} G$, $X_1 = T X_0$. 
The straight interpolation $X_t = (1-t)X_0+tX_1$ gives $X_t = L_t G$, where
$L_t=(1-t)\Sigma_1^{1/2}+tT^*\Sigma_1^{1/2}U_\theta$.
Set $\Sigma_t=L_tL_t^\top$.  By Lemma~\ref{lem:linear-proj-gaussian} and
the Gaussian framework in Section~\ref{sec:crect-gaussian-framework}, the
projected velocity is $x\mapsto B_tx$, where $B_t$ is the unique symmetric
solution of $B_t\Sigma_t+\Sigma_tB_t=\dot\Sigma_t$.  The corresponding
\textit{c}-rectified flow is $\widetilde X_t=\widetilde L_tG$ with $\widetilde L_0=\Sigma_1^{1/2}$.

By the marginal-preserving property in
Theorem~\ref{thm:crect-marginal-preservation},
$\widetilde L_t\widetilde L_t^\top=L_tL_t^\top$,
so $L_t^{-1}\widetilde L_t$ is an orthogonal path starting from $I_2$.
Thus $L_t^{-1}\widetilde L_t=U_{\alpha_t}$ for some angle $\alpha_t$ with
$\alpha_0=0$.  Moreover,
$\widetilde L_t^\top\dot{\widetilde L}_t
=\widetilde L_t^\top B_t\widetilde L_t$ is symmetric.
Differentiating $\widetilde L_t=L_tU_{\alpha_t}$ gives
$\widetilde L_t^\top\dot{\widetilde L}_t
=U_{\alpha_t}^\top
\left(L_t^\top\dot L_t
+L_t^\top L_t\dot U_{\alpha_t}U_{\alpha_t}^\top\right)
U_{\alpha_t}$.
Therefore, using
$\dot U_{\alpha_t}U_{\alpha_t}^\top=\dot\alpha_tJ$ and $J^\top=-J$,
and equating the matrix inside the parentheses with its transpose, we obtain
\begin{equation}
\label{eq:2d-alpha-symmetry}
\dot\alpha_t
\bigl(L_t^\top L_tJ+JL_t^\top L_t\bigr)
=-\bigl(L_t^\top\dot L_t-\dot L_t^\top L_t\bigr).
\end{equation}
For every symmetric $2\times2$ matrix $C$, we have
$CJ+JC=(\tr C)J$.  A direct expansion also gives
$L_t^\top\dot L_t-\dot L_t^\top L_t
=HU_\theta-U_{-\theta}H=(\tr H)\sin\theta\,J$.
Substitution into \eqref{eq:2d-alpha-symmetry} yields
\begin{equation}
\label{eq:2d-alpha-ode}
\dot\alpha_t
=-\frac{(\tr H)\sin\theta}
{(\tr\Sigma_1)(1-t)^2+(\tr\Sigma_2)t^2
+2(\tr H)\cos\theta\,t(1-t)}.
\end{equation}
Since
$\widetilde L_1=T^*\Sigma_1^{1/2}U_{\theta+\alpha_1}$, the updated angle is
$\theta^+=\theta+\alpha_1$.  The change of variables
$u=\frac{\sqrt{\tr\Sigma_2}\,t}
{\sqrt{\tr\Sigma_1}\,(1-t)}$
in \eqref{eq:2d-alpha-ode} gives
\begin{equation}
\label{eq:2d-angle-recursion}
T^+\Sigma_1^{1/2}
=T^*\Sigma_1^{1/2}U_{\theta^+},
\qquad
\theta^+
=\theta-\rho\sin\theta
\int_0^\infty
\frac{\diff u}{u^2+2\rho\cos\theta\,u+1}.
\end{equation}

It remains to identify the range of $\rho$.  Since $H$ is a positive
$2\times2$ matrix, $(\tr H)^2=\tr(H^2)+2\det H
=\tr(\Sigma_1\Sigma_2)+2\sqrt{\det\Sigma_1\det\Sigma_2}$.
After orthogonally diagonalizing
$\Sigma_1=\diag(a_1,a_2)$ and writing $\Sigma_2=(b_{ij})$, by the
Cauchy--Schwarz inequality and the fact that
$b_{11}b_{22}\ge\det\Sigma_2$, we obtain
\begin{align*}
(\tr\Sigma_1)(\tr\Sigma_2)-(\tr H)^2
&=a_1b_{22}+a_2b_{11}
-2\sqrt{a_1a_2\det\Sigma_2}
\ge0.
\end{align*}
Equality holds if and only if
$b_{12}=0$ and $a_1b_{22}=a_2b_{11}$, which is equivalent to
$\Sigma_2=c\Sigma_1$ for some $c>0$.  Therefore, since $\Sigma_1$ and $\Sigma_2$ do not commute, we have
$0<\rho<1$.
% If $\rho=1$, the identity
% $\sin\theta\int_0^\infty
% \frac{\diff u}{u^2+2\cos\theta\,u+1}=\theta$, valid for
% $-\pi<\theta<\pi$,
% shows that $\theta^+=0$.  
If $0<\theta<\pi$, then for all $u\ge0$,
\[
0<
\frac{\rho}{u^2+2\rho\cos\theta\,u+1}
<
\frac{1}{u^2+2\cos\theta\,u+1}. 
\]
Integrating on both sides, we get
\[
0 < \int_0^\infty
\frac{\rho \diff u}{u^2+2\rho\cos\theta\,u+1} < \int_0^\infty
\frac{\diff u}{u^2+2\cos\theta\,u+1} = \frac{\theta}{\sin\theta}.
\]
Combining these relations with
\eqref{eq:2d-angle-recursion} gives
$0<\theta^+<\theta$.
The update in \eqref{eq:2d-angle-recursion} is odd in $\theta$,
so similarly
$\theta<\theta^+<0$ when $-\pi<\theta<0$.
In particular, $\theta^+$ lies strictly between $0$ and $\theta$.
Setting $\theta^{(k+1)}=\theta^+$, we see that it has the same
sign as $\theta^{(1)}$ and satisfies
$0<|\theta^{(k+1)}|<|\theta^{(k)}|\le|\theta^{(1)}|$.
This verifies the induction hypothesis at the next step.
Therefore, all iterates exist and satisfy
$\det T^{(k)}>0$,
$\tr T^{(k)}>0$,
and induce $c$-rectifiable couplings by
Lemma~\ref{lem:2d-gaussian-rectifiability}.

The angles remain in $(-\pi,\pi)$, where their representatives
are unique, and satisfy
$0<|\theta^{(k+1)}|<|\theta^{(k)}|$,
$\operatorname{sign}(\theta^{(k+1)})
=\operatorname{sign}(\theta^{(k)})$.
Since $\theta^{(1)}\ne0$, no subsequent angle vanishes.
Consequently, $T^{()}\ne T^*$ for every finite $k\ge1$.
This proves part~(ii) and the asserted rectifiability of all iterates.
\end{proof}

\subsubsection{Generic Dimensional  Gaussian Case Study}
\label{sec:general-dimensional-gaussian-crect}

The preceding two-dimensional calculation also gives a generic
non-recovery result in every dimension: for Lebesgue-almost every covariance
pair, the \textit{c}-rectified flow does not recover the Gaussian optimal map
in finite steps.  The argument no longer tracks the
rotational part explicitly.  Instead, for each fixed number of iterations,
we regard the discrepancy between the iterate and the Gaussian optimal map
as a real-analytic function of the two covariance matrices.

Let $\mathbb S_{++}^d$ denote the cone of $d\times d$ symmetric positive
definite matrices.  For each covariance pair in
$(\mathbb S_{++}^d)^2$, write $T^{(k)}$ for the linear transport after the
$k$-th quadratic \textit{c}-rectified update from the independent coupling,
and write $T^*$ for the Gaussian optimal map.

\begin{proposition}[Generic finite-step non-recovery]
\label{prop:general-d-finite-step}
Fix $d\ge2$.  For every finite $k\ge1$, the set
\begin{equation}
\label{eq:general-d-exceptional-set}
\mathcal U_k
:=
\left\{(\Sigma_1,\Sigma_2)\in(\mathbb S_{++}^d)^2:
T^{(k)}(\Sigma_1,\Sigma_2)=T^*(\Sigma_1,\Sigma_2)\right\}
\end{equation}
has Lebesgue measure zero in the space of pairs of symmetric matrices.
Consequently, for Lebesgue-almost every
$(\Sigma_1,\Sigma_2)\in(\mathbb S_{++}^d)^2$,
$T^{(k)}\ne T^*$
for every $k\ge1$.
\end{proposition}

\begin{proof}
We first verify the analytic dependence on the covariance matrices.  For the
first update from the independent coupling, the covariance path is
$\Sigma_t=(1-t)^2\Sigma_1+t^2\Sigma_2\succ0$, and the projected velocity
matrix $G_t$ is the unique symmetric solution of
$G_t\Sigma_t+\Sigma_tG_t=\dot\Sigma_t$.  The solution, regarded as the mapping of the inverse Lyapunov
operator, depends real-analytically on $\Sigma_t$.

At a subsequent deterministic step with coupling map $T$, write
$\Sigma_t=((1-t)I+tT)\Sigma_1((1-t)I+tT)^\top$ and
$V_t=(T-I)((1-t)I+tT)^{-1}$.  By Lemma~\ref{lem:linear-proj-gaussian}, the projected
velocity is $G_t=V_t-\Sigma_tK_t$, where the skew-symmetric matrix $K_t$ is
the unique solution of
$\Sigma_tK_t+K_t\Sigma_t=V_t-V_t^\top$.  For every such
\textit{c}-rectified update, $(1-t)I+tT$ remains invertible for
$0\le t\le1$.  Matrix
inversion and the inverse Lyapunov operator are therefore real analytic in
the parameters.  The endpoint of the linear equation
$\dot{\mathcal T}_t=G_t\mathcal T_t$, $\mathcal T_0=I$, also depends
real-analytically on the parameters.  Induction over the updates shows that
$(\Sigma_1,\Sigma_2)\mapsto T^{(k)}(\Sigma_1,\Sigma_2)$ is real analytic on
$(\mathbb S_{++}^d)^2$.  The same is true of $T^*$ because the positive
definite matrix square root is real analytic on $\mathbb S_{++}^d$.
Then,
$f_k(\Sigma_1,\Sigma_2):=
\|T^{(k)}(\Sigma_1,\Sigma_2)-T^*(\Sigma_1,\Sigma_2)\|_F^2$
is a nonnegative real-analytic function on the connected open set
$(\mathbb S_{++}^d)^2$.  We next show that it is not identically zero.

For $\varepsilon>0$, set
\begin{equation*}
A_\varepsilon=
\begin{pmatrix}1+\varepsilon&0\\0&1-\varepsilon\end{pmatrix},
\qquad
B_\varepsilon=
U_{\pi/4}
\begin{pmatrix}1+2\varepsilon&0\\0&1-2\varepsilon\end{pmatrix}
U_{\pi/4}^\top.
\end{equation*}
For sufficiently small $\varepsilon$, these matrices are positive definite,
do not commute, and are arbitrarily close to $I_2$.  In dimension $d>2$,
embed them by setting
$\Sigma_{1,\varepsilon}=\diag(A_\varepsilon,I_{d-2})$ and
$\Sigma_{2,\varepsilon}=\diag(B_\varepsilon,I_{d-2})$; for $d=2$, simply
take $\Sigma_{1,\varepsilon}=A_\varepsilon$ and
$\Sigma_{2,\varepsilon}=B_\varepsilon$.

We notice that the Gaussian covariance paths,
the Lyapunov equations, and the matrix ODE all preserve the block-diagonal
decomposition.  The leading $2\times2$ block therefore evolves exactly as in
Proposition~\ref{prop:2d-finite-step}.  Since $A_\varepsilon$ and
$B_\varepsilon$ do not commute, its proof through the $k$-th step gives
$T^{(k)}\ne T^*$.  Thus $f_k$ is not identically zero on
$(\mathbb S_{++}^d)^2$.

A nonzero real-analytic function on a connected open subset of Euclidean
space has a Lebesgue-measure zero set.  Applying this standard fact to $f_k$
proves that $\mathcal U_k$ is measure zero.  Finally, the set of covariance pairs for
which $T^{(k)}=T^*$ for at least one finite $k$ is
$\bigcup_{k\ge1}\mathcal U_k$, a countable union of measure zero sets, and is
therefore measure zero.
This finishes the proof.
\end{proof}

\subsection*{Quantitative Variational Lower Bound for One-step \textit{c}-Rectified Flow}

The goal of this subsection is to show that one step of \(c\)-rectified flow removes a definite amount of the non-gradient component of the initial velocity. The argument is variational: instead of computing the full projection \(\proj_{{\mathcal S}_t}V_t\), we test it against a carefully transported divergence-free field \(H_t\). This gives a lower bound on the projection norm that is uniform over a nontrivial time interval. % More precisely, we lower bound the cost decrease by a multiple of \(\|\Sigma_1K\|_P^2\), where \(\Sigma_1Kx\) is the \(S_0\)-projection of the initial velocity.

Let \({\mathcal S}_t\) denote the \(p_t\)-divergence-free subspace of \(L^2(P_t;\R^d)\), and let ${\mathcal G}_t:={\mathcal S}_t^\perp$.
By Lemma~\ref{lem:linear-proj-gaussian}, suppose \(K\) is the unique skew-symmetric matrix satisfying 
\begin{equation}\label{eq:def-proj-T-mat}
\Sigma_1 K + K \Sigma_1 = (T-I) - (T-I)^\top = T - T^\top,
\end{equation} 
the orthogonal projection of \(T-\id\) onto \({\mathcal S}_0\) is then given by
\[
u(x) = \proj_{{\mathcal S}_0}(T-\id)(x) = \Sigma_1 K x.
\]
Since every vector field in this Gaussian-linear setting is a time-dependent linear field, we identify a time-dependent mapping of the form
\(y\mapsto A_ty\) with its matrix \(A_t\).
% , and write
% \[
% \langle A,B\rangle_{P_t}:=\tr(A\Sigma_t B^\top),
% \qquad
% \|A\|_{P_t}^2:=\tr(A\Sigma_t A^\top).
% \]
Define
\[
V_t:=(T-I)T_t^{-1},
\qquad
H_t:=T_t \Sigma_1 K T_t^{-1},
\qquad
h_t(y)=JT_t(T_t^{-1}y)\,u(T_t^{-1}y)=H_t y.
\]
Since 
\[
\Sigma_t^{-1} H_t = (T_t \Sigma_1 T_t^\top)^{-1}T_t \Sigma_1 K T_t^{-1} = T_t^{-\top} K T_t^{-1} = - \left(T_t^{-\top} K T_t^{-1}\right)^{\top},
\]
in other words, $\Sigma_t^{-1} H_t$ is skew-symmetric, so
according to Lemma~\ref{lem:linear-proj-gaussian}, we have 
\(h_t\in {\mathcal S}_t\).

Applying the variational lower bound in Lemma \ref{lemma:variational-lb},
\[
\|\proj_{{\mathcal S}_t}V_t\|_{P_t}^2
\ge
\frac{\langle V_t,H_t\rangle_{P_t}^2}{\|H_t\|_{P_t}^2}.
\]

For the numerator,
\begin{align*}
\langle V_t,H_t\rangle_{P_t}
&=\E_{X\sim \N(0,\Sigma_1)}\big[(v_t(T_tX))^\top(h_t(T_tX))\big] =\E_{X\sim \N(0,\Sigma_1)}\big[((T-I)X)^\top(T_t \Sigma_1 K X)\big]  \\
&=\E_{X\sim \N(0,\Sigma_1)}\big[X^\top(T-I)^\top T_t \Sigma_1 K X\big]=\tr\big((T-I)^\top T_t \Sigma_1 K\Sigma_1\big)
=\langle {T_t^\top(T-I)}, {\Sigma_1 K}\rangle_{P}.
\end{align*}
Since
$T_t^\top(T-I)-(T-I)=t(T-I)^\top(T-I)$,
the matrix \(T_t^\top(T-I)-(T-I)\) is symmetric. On the other hand,
\((\Sigma_1 K\Sigma_1)^\top=-\Sigma_1 K\Sigma_1\), so \(\Sigma_1 K\Sigma_1\) is skew-symmetric. Hence, by \eqref{eq:matrix-inner-product},
\[
\langle {T_t^\top(T-I)}-{T-I}, {\Sigma_1 K}\rangle_{P}
=-\tr\big((T_t^\top(T-I)-(T-I))\cdot \Sigma_1 K\Sigma_1\big)=0.
\]
Therefore
\[
\langle V_t,H_t\rangle_{P_t}
=
\langle {T_t^\top(T-I)}, {\Sigma_1 K}\rangle_{P}
=
\langle {T-I}, {\Sigma_1 K}\rangle_{P}.
\]
Since \({\Sigma_1 K}x=\proj_{{\mathcal S}_0}(({T-I})x)\), we conclude that
$\langle V_t,H_t\rangle_{P_t}
% =
% \|f_{\Sigma_1 K}\|_{P}^2
=
\|\Sigma_1 K\|_{P}^2$.

Similarly, for the denominator, % using \eqref{eq:elliptic-linear},
\[
\|H_t\|_{P_t}^2
=
\E\|T_t \Sigma_1 K X\|^2
\le
\|T_t\|_{\rm op}^2\,\E\|\Sigma_1 KX\|^2
\le
4\|\Sigma_1 K\|_{P}^2.
\]
Hence for every \(t\in[0,
\|T-1\|_{\rm op}^{-1}\wedge 1)\),
\[
\|\proj_{{\mathcal S}_t}V_t\|_{P_t}^2
\ge
\frac{1}{4}\|\Sigma_1 K\|_{P}^2.
\]
Integrating over \(t\) and using the \textit{c}-rectified flow variational identity \eqref{eq:c-rect-variational} as well as Jensen's inequality yields 
\begin{align}\label{eq:main-reduction-linear}
\E\|X_1-X_0\|^2-\E\|Z_1-Z_0\|^2
% &\ge L_{\vv X,c}(f^{\vv X,c})\notag\\
& \ge \int_0^1 \E \|\dot X_s-g_s(X_s)\|^2
\diff s
\ge \int_0^1 \E \|\E[\dot X_s|X_s]-g_s(X_s)\|^2
\diff s \notag\\
&\ge \int_0^{\|T-I\|_{\rm op}^{-1}\wedge 1} \E \|v_s(X_s)-g_s(X_s)\|^2
\diff s = \int_0^{\|T-I\|_{\rm op}^{-1}\wedge 1}  \|\proj_{{\mathcal S}_t}V_t\|_{P_t}^2
\diff s \notag\\
&\ge 
\frac{\|T-1\|_{\rm op}^{-1}\wedge 1}{4}\|\Sigma_1 K\|_{P}^2 \ge \frac{1}{4(1+\|T^*\|_{\rm op})}\|\Sigma_1 K\|_{P}^2.
\end{align}

\subsection*{Spectral Criterion for the Projection Bound}

By Lemma~\ref{lem:AU-factorization},
$T-T^*=T^*(U-I)$.
Thus the excess over the optimal map is entirely encoded by the ``rotational'' part \(U-I\).

We formulate the projection hypothesis explicitly as follows.
The preceding estimate reduces the contraction problem to a purely algebraic question: whether the initial divergence-free component \(\Sigma_1Kx\) controls the deviation \(T-T^*\). The next proposition records this projection bound as the finite-dimensional condition needed for contraction.
\begin{proposition}[Explicit Gaussian projection bound]\label{prop:projbound-gauss-linear}
There exists \(c>0\) such that for all admissible \(T\) in Theorem \ref{thm:gaussian-contraction},
for \(K\) defined by \eqref{eq:def-proj-T-mat},
% the orthogonal projection of \((T-T^*)x\) onto \(S_0\) satisfies
\[
\|\proj_{{\mathcal S}_0}({T-I})\|_P^2
=
\|\Sigma_1 K\|_{P}^2
\ge
h_0\,\|T-T^*\|_{P}^2.
\]
\end{proposition}
It remains to give a checkable condition under which this projection bound holds. The natural condition is spectral: after whitening by \(\Sigma_1\), the rotational factor becomes an ordinary orthogonal matrix, and the relevant obstruction is the presence of eigenvalues close to \(-1\). The following proposition shows that excluding this obstruction quantitatively gives the desired bound.
\subsubsection*{Verification under a Spectral Gap}
To derive this bound from the spectrum of \(U\), introduce the whitened variables
\[
% W:=\Sigma_1^{1/2},\qquad 
\widetilde T^*:=\Sigma_1^{1/2}T^*\Sigma_1^{1/2}=(\Sigma_1^{1/2}\Sigma_2\Sigma_1^{1/2})^{1/2},\qquad \widetilde U:=\Sigma_1^{-1/2}U\Sigma_1^{1/2}\in {\mathcal O}(d).
\]
Then \(\widetilde U\) is orthogonal and \(\mathrm{Spec}(\widetilde U)=\mathrm{Spec}(U)\). Since \(T^* = \Sigma_1^{-1/2} \widetilde T^* \Sigma_1^{-1/2}\) and \(U = \Sigma_1^{1/2} \widetilde U \Sigma_1^{-1/2}\), we have \(T-T^* = T^*(U-I) = \Sigma_1^{-1/2} \widetilde T^*(\widetilde U-I) \Sigma_1^{-1/2}\), which yields:
\[
\Sigma_1^{1/2}(T-T^*)\Sigma_1^{1/2}=\widetilde T^*(\widetilde U-I).
\]

\begin{proposition}[Spectral criterion for the Gaussian projection bound]
\label{prop:spectral-criterion-gaussian}
Assume that \(\widetilde U\) has no \(-1\) eigenvalue, and write its nontrivial eigenvalues as
\(e^{\pm i\theta_1},\dots,e^{\pm i\theta_k}, \theta_j\in(0,\pi)\).
If $\dist(\Spec (U), -1) \ge \gamma > 0$,
% equivalently,
% \[
% \mathrm{dist}(-1,\mathrm{Spec}(U))>0,
% \]
then
\begin{equation}\label{eq:whitened-proj-bound}
\|\Skew(\widetilde T^*(\widetilde U-I))\|_F^2
\ge
\frac{\lambda_{\min}(T^*)^2}{\lambda_{\max}(T^*)^2} \frac{\gamma^2}{4} 
\|\widetilde T^*(\widetilde U-I)\|_F^2.
\end{equation}
Moreover, this implies the projection bound
\[
\|\Sigma_1 K\|_{P}^2 \ge h_0\,\|T-T^*\|_{P}^2
\]
with \(h_0 = \frac{\gamma^2}{4} \frac{\lambda_{\min}(T^*)^2 \lambda_{\min}(\Sigma_1)^2}{\lambda_{\max}(T^*)^2 \lambda_{\max}(\Sigma_1)^2}\), which verifies the bound in Proposition~\ref{prop:projbound-gauss-linear}.
\end{proposition}

\begin{proof}[Proof of Proposition~\ref{prop:spectral-criterion-gaussian}]
We split the proof into three steps.

\medskip
\noindent
\textbf{Step 1: Reduction to a block-diagonal orthogonal form.}
Choose \(V\in {\mathcal O}(d)\) such that
\[
V\widetilde U V^\top
=
\diag\big(R_{\theta_1},\dots,R_{\theta_k},I_{m_+}\big),
\]
where \(R_\theta\) is the standard \(2\times2\) rotation matrix. There is no \(-1\) block because \(-1\notin\mathrm{Spec}(\widetilde U)\).
Let \(\widetilde T^*_V:=V\widetilde T^* V^\top\). Then \(\widetilde T^*_V\) is symmetric positive definite, and dropping all off-block contributions only decreases the Frobenius norm, so
\[
\|\widetilde T^*\widetilde U-\widetilde U^\top \widetilde T^*\|_F^2
\ge
\sum_{j=1}^k
\|(\widetilde T^*_V)_jR_{\theta_j}-R_{\theta_j}^\top (\widetilde T^*_V)_j\|_F^2,
\]
where \((\widetilde T^*_V)_j\) is the \(2\times2\) principal block corresponding to the \(j\)-th rotation block.

\medskip
\noindent
\textbf{Step 2: Lower bound from the eigenangles.}
Writing \((\widetilde T^*_V)_j\) explicitly and taking the trace yields
\[
\|(\widetilde T^*_V)_jR_{\theta_j}-R_{\theta_j}^\top (\widetilde T^*_V)_j\|_F^2
\ge
8\lambda_{\min}(T^*)^2\sin^2\theta_j.
\]
Summing over \(j\) and using \(\Skew(\widetilde T^*(\widetilde U-I))=\frac12(\widetilde T^* \widetilde U-\widetilde U^\top \widetilde T^*)\), we obtain
\begin{equation}\label{eq:skew-lower-H}
\|\Skew(\widetilde T^*(\widetilde U-I))\|_F^2
\ge
2\lambda_{\min}(T^*)^2\sum_{j=1}^k\sin^2\theta_j.
\end{equation}
Since \(\|\widetilde T^*(\widetilde U-I)\|_F^2 \le \lambda_{\max}(T^*)^2\|\widetilde U-I\|_F^2\) and \(\|\widetilde U-I\|_F^2 = \sum_{j=1}^k 4(1-\cos\theta_j)\), we use \(\sin^2\theta_j = 2(1-\cos\theta_j)\cos^2\frac{\theta_j}{2}\) to conclude
\[
\|\Skew(\widetilde T^*(\widetilde U-I))\|_F^2
\ge
\frac{\lambda_{\min}(T^*)^2}{\lambda_{\max}(T^*)^2}
\left(\min_j\cos^2\frac{\theta_j}{2}\right)
\|\widetilde T^*(\widetilde U-I)\|_F^2.
\]
Finally, since $\dist(\Spec (U), -1) \ge \gamma$, for all $1\le j \le k$, we have
$\gamma^2 \le |e^{i\theta_j}+1|^2 = 2+2\cos \theta_j = 4 \cos^2 \frac{\theta_j}{2}$.
This proves \eqref{eq:whitened-proj-bound}.

\medskip
\noindent
\textbf{Step 3: Connecting the whitened skew bound to the \(L^2(P;\R^d)\) projection.}
We define the whitened matrix \(\widetilde{K} := \Sigma_1^{1/2} K \Sigma_1^{1/2}\). Because \(K\) is skew-symmetric and \(\Sigma_1^{1/2}\) is symmetric, \(\widetilde{K}\) is also skew-symmetric. 
Furthermore, by the definition of \(K\) in \eqref{eq:def-proj-T-mat},
\begin{align}
\Sigma_1 \widetilde{K} + \widetilde{K} \Sigma_1&=\Sigma_1^{1/2}(\Sigma_1 K + K \Sigma_1)\Sigma_1^{1/2} = \Sigma_1^{1/2}(T-T^*)\Sigma_1^{1/2} - \Sigma_1^{1/2}(T-T^*)^\top \Sigma_1^{1/2}\\\notag
&=\widetilde T^*(\widetilde U-I) - (\widetilde T^*(\widetilde U-I))^\top = 2\Skew(\widetilde T^*(\widetilde U-I)).
\end{align}
% Thus, \(\widetilde{K}\) satisfies a whitened Lyapunov equation:
% \[
% \Sigma_1 \widetilde{K} + \widetilde{K} \Sigma_1 = 2\Skew(\widetilde T^*(\widetilde U-I)).
% \]
Taking the Frobenius norm of both sides and applying the triangle inequality yields:
\[
2\|\Skew(\widetilde T^*(\widetilde U-I))\|_F = \|\Sigma_1 \widetilde{K} + \widetilde{K} \Sigma_1\|_F \le 2\lambda_{\max}(\Sigma_1)\|\widetilde{K}\|_F \implies \|\widetilde{K}\|_F^2 \ge \frac{\|\Skew(\widetilde T^*(\widetilde U-I))\|_F^2}{\lambda_{\max}(\Sigma_1)^2}.
\]

Now we compute the actual projection norm in \(L^2(P;\R^d)\). Substituting \(K = \Sigma_1^{-1/2}\widetilde{K}\Sigma_1^{-1/2}\), we have:
\begin{align*}
\|\Sigma_1 K\|_{P}^2 &= \tr(\Sigma_1 K \Sigma_1 (\Sigma_1 K)^\top) = -\tr(\Sigma_1 K \Sigma_1 K \Sigma_1) 
% &= -\tr(\Sigma_1\cdot \Sigma_1^{-1/2}\widetilde{K}\Sigma_1^{-1/2}\cdot \Sigma_1 \cdot \Sigma_1^{-1/2}\widetilde{K}\Sigma_1^{-1/2}\cdot \Sigma_1) \\
= -\tr(\Sigma_1^{1/2} \widetilde{K} \widetilde{K} \Sigma_1^{1/2}) = \tr(\widetilde{K} \widetilde{K}^\top \Sigma_1).
\end{align*}
Because \(\widetilde{K} \widetilde{K}^\top\) is positive semidefinite, \(\tr(\widetilde{K} \widetilde{K}^\top \Sigma_1) \ge \lambda_{\min}(\Sigma_1) \tr(\widetilde{K} \widetilde{K}^\top) = \lambda_{\min}(\Sigma_1) \|\widetilde{K}\|_F^2\).
Combining this with our Frobenius bound gives:
\[
\|\Sigma_1 K\|_{P}^2 \ge \frac{\lambda_{\min}(\Sigma_1)}{\lambda_{\max}(\Sigma_1)^2} \|\Skew(\widetilde T^*(\widetilde U-I))\|_F^2.
\]

Finally, by \eqref{eq:whitened-proj-bound}, observing that \begin{align*}
\|\widetilde T^*(\widetilde U-I)\|_F^2 &= \|\Sigma_1^{1/2}(T-T^*)\Sigma_1^{1/2}\|_F^2 = \tr(\Sigma_1(T-T^*)\Sigma_1(T-T^*)^\top) \\
&\ge \lambda_{\min}(\Sigma_1) \tr((T-T^*)\Sigma_1(T-T^*)^\top) = \lambda_{\min}(\Sigma_1) \|T-T^*\|_{P}^2,
\end{align*}
we conclude:
\begin{align*}
\|\Sigma_1 K\|_{P}^2 
&\ge \frac{\gamma^2}{4} \frac{\lambda_{\min}(\Sigma_1) \lambda_{\min}(T^*)^2}{\lambda_{\max}(\Sigma_1)^2 \lambda_{\max}(T^*)^2} \|\widetilde T^*(\widetilde U-I)\|_F^2 \ge \frac{\gamma^2}{4} \frac{\lambda_{\min}(\Sigma_1)^2 \lambda_{\min}(T^*)^2}{\lambda_{\max}(\Sigma_1)^2 \lambda_{\max}(T^*)^2} \|T-T^*\|_{P}^2.
\end{align*}
This strictly bounds the true projection in \(L^2(P; \R^d)\) and verifies Proposition~\ref{prop:projbound-gauss-linear}.
\end{proof}
We now return from the whitened matrix estimate to the original transport problem. 
Since \(T^*\) is symmetric, the linear function \(T^* x\) is a gradient field. Hence, the vector field \(u(x) = \Sigma_1 K x\) defined via \(\proj_{{\mathcal S}_0}(T-I)\) is identically \(\proj_{{\mathcal S}_0}(T-T^*)\).

Therefore, once Proposition~\ref{prop:projbound-gauss-linear} is established, we immediately obtain
$\|\Sigma_1 K\|_{P}^2
\ge
h_0\,\|T-T^*\|_{P}^2$ for some constant $h_0>0$ that only depends on $\gamma$ and second moments of $P$, $Q$.
Substituting this directly into \eqref{eq:main-reduction-linear} yields
\[
\E\|X_1-X_0\|^2-\E\|Z_1-Z_0\|^2
\ge
\frac{h_0}{4(1+\lambda_{\max}(T^*))}\|T-T^*\|_{P}^2 : = \tilde h \|T-T^*\|_{P}^2.
\]

\subsection{Proof of Theorem~\ref{thm:gaussian-contraction} and Theorem~\ref{thm:exp-conv-crect}}

The final ingredient is to compare the matrix distance \(\|T-T^*\|_P^2\) with the excess quadratic transport cost. For Gaussian linear transports, this comparison is exact up to multiplication by \((T^*)^{-1}\), and it allows us to convert the projection lower bound into the desired contraction estimate.

\begin{lemma}\label{lem:gaussian-cost-difference}
%Let $X\sim P=\N(0,\Sigma_1)$, and let $T$ be a linear transport map from $P$ to $Q$. Then
For $P=\N(0,\Sigma_1)$, $Q=\N(0,\Sigma_2)$, and a transport map $T$ from $P$ to $Q$, we have
\[
\E\|T(X)-X\|^2-W_2^2(P,Q)
=
\tr\Big((T^*)^{-1}(T-T^*)\Sigma_1(T-T^*)^\top\Big)
\le
\frac{1}{\lambda_{\min}(T^*)}\|T-T^*\|_{P}^2.
\]
\end{lemma}

\begin{proof}[Proof of Lemma~\ref{lem:gaussian-cost-difference}]
% Set
% \[
% A:=T-T^*.
% \]
Since both $T$ and $T^*$ push $P$ forward to $Q$, we have
$T\Sigma_1T^\top=T^*\Sigma_1(T^*)^\top$.
Because $T^*$ is symmetric, this becomes
% \[
% (T^*+(T-T^*))\Sigma_1(T^*+(T-T^*))^\top=T^*\Sigma_1T^*,
% \]
% hence
\begin{equation}\label{eq:gaussian-covariance-identity}
(T-T^*)\Sigma_1(T-T^*)^\top+(T-T^*)\Sigma_1T^*+T^*\Sigma_1(T-T^*)^\top=0.
\end{equation}

On the other hand, since $T(X)$ and $T^*(X)$ have the same law, %they have the same second moment, so
$\E\|T(X)\|^2=\E\|T^*(X)\|^2$.
Therefore,
\begin{align*}
\E\|T(X)-X\|^2-W_2^2(P,Q)
&=\E\|T(X)-X\|^2-\E\|T^*(X)-X\|^2 \\
&=-2\E\langle (T-T^*)X,X\rangle
=-2\tr((T-T^*)\Sigma_1).
\end{align*}
Now multiply \eqref{eq:gaussian-covariance-identity} on the left by $(T^*)^{-1}$ and take traces. This gives
\begin{align*}
0
&=\tr\big((T^*)^{-1}(T-T^*)\Sigma_1(T-T^*)^\top\big)+\tr((T-T^*)\Sigma_1)+\tr(\Sigma_1(T-T^*)^\top).
\end{align*}
Since $\tr(\Sigma_1(T-T^*)^\top)=\tr((T-T^*)\Sigma_1)$, we obtain
$\tr\big((T^*)^{-1}(T-T^*)\Sigma_1(T-T^*)^\top\big)=-2\tr((T-T^*)\Sigma_1)$.
Combining this with the previous identity yields the exact formula
\[
\E\|T(X)-X\|^2-W_2^2(P,Q)
=
\tr\Big((T^*)^{-1}(T-T^*)\Sigma_1(T-T^*)^\top\Big).
\]
Finally, since $T^*\succeq \lambda_{\min}(T^*) I$, we have $(T^*)^{-1}\preceq \lambda_{\min}(T^*)^{-1}I$, and therefore
\[
\tr\Big((T^*)^{-1}(T-T^*)\Sigma_1(T-T^*)^\top\Big)
\le
\frac{1}{\lambda_{\min}(T^*)}\tr\Big((T-T^*)\Sigma_1(T-T^*)^\top\Big)
=
\frac{1}{\lambda_{\min}(T^*)}\|T-T^*\|_P^2.
\]
This proves the lemma.
\end{proof}
% \[
% \E\|T(X)-X\|^2-W_2^2(P,Q)
% \le
% \frac{1}{\lambda_{\min}(T^*)}\|T-T^*\|_{P}^2,,
% \]
% proved exactly as before via the quadratic potential
% \[
% \phi(x)=\frac12 x^\top T^* x.
% \]
Combining the two inequalities gives the contraction estimate
\begin{proof}[Proof of Theorem~\ref{thm:gaussian-contraction}]
From \eqref{eq:main-reduction-linear} and Proposition~\ref{prop:projbound-gauss-linear}, we already proved that
\[
\E\|X_1-X_0\|^2-\E\|Z_1-Z_0\|^2
\ge
\tilde{h}\|T-T^*\|_{P}^2.
\]
By Lemma~\ref{lem:gaussian-cost-difference},
$\|T-T^*\|_{P}^2
\ge
\lambda_{\min}(T^*)\Bigl(\E\|X_1-X_0\|^2-W_2^2(P,Q)\Bigr)$.
Substituting this lower bound into the previous inequality yields
\[
\E\|X_1-X_0\|^2-\E\|Z_1-Z_0\|^2
\ge
\tilde{h}\cdot \lambda_{\min}(T^*)
\Bigl(\E\|X_1-X_0\|^2-W_2^2(P,Q)\Bigr).
\]
Rearranging gives the contraction
\[
\E\|Z_1-Z_0\|^2-W_2^2(P,Q)
\le
(1-\tilde{h}\cdot \lambda_{\min}(T^*))
\Bigl(\E\|X_1-X_0\|^2-W_2^2(P,Q)\Bigr).
\]
Finally, % redefine $c = \tilde{c}\cdot \lambda_{\min}(T^*)$, then 
\[
\tilde{h}\cdot \lambda_{\min}(T^*) = \frac{\lambda_{\min}(T^*)^3 \lambda_{\min}(\Sigma_1)^2\cdot \gamma^2 }{16{\lambda_{\max}(T^*)^2 \lambda_{\max}(\Sigma_1)^2}(1+\lambda_{\max}(T^*))} := \gamma^2 h,
\]
where $h = C(\Spec(\Sigma_1),\Spec(\Sigma_2))$ is a constant that only depends on the spectrum of $\Sigma_1$ and $\Sigma_2$, i.e.~the second moments of $P$ and $Q$.
\end{proof}

\begin{proof}[Proof of Theorem~\ref{thm:exp-conv-crect}]
Define the excess cost
\[
\mathcal E_k:=
\E\left\|Z_1^{(k)}-Z_0^{(k)}\right\|^2
- W_2^2(P,Q).
\]
Since $W_2^2(P,Q)$ is the optimal quadratic transport cost, we have $\mathcal E_k\ge 0$ for all $k$.

Whichever coupling we are starting from, as long as it is rectifiable,
By Proposition \ref{prop:gaussian-joint-velocity},
the velocity field of the straight coupling is always a linear field, as we have argued in Section \ref{sec:crect-gaussian-framework}. Therefore, after running one-step \textit{c}-rectified flow, we will get a linear transport coupling, i.e.~there exists a linear transport map $T$ such that $(Z_0^{(1)},Z_1^{(1)}) = (Z_0^{(1)}, T Z_0^{(1)})$. Similarly,
for all $k\ge 1$, there exists a $T_k$ such that $Z_1^{(k)} = T_k Z_0^{(k)}$ and $T_k = T^* U_k$ where $U_k$ is a $\Sigma_1$-orthogonal matrix by Lemma~\ref{lem:AU-factorization}.

We will show that there exists a $\gamma > 0$ and $K \in \mathbb{N}$ such that for all $k\ge K$, $\dist(\Spec (U_k), -1) \ge \gamma$.
Suppose on the contrary that there exists a subsequence $k_n$ such that $\dist (\Spec (U_{k_n}), -1) < n^{-1}$. According to the proof of Theorem~\ref{thm:convergence-square}, $\{(Z_0^{(k)},Z_1^{(k)})\}_{k\in\mathbb N}$ is tight, and so $\{(Z_0^{(k_n)},Z_1^{(k_n)})\}_{n\in\mathbb N}$ is tight, there is a further subsequence of $k_n$, still denoted by the index $k$, such that 
\[
(Z_0^{(k)},Z_1^{(k)})\xrightarrow{w} (Z_0^*,Z_1^*).
\]
By the continuity property, $Z_1^{(k)} - T^*Z_0^{(k)}\xrightarrow{w} Z_1^*-T^*Z_0^* = 0$. Since $Z_1^{(k)} = T^* U_k Z_0^{(k)}$, this further shows
\[
T^*(U_k -  I)Z_0^{(k)} \xrightarrow{w} 0.
\]
Since $T^*$ is invertible, and $Z_0^{(k)} \sim P$ is a nondegenerate Gaussian distribution on the full space for every $k$, it follows that $U_k - I \to 0$, i.e.~$U_k \to I$. Therefore, $\dist(\Spec (U_k), -1) \to 2$, contradiction!

By Theorem~\ref{thm:gaussian-contraction}, for every $k\ge 2$, once there is a uniform spectral gap on the rotational part of $T$, there is a one-step contraction for some uniform $h>0$ that
$\mathcal E_{k+1}\le (1-h)\mathcal E_k$.
Iterating this bound yields
\[
\mathcal E_k\le (1-h)^{k-1}\mathcal E_1\le (1-h)^{k-1}\mathcal E_0 \le e^{-h(k-1)} \mathcal E_0,
\]
finishing the proof.
\end{proof}

\section{Proof of Section~\ref{sec:general-contraction-criterion}}
\label{app:general-contraction-criterion-proof}
\begin{proof}[Proof of Proposition~\ref{prop:Tt-invertible}]
For any $x,y\in\R^d$,
$T_t(x)-T_t(y)=(1-t)(x-y)+t\bigl(T(x)-T(y)\bigr)$.
By the triangle inequality and the Lipschitz bound on $T$, for all $x,y\in\R^d$,
\[
\bigl(1-(1+L)t\bigr)\|x-y\|
\le
\|T_t(x)-T_t(y)\|
\le
\bigl((1-t)+tL\bigr)\|x-y\|.
\]
When $t<(1+L)^{-1}$, the lower bound is strictly positive for $x\neq y$, so $T_t$ is injective. The same lower bound implies the Lipschitz estimate for the inverse on the image $T_t(\R^d)$.
\end{proof}

\subsection*{Projection Estimates Along the Interpolation}
The proof of Theorem~\ref{thm:main_bound} is based on testing the solenoidal
part of \(v_t\) against a transported initial divergence-free field. The next
lemma shows that this test field remains divergence-free under the
interpolation and that its pairing with \(v_t\) is independent of \(t\).
To utilize the variational bound, we construct a specific test function in ${\mathcal S}_t$.

\begin{lemma}[Transported divergence-free fields]\label{lemma:ht_construction}
    Let $u \in {\mathcal S}_0$. Define the vector field $h_t$ on $\mathbb{R}^d$ by
    \[
    h_t(x) = J T_t(T_t^{-1}(x)) \, u(T_t^{-1}(x)).
    \]
    Then $h_t \in {\mathcal S}_t$.
\end{lemma}
\begin{proof}[Proof of Lemma~\ref{lemma:ht_construction}]
    Let $g \in C_c^\infty(\mathbb{R}^d)$ be an arbitrary smooth test function. Using the change of variables $y = T_t(x)$ (noting that $p_t$ is the pushforward of $p$ by $T_t$), we compute:
    \begin{align*}
        \int \nabla g(y)^\top h_t(y) p_t(y)  \diff y 
        &= \int \nabla g(T_t(x))^\top J T_t(x) u(x) p(x)  \diff x = \int \nabla (g \circ T_t)(x)^\top u(x) p(x)  \diff x.
    \end{align*}
    Since $u \in {\mathcal S}_0$, it is orthogonal to gradients of scalar functions in $L^2(P;\R^d)$, so the integral vanishes. Thus, $\nabla \cdot (p_t h_t) = 0$ in the weak sense.
\end{proof}

The following lemma identifies the \({\mathcal S}_0\)-projection of the field appearing in the
pairing with \(v_t\).

\begin{lemma}[Projection identity along the interpolation]\label{lemma:projection_identity}
    It holds that
    \[
    \proj_{{\mathcal S}_0}\left( J T_t^\top (T-\id) \right) = \proj_{{\mathcal S}_0}(T-\id) = \proj_{{\mathcal S}_0}(T).
    \]
\end{lemma}
\begin{proof}[Proof of Lemma~\ref{lemma:projection_identity}]
    Recall that $L^2(P;\R^d)$ admits the orthogonal decomposition $L^2(P;\R^d) = {\mathcal S}_0 \oplus G$, where $G$ is the space of gradient fields. Since $T_t = (1-t)\id + tT$, its Jacobian is $JT_t = I + t(J T - I)$. 
    
    We observe that
    \(
    JT_t(x)^\top (T(x)-x) = (T(x)-x) + t (J T(x) - I)^\top (T(x) - x).
    \)
    Since $(J T(x) - I)^\top (T(x) - x) = \frac{1}{2} \nabla \|T(x) - x\|^2$, therefore,
    \[
    JT_t^\top (T-\id) = (T - \id) + \nabla \left( \frac{t}{2} \|T - \id\|^2 \right).
    \]
    The second term is a gradient field and thus lies in ${\mathcal S}_0^\perp$. Consequently,
    \[
    \proj_{{\mathcal S}_0}(JT_t^\top (T-\id)) = \proj_{{\mathcal S}_0}(T - \id) = \proj_{{\mathcal S}_0}(T) - \proj_{{\mathcal S}_0}(\id).
    \]
    Since $\id = \nabla (\frac{1}{2}\|x\|^2)$ is a gradient, $\proj_{{\mathcal S}_0}(\id) = 0$, completing the proof.
\end{proof}

\begin{lemma}[Transport Stability]\label{lemma:cost_difference}
    Under Assumption \ref{ass:convexity}, the difference in quadratic costs satisfies:
    \[
    \E\left\|T(X_0)-X_0\right\|^2 - W_2^2(P,Q) \le \frac{1}{\lambda} \|T-T^*\|_{P}^2.
    \]
\end{lemma}
\begin{proof}[Proof of Lemma~\ref{lemma:cost_difference}]
    We expand the squared norm $\|T(X_0)-X_0\|^2 = \|T(X_0)\|^2 + \|X_0\|^2 - 2\langle T(X_0), X_0 \rangle$. Noting that $\|T(X_0)\|^2$ and $\|T^*(X_0)\|^2$ have the same expectation (the second moment of $Q$), the difference in costs becomes:
    \[
    \E\|T(X_0)-X_0\|^2 - \E\|T^*(X_0)-X_0\|^2 = 2 \E\langle T^*(X_0) - T(X_0), X_0 \rangle .
    \]
    We substitute $X_0 = \nabla \phi_0^*(T^*(X_0))$ using convex duality. Since $\E[\phi_0^*(T(X_0))] = \E[\phi_0^*(T^*(X_0))]$, we can rewrite the term on the right-hand side as a Bregman divergence:
    \[
    2 \E\langle T^*(X_0) - T(X_0), \nabla \phi_0^*(T^*(X_0)) \rangle = 2 \E\left[ D_{\phi_0^*}(T(X_0), T^*(X_0)) \right].
    \]
    By Assumption \ref{ass:convexity}, $\phi_0$ is $\lambda$-strongly convex, which implies $\phi_0^*$ is $(1/\lambda)$-smooth. This smoothness yields the bound $D_{\phi_0^*}(y, z) \le \frac{1}{2\lambda}\|y-z\|^2$. Therefore,
    \[
    \E\|T(X_0)-X_0\|^2 - W_2^2(P,Q) \le \frac{1}{\lambda} \E\left\|T(X_0) - T^*(X_0)\right\|^2,
    \]
    which concludes the proof.
\end{proof}

\subsection{Proof of Theorem~\ref{thm:main_bound}: Contraction Criterion}
Now we prove Theorem~\ref{thm:main_bound} and its contraction corollary.

\begin{proof}[Proof of Theorem~\ref{thm:main_bound}]
    Using the \textit{c}-rectified flow variational identity \eqref{eq:c-rect-variational} and Jensen's inequality, the cost reduction is bounded by the integrated projection of the velocity field:
    \begin{align*}
    \E\|T(X_0)-X_0\|^2 - \E\|Z_1-Z_0\|^2 &\ge \inf_f\int_0^1 \E \|\dot X_s-\nabla f_s(X_s)\|^2 \diff s \ge \inf_f \int_0^1 \E \|\E[\dot X_s|X_s]-\nabla f_s(X_s)\|^2
\diff s \\
 & = \inf_f \int_0^1  \|v_s-\nabla f_s\|_{P_s}^2
\diff s = \int_0^1  \|v_s-\proj_{G_s}  v_s\|_{P_s}^2
\diff s \\
&\ge \int_0^{(2+2L)^{-1}} \|v_s - \proj_{S_s^\perp} v_s\|_{P_s}^2 \diff s.
    \end{align*}
    Let $w_t = \proj_{{\mathcal S}_t} v_t$. By Lemma \ref{lemma:variational-lb}, for any test function $h_t \in {\mathcal S}_t$,
    \[
    \|w_t\|_{P_t}^2 \ge \frac{\langle v_t, h_t \rangle_{P_t}^2}{\|h_t\|_{P_t}^2}.
    \]
    We choose $h_t$ as in Lemma \ref{lemma:ht_construction}, setting $u = \proj_{{\mathcal S}_0}(T)$.
    
    Using the change of variables $y = T_t(x)$ and Lemma \ref{lemma:projection_identity}, we compute the numerator:
    \begin{align*}
        \langle v_t, h_t \rangle_{P_t} &= \int (v_t \circ T_t)^\top (h_t \circ T_t) \, p(x)  \diff x = \int (T(x)-x)^\top J T_t(x) u(x) \, p(x)  \diff x \\
        &= \langle J T_t^\top(T-\id), u \rangle_P = \langle \proj_{{\mathcal S}_0}(J T_t^\top(T-\id)), u \rangle_P \\
        &= \langle \proj_{{\mathcal S}_0}(T), u \rangle_P = \|u\|_P^2.
    \end{align*}

    For the denominator, Assumption \ref{ass:lipschitz} implies $\|JT\|_{op} \le L$. The Jacobian $JT_t = (1-t)I + tJ T$ therefore satisfies $\|JT_t(x)\|_{op} \le 2$ for all $t\le (2L+2)^{-1}$. Consequently:
    \[
    \|h_t\|_{P_t}^2 = \int \|JT_t(x) u(x)\|^2 p(x)  \diff x \le 4 \|u\|_P^2.
    \]
    
    Combining these estimates gives:
    \[
    \|w_t\|_{P_t}^2 \ge \frac{(\|u\|_P^2)^2}{4 \|u\|_P^2} = \frac{1}{4} \|u\|_P^2.
    \]
    Recalling that $u = \proj_{{\mathcal S}_0}(T) = \proj_{{\mathcal S}_0}(T-T^*)$ (since $T^* = \nabla \phi \perp {\mathcal S}_0$), and invoking Assumption \ref{ass:projection_bound}:
    \[
    \|w_t\|_{P_t}^2 \ge \frac{h}{4} \|T-T^*\|_P^2.
    \]
    Integrating over $t \in [0,(2+2L)^{-1}]$ yields
    \begin{align*}
    \E\|T(X_0)-X_0\|^2 - \E\|Z_1-Z_0\|^2 &\ge \int_0^{(2+2L)^{-1}} \|w_t\|_{P_t}^2 \diff t \ge  \frac{h}{8(L+1)} \|T-T^*\|_P^2 .
    \end{align*}
\end{proof}

\begin{proof}[Proof of Corollary~\ref{cor:one-step-contraction-general-L2}]
    Theorem \ref{thm:main_bound} provides the lower bound for the cost reduction:
    \[
     \E\left\|T(X_0)-X_0\right\|^2 - \E\left\|Z_1-Z_0\right\|^2 \ge \frac{h}{8(L+1)} \|T-T^*\|_{P}^2.
    \]
    By Lemma \ref{lemma:cost_difference}, we have $\|T-T^*\|_{P}^2 \ge \lambda \left(\E\|T(X_0)-X_0\|^2 - W_2^2(P,Q)\right)$.
    Substituting this into the inequality yields:
    \[
    \E\left\|T(X_0)-X_0\right\|^2 - \E\left\|Z_1-Z_0\right\|^2 \ge \frac{\lambda  h}{8(L+1)} \left(\E\left\|T(X_0)-X_0\right\|^2 - W_2^2(P,Q)\right).
    \]
    Rearranging the terms completes the proof.
\end{proof}

\subsubsection*{Local Validity of the Projection Bound}

We now explain why Assumption~\ref{ass:projection_bound} is natural near
the optimal transport map.  The following result proves the bound along
smooth $P$-measure-preserving perturbations of the identity.  Thus, although
the main criterion is stated abstractly, it is locally valid around $T^*$ in
directions generated by incompressible perturbations of $P$.

\begin{proposition}[Local projection stability near $T^*$]
\label{prop:local-projection-stability}
Assume that $P$ has density $p$, and that the Brenier potential
$\phi_0\in C^2(\mathbb R^d)$ satisfies
\[
\lambda I_d\preceq\nabla^2\phi_0(x)\preceq\Lambda I_d
\]
for some $0<\lambda\le\Lambda<\infty$ and all $x\in\mathbb R^d$.

Let $(S_\varepsilon)_\varepsilon$ be a family of $P$-measure-preserving maps
such that $S_0=\id$ and
\[
S_\varepsilon(x)=x+\varepsilon u(x)+r_\varepsilon(x),
\qquad
\frac{\|r_\varepsilon\|_P}{|\varepsilon|}\to0,
\]
for some nonzero $u\in L^2(P;\mathbb R^d)$.  Define
$T_\varepsilon:=T^*\circ S_\varepsilon$.  Then
$(T_\varepsilon)_\#P=Q$.  Moreover, $u\in\mathcal S_0$, and there exists
$\varepsilon_0>0$ such that, for all
$0<|\varepsilon|<\varepsilon_0$,
\[
\|\proj_{\mathcal S_0}(T_\varepsilon-T^*)\|_P^2
\ge
\frac{\lambda^2}{4\Lambda^2}
\|T_\varepsilon-T^*\|_P^2.
\]
In particular, Assumption~\ref{ass:projection_bound} holds locally along the
family $T_\varepsilon$, with $h=\lambda^2/(4\Lambda^2)$.
\end{proposition}

\begin{proof}[Proof of Proposition~\ref{prop:local-projection-stability}]
Since \(S_\varepsilon\) is \(P\)-measure preserving and \(T^*_\#P=Q\), we have
\[
    (T_{\varepsilon})_\#P
    =
    (T^*\circ S_\varepsilon)_\#P
    =
    T^*_\#((S_{\varepsilon})_\#P)
    =
    T^*_\#P
    =
    Q.
\]

We first show that the infinitesimal perturbation \(u\) lies in \({\mathcal S}_0\). Let
\(g\in C_c^\infty(\mathbb R^d)\). Since \(S_\varepsilon\) preserves \(P\),
    $\int g(S_\varepsilon(x))\,p(x)\diff x
    =
    \int g(x)\,p(x)\diff x$
for every sufficiently small \(\varepsilon\). Differentiating at
\(\varepsilon=0\), using
\((S_\varepsilon-\id)/\varepsilon\to u\) in \(L^2(P; \R^d)\), gives
\[
    \int \nabla g(x)^\top u(x)\,p(x)\diff x=0.
\]
Thus \(u\in {\mathcal S}_0\).
By Taylor's formula,
\[
    \frac{T_\varepsilon(x)-T^*(x)}{\varepsilon}
    =
    \int_0^1
    \nabla^2\phi\bigl(x+s(S_\varepsilon(x)-x)\bigr)
    \,\diff s \cdot \frac{S_\varepsilon(x)-x}{\varepsilon}.
\]
Since \((S_\varepsilon-\id)/\varepsilon\to u\) in \(L^2(P; \R^d)\), we also have
\(S_\varepsilon-\id\to0\) in \(P\)-probability. By continuity and boundedness of
\(\nabla^2\phi_0\), the integral operator above converges strongly in \(L^2(P; \R^d)\) to
multiplication by \(\nabla^2\phi_0(x)\). Hence
\[
    \frac{T_\varepsilon-T^*}{\varepsilon}
    \to
    \nabla^2\phi_0\cdot u
    \qquad\text{in }L^2(P;\mathbb R^d).
\]
Because \(\proj_{{\mathcal S}_0}\) is a bounded linear operator on \(L^2(P;\mathbb R^d)\),
we also have
    ${\proj_{{\mathcal S}_0}(T_\varepsilon-T^*)}/{\varepsilon}
    \longrightarrow
    \proj_{{\mathcal S}_0}(\nabla^2\phi_0\cdot u)$
    in $L^2(P;\mathbb R^d)$.
Therefore,
\[
    \frac{\|\proj_{{\mathcal S}_0}(T_\varepsilon-T^*)\|_P}
         {\|T_\varepsilon-T^*\|_P}
    \longrightarrow
    \frac{\|\proj_{{\mathcal S}_0}(\nabla^2\phi_0\cdot u)\|_P}{\|\nabla^2\phi_0\cdot u\|_P}.
\]

It remains to lower bound the limiting ratio. Since \(u\in {\mathcal S}_0\), the
variational characterization of orthogonal projection in Lemma~\ref{lemma:variational-lb} gives
\[
    \|\proj_{{\mathcal S}_0}(\nabla^2\phi_0\cdot u)\|_P
    \ge
    \frac{\langle \nabla^2\phi_0\cdot u,u\rangle_P}{\|u\|_P}.
\]
Using \(\lambda I_d\preceq \nabla^2\phi_0(x)\preceq \Lambda I_d\), we have
\[
    \langle \nabla^2\phi_0\cdot u,u\rangle_P
    =
    \int u(x)^\top \nabla^2\phi_0(x)u(x)\,p(x)\diff x
    \ge
    \lambda\|u\|_P^2,
\]
and
    $\|\nabla^2\phi_0\cdot u\|_P\le \Lambda\|u\|_P$.
This implies
\[
    \frac{\|\proj_{{\mathcal S}_0}(\nabla^2\phi_0\cdot u)\|_P}{\|\nabla^2\phi_0\cdot u\|_P}
    \ge
    \frac{\lambda}{\Lambda}.
\]
Consequently, % for every \(\eta\in(0,\lambda/\Lambda)\), 
if
\(|\varepsilon|\) is sufficiently small, say there exists \(\varepsilon_0>0\) such that
for all \(0<|\varepsilon|<\varepsilon_0\),
\[
    \frac{\|\proj_{{\mathcal S}_0}(T_\varepsilon-T^*)\|_P}
         {\|T_\varepsilon-T^*\|_P}
    \ge
    \frac{\lambda}{2\Lambda}.
\]
This gives the stated local bound with
\(h=\lambda^2/(4\Lambda^2)\).
\end{proof}

\begin{proof}[Proof of Theorem~\ref{thm:iterated-general-contraction}]
Define the excess cost at the \(k\)-th step by
    $\mathcal E_k
    :=
    \E \left\|Z_1^{(k)}-Z_0^{(k)}\right\|^2
    -
    W_2^2(P,Q)$.
Since \(W_2^2(P,Q)\) is the optimal quadratic transport cost between \(P\) and
\(Q\), we have \(\mathcal E_k\ge0\) for every \(k\ge0\).
Since \textit{c}-rectified flow reduces the cost, $\mathcal E_1 \le \mathcal E_0$.

After the first round of \textit{c}-rectified flow, it induces a transport map $T_1$, and in each following run, the \textit{c}-rectified flow always induces a new transport map. Therefore, we can use our one-step contraction results after the first step.
By the uniform assumptions,  Corollary~\ref{cor:one-step-contraction-general-L2} applies to each
coupling \((Z_0^{(k)},Z_1^{(k)})\), with \(T\) replaced by \(T_k\). Therefore,
for every \(k\ge1\),
    $\mathcal E_{k+1}
    \le
    \left(1-\frac{\lambda h}{8(L+1)}\right)\mathcal E_k$.
Iterating this inequality gives
\[
    \mathcal E_k
    \le
    \left(1-\frac{\lambda h}{8(L+1)}\right)^{k-1}\mathcal E_1 \le
    \left(1-\frac{\lambda h}{8(L+1)}\right)^{k-1}\mathcal E_0
% \]
% Since \(1-\theta\le e^{-\theta}\), we also have
% \[
    % \mathcal E_k
    \le
    \exp{\left(-\frac{\lambda h}{8(L+1)} (k-1)\right)} \mathcal E_0.
\]
Substituting the definition of \(\mathcal E_0\) yields the stated exponential
decay estimate. 
% The iteration complexity bound follows by solving
% \(e^{-\theta k}\mathcal E_0\le\varepsilon\) for \(k\).
\end{proof}

\section{Proof of Section~\ref{sec:statest}}
\label{app:statest-proof}
\subsection*{Smoothing Identities and Score Reduction}

For bounded measurable \(g:\R^d\to\R^k\), define the multivariate
smoothing operator
\begin{equation}\label{eq:multi-At}
 \mathcal A_tg(x)
 =
 \frac{(g\varphi_d)*\varphi_t^{(d)}(x)}{\gamma_t(x)}
 =
 \E[g(M_{x,t})],
 \qquad
 M_{x,t}\sim
 N\left(\frac{x}{1+t},\frac{t}{1+t}I_d\right).
\end{equation}
Conditional Jensen's inequality gives
\begin{equation}\label{eq:multi-At-contraction}
 \int_{\R^d}\|\mathcal A_tg(x)\|^2\gamma_t(x)\diff x
 \le
 \int_{\R^d}\|g(x)\|^2\varphi_d(x)\diff x.
\end{equation}
Moreover,
\begin{equation}\label{eq:multi-At-derivative}
 R_t=\mathcal A_tr,
 \qquad
 \nabla R_t=\frac1{1+t}\mathcal A_t(\nabla r),
 \qquad
 \|\nabla R_t\|_\infty\lesssim(1+t)^{-1}.
\end{equation}
The next lemma reduces multivariate score estimation to estimation of the
smoothed ratio and its gradient.

\begin{lemma}[Multivariate score reduction]
\label{lem:multi-score-reduction}
Let \(\hat R_{t,0}\) and \(\hat R_{t,1}\) be arbitrary estimators of
\(R_t\) and \(\nabla R_t\), respectively, and define
\[
 \hat s_t(x)
 =
 -\frac{x}{1+t}
 +
 \frac{\hat R_{t,1}(x)}
 {\Pi_{[C_1/2,2C_2]}(\hat R_{t,0})(x)}.
\]
Then
\[
 \int_{\R^d}\|\hat s_t-s_t\|^2p_t
 \lesssim
 \int_{\R^d}\|\hat R_{t,1}-\nabla R_t\|^2\gamma_t
 +
 \frac1{(1+t)^2}
 \int_{\R^d}|\hat R_{t,0}-R_t|^2\gamma_t.
\]
\end{lemma}

\begin{proof}
Put \(\bar R_{t,0}=\Pi_{[C_1/2,2C_2]}(\hat R_{t,0})\).  Since
\(R_t\in[C_1,C_2]\), clipping is nonexpansive and
\(|\bar R_{t,0}-R_t|\le|\hat R_{t,0}-R_t|\).  Adding and subtracting
\(\nabla R_t/\bar R_{t,0}\) gives
\[
 \left\|
 \frac{\hat R_{t,1}}{\bar R_{t,0}}-
 \frac{\nabla R_t}{R_t}
 \right\|^2
 \lesssim
 \|\hat R_{t,1}-\nabla R_t\|^2
 +
 \|\nabla R_t\|^2|\hat R_{t,0}-R_t|^2.
\]
Now use \(p_t=R_t\gamma_t\), \(C_1\le R_t\le C_2\), and
\(\|\nabla R_t\|_\infty\lesssim(1+t)^{-1}\) from
\eqref{eq:multi-At-derivative}.
\end{proof}

\subsection{Upper Bound for Theorem~\ref{thm:multi-score-minimax}}

We begin with the multivariate low-noise estimator.  For a multi-index
\(\nu=(\nu_1,\ldots,\nu_d)\), write
\(|\nu|=\sum_j\nu_j\), \(u^\nu=\prod_ju_j^{\nu_j}\), and
\(D^\nu=\prod_j\partial_j^{\nu_j}\).

\begin{proposition}[Multivariate weighted ratio derivative estimation]
\label{prop:multi-ratio-derivative-low}
There exists \(b_0>0\), depending only on
\((C_1,C_2,L,\alpha,d)\), for which one can construct estimators
\(\hat r_{b,0}\) and \(\hat r_{b,1}\) such that, for every
\(0<b\le b_0\),
\[
 \sup_{f\in\mathcal F_{\alpha,d}}
 \E_f\int_{\R^d}|\hat r_{b,0}(x)-r(x)|^2\varphi_d(x)\diff x
 \lesssim
 b^{2\alpha}+\frac1{nb^d},
\]
and
\[
 \sup_{f\in\mathcal F_{\alpha,d}}
 \E_f\int_{\R^d}\|\hat r_{b,1}(x)-\nabla r(x)\|^2
 \varphi_d(x)\diff x
 \lesssim
 b^{2(\alpha-1)}+\frac1{nb^{d+2}}.
\]
\end{proposition}

\begin{proof}
Let \(p=m_\alpha\), and let
\(\mathcal I_p=\{\nu\in\mathbb N_0^d:|\nu|\le p\}\).  Choose a
nonnegative \(C^\infty\) kernel \(K\), supported on the unit ball and bounded
from below by a positive constant on the ball of radius \(1/2\).  Put
\[
 E_b=\{x:\|x\|\le b^{-1/3}\},
 \qquad
 \ell_b(x)=b(1+\|x\|)^2.
\]
If \(x\in E_b\) and \(\|u-x\|\le\ell_b(x)\), then
\[
 |\log\varphi_d(u)-\log\varphi_d(x)|
 \le
 \|x\|\|u-x\|+\frac12\|u-x\|^2
 \lesssim1,
\]
provided \(b_0\) is sufficiently small.  Hence
\(\varphi_d(u)\asymp\varphi_d(x)\) uniformly on all local neighborhoods
used below.

For \(\nu,\mu\in\mathcal I_p\), define
\[
 G_{\nu\mu}(x)
 =
 \int
 K\left(\frac{u-x}{\ell_b(x)}\right)
 (u-x)^{\nu+\mu}\varphi_d(u)\diff u.
\]
After the change of variables \(u=x+\ell_b(x)v\), the rescaled moment
matrix is uniformly positive definite on the finite-dimensional polynomial
space of degree at most \(p\).  Consequently,
\begin{equation}\label{eq:multi-gram-inverse-low}
 |(G_x^{-1})_{\nu\mu}|
 \lesssim
 \ell_b(x)^{-d-|\nu|-|\mu|}\varphi_d(x)^{-1},
 \qquad
 \nu,\mu\in\mathcal I_p,
\end{equation}
uniformly over \(x\in E_b\).

For \(x\in E_b\), let
\(\hat a(x)=(\hat a_\nu(x):\nu\in\mathcal I_p)\) minimize
\[
 \int
 K\left(\frac{u-x}{\ell_b(x)}\right)
 \left\{\sum_{\nu\in\mathcal I_p}a_\nu(u-x)^\nu\right\}^2
 \varphi_d(u)\diff u
 -
 \frac2n\sum_{i=1}^n
 K\left(\frac{X_i-x}{\ell_b(x)}\right)
 \sum_{\nu\in\mathcal I_p}a_\nu(X_i-x)^\nu.
\]
Set
\[
 \hat r_{b,0}(x)=\hat a_0(x),
 \qquad
 \hat r_{b,1}(x)
 =
 (\hat a_{e_1}(x),\ldots,\hat a_{e_d}(x))^\top,
 \qquad x\in E_b,
\]
and set \(\hat r_{b,0}=1\), \(\hat r_{b,1}=0\) on \(E_b^c\).

Write \(\hat a(x)=G_x^{-1}\hat g(x)\), where
\[
 \hat g_\nu(x)
 =
 \frac1n\sum_{i=1}^n
 K\left(\frac{X_i-x}{\ell_b(x)}\right)(X_i-x)^\nu,
 \qquad
 g_\nu(x)=\E_f\hat g_\nu(x),
\]
and put \(a^\star(x)=G_x^{-1}g(x)\).  The multivariate Taylor remainder for
\(C^\alpha\) functions, together with
\eqref{eq:multi-gram-inverse-low}, gives, for \(|\rho|\le1\),
\begin{equation}\label{eq:multi-bias-low}
 \left|a_\rho^\star(x)-\frac{D^\rho r(x)}{\rho!}\right|
 \lesssim
 \ell_b(x)^{\alpha-|\rho|}.
\end{equation}
Since \(f\le C_2\varphi_d\),
\[
 |\operatorname{Cov}(\hat g_\nu(x),\hat g_\mu(x))|
 \lesssim
 \frac1n
 \ell_b(x)^{d+|\nu|+|\mu|}\varphi_d(x).
\]
A second application of \eqref{eq:multi-gram-inverse-low} yields
\begin{equation}\label{eq:multi-var-low}
 \operatorname{Var}(\hat a_\rho(x))
 \lesssim
 \frac1{n\ell_b(x)^{d+2|\rho|}\varphi_d(x)},
 \qquad |\rho|\le1.
\end{equation}
Combining \eqref{eq:multi-bias-low} and
\eqref{eq:multi-var-low}, and integrating over \(E_b\), gives
\[
 \int_{E_b}\ell_b(x)^{2\alpha}\varphi_d(x)\diff x
 \lesssim b^{2\alpha},
 \qquad
 \int_{E_b}\ell_b(x)^{2(\alpha-1)}\varphi_d(x)\diff x
 \lesssim b^{2(\alpha-1)},
\]
while
\[
 \int_{E_b}\frac{\diff x}{n\ell_b(x)^d}
 \lesssim\frac1{nb^d},
 \qquad
 \int_{E_b}\frac{\diff x}{n\ell_b(x)^{d+2}}
 \lesssim\frac1{nb^{d+2}}.
\]
Here we used
\(\int_{\R^d}(1+\|x\|)^{-2d}\diff x<\infty\).  Finally,
\(r\) and \(\nabla r\) are uniformly bounded, and Gaussian tails give
\[
 \int_{E_b^c}\varphi_d(x)\diff x
 \lesssim e^{-h_0b^{-2/3}}
 \lesssim b^m
\]
for every fixed \(m>0\).  This proves both assertions.
\end{proof}

For \(0<t\le1\), define
\[
 \hat R_{t,b,0}^{\rm lo}=\mathcal A_t\hat r_{b,0},
 \qquad
 \hat R_{t,b,1}^{\rm lo}
 =\frac1{1+t}\mathcal A_t\hat r_{b,1},
\]
and let \(\hat s_{t,b}^{\rm lo}\) be the score estimator obtained from
Lemma~\ref{lem:multi-score-reduction}.  By
\eqref{eq:multi-At-contraction} and
Proposition~\ref{prop:multi-ratio-derivative-low}, with \(b=h_n\),
\begin{equation}\label{eq:multi-low-upper}
 \sup_{f\in\mathcal F_{\alpha,d}}
 \E_f\int_{\R^d}
 \|\hat s_{t,h_n}^{\rm lo}-s_t\|^2p_t
 \lesssim
 n^{-\frac{2(\alpha-1)}{2\alpha+d}},
 \qquad 0<t\le1.
\end{equation}

We next construct the intermediate-noise estimator.

\begin{lemma}[Multivariate Gaussian second moments]
\label{lem:multi-gaussian-moments}
Let \(0<t\le1\), \(\sigma=\sqrt t\), and \(Y\sim\varphi_d\).  Then
\[
 \gamma_t(x)
 \E\left\{
 \frac{\varphi_t^{(d)}(x-Y)}{\gamma_t(x)}
 \right\}^2
 \lesssim
 \sigma^{-d}e^{-h_0\sigma^2\|x\|^2},
\]
and
\[
 \gamma_t(x)
 \E\left\|
 \nabla_x\left\{
 \frac{\varphi_t^{(d)}(x-Y)}{\gamma_t(x)}
 \right\}
 \right\|^2
 \lesssim
 \sigma^{-d}(\sigma^{-2}+\|x\|^2)
 e^{-h_0\sigma^2\|x\|^2},
\]
where \(h_0>0\) is universal.
\end{lemma}

\begin{proof}
Completing the square coordinatewise gives
\[
 \int\varphi_t^{(d)}(x-y)^2\varphi_d(y)\diff y
 \lesssim
 \sigma^{-d}\gamma_t(x)e^{-h_0\sigma^2\|x\|^2}.
\]
Moreover,
\[
 \nabla_x\left\{
 \frac{\varphi_t^{(d)}(x-y)}{\gamma_t(x)}
 \right\}
 =
 \frac1t\left(y-\frac{x}{1+t}\right)
 \frac{\varphi_t^{(d)}(x-y)}{\gamma_t(x)}.
\]
A second coordinatewise Gaussian calculation gives the stated derivative
bound.
\end{proof}

The following local polynomial is fitted on an inward set so that the
Gaussian reference density is never smaller than at the point of estimation.

\begin{lemma}[Inward multivariate local-polynomial ratio estimation]
\label{lem:multi-local-ratio}
Let \(Y_1,\ldots,Y_n\) be i.i.d. with density \(R\gamma_t\), where
\(0<t\le1\), \(C_1\le R\le C_2\), and
\(\|R\|_{C^\alpha(\R^d)}\le L\).  There exist estimators
\(\hat R_{t,0}^{\rm lp}(x)\) and
\(\hat R_{t,1}^{\rm lp}(x)\), defined for \(\|x\|\ge1\), such that
\[
 \E|\hat R_{t,0}^{\rm lp}(x)-R(x)|^2
 \lesssim
 \{n\gamma_t(x)\}^{-\frac{2\alpha}{2\alpha+d}}\wedge1,
\]
and
\[
 \E\|\hat R_{t,1}^{\rm lp}(x)-\nabla R(x)\|^2
 \lesssim
 \{n\gamma_t(x)\}^{-\frac{2(\alpha-1)}{2\alpha+d}}\wedge1.
\]
\end{lemma}

\begin{proof}
Let \(p=m_\alpha\), and fix the open ball
\(\mathcal U=B(e_1/2,1/8)\).  Since restriction to \(\mathcal U\) defines a
positive-definite inner product on the finite-dimensional space of
polynomials of degree at most \(p\), there exist bounded functions
\(L_0,L_1,\ldots,L_d\), supported on \(\mathcal U\), such that, for every
polynomial \(P\) of degree at most \(p\),
\[
 \int_{\mathcal U}L_0(u)P(u)\diff u=P(0),
 \qquad
 \int_{\mathcal U}L_j(u)P(u)\diff u=\partial_jP(0),
 \quad 1\le j\le d.
\]
Write \(L_\nabla=(L_1,\ldots,L_d)^\top\).  Fix a Borel measurable choice of
an orthogonal matrix \(Q_x\) satisfying
\(Q_xe_1=x/\|x\|\), \(x\ne0\).

If \(n\gamma_t(x)<1\), set
\(\hat R_{t,0}^{\rm lp}(x)=1\) and
\(\hat R_{t,1}^{\rm lp}(x)=0\).  Otherwise put
\[
 b_x=\{n\gamma_t(x)\}^{-1/(2\alpha+d)}\le1,
 \qquad
 U_i=\frac{Q_x^\top(x-Y_i)}{b_x},
\]
and define
\[
 \hat R_{t,0}^{\rm lp}(x)
 =
 \frac1{nb_x^d}\sum_{i=1}^n
 \frac{L_0(U_i)\1_{\{U_i\in\mathcal U\}}}{\gamma_t(Y_i)},
\]
\[
 \hat R_{t,1}^{\rm lp}(x)
 =
 -\frac{Q_x}{nb_x^{d+1}}\sum_{i=1}^n
 \frac{L_\nabla(U_i)\1_{\{U_i\in\mathcal U\}}}{\gamma_t(Y_i)}.
\]
For \(u\in\mathcal U\), \(\|x\|\ge1\), and \(0<b_x\le1\),
\[
 \|x-b_xQ_xu\|^2
 =
 \|x\|^2-2b_x\|x\|u_1+b_x^2\|u\|^2
 \le\|x\|^2.
\]
Hence
\(\gamma_t(x-b_xQ_xu)\ge\gamma_t(x)\).

The reproduction identities and the multivariate Taylor remainder give
\[
 |\E\hat R_{t,0}^{\rm lp}(x)-R(x)|\lesssim b_x^\alpha,
 \qquad
 \|\E\hat R_{t,1}^{\rm lp}(x)-\nabla R(x)\|
 \lesssim b_x^{\alpha-1}.
\]
Using \(R\le C_2\), the inward inequality above, and the change of variables
\(y=x-b_xQ_xu\),
\[
 \operatorname{Var}(\hat R_{t,0}^{\rm lp}(x))
 \lesssim
 \frac1{n\gamma_t(x)b_x^d},
 \qquad
 \E\|\hat R_{t,1}^{\rm lp}(x)-\E\hat R_{t,1}^{\rm lp}(x)\|^2
 \lesssim
 \frac1{n\gamma_t(x)b_x^{d+2}}.
\]
The choice of \(b_x\) balances squared bias and variance.  When
\(n\gamma_t(x)<1\), the result follows from the uniform bounds on
\(R\) and \(\nabla R\).
\end{proof}

Consequently, for every measurable
\(E\subset\{x:\|x\|\ge1\}\), with \(D=2\alpha+d\),
\begin{equation}\label{eq:multi-local-ratio-integrated}
 \E\int_E
 \|\hat R_{t,1}^{\rm lp}-\nabla R\|^2\gamma_t
 \lesssim
 n^{-\frac{2(\alpha-1)}D}
 \int_E\gamma_t^{\frac{d+2}D},
\end{equation}
and
\begin{equation}\label{eq:multi-local-ratio-integrated-zero}
 \E\int_E
 |\hat R_{t,0}^{\rm lp}-R|^2\gamma_t
 \lesssim
 n^{-\frac{2\alpha}D}
 \int_E\gamma_t^{\frac dD}.
\end{equation}

We finally record the exact global Gaussian calculation.

\begin{lemma}[Multivariate global plug-in bound]
\label{lem:multi-global-plugin}
Define
\[
 \hat R_{t,0}^G(x)
 =
 \frac{n^{-1}\sum_{i=1}^n\varphi_t^{(d)}(x-X_i)}{\gamma_t(x)},
 \qquad
 \hat R_{t,1}^G=\nabla\hat R_{t,0}^G,
\]
and let \(\hat s_t^G\) be the corresponding clipped ratio score estimator.
Then, uniformly over \(t>0\),
\[
 \sup_{f\in\mathcal F_{\alpha,d}}
 \E_f\int_{\R^d}\|\hat s_t^G-s_t\|^2p_t
 \lesssim
 \frac{(1+t)^{d-1}}{nt^{d+1}}.
\]
In particular, the right-hand side is of order
\(1/(nt^{d+1})\) for \(0<t\le1\) and of order \(1/(nt^2)\) for
\(t>1\).
\end{lemma}

\begin{proof}
Let
 $K_t(x,y)=\frac{\varphi_t^{(d)}(x-y)}{\gamma_t(x)}$.
By tensorization of the corresponding one-dimensional identities,
\[
 \iint K_t(x,y)^2\gamma_t(x)\varphi_d(y)\diff x\diff y
 =
 \left(\frac{1+t}{t}\right)^d.
\]
Also,
\[
 \nabla_xK_t(x,y)
 =
 \frac1t\left(y-\frac{x}{1+t}\right)K_t(x,y),
\]
and a direct Gaussian calculation yields
\begin{equation}\label{eq:multi-global-gradient-identity}
 \iint\|\nabla_xK_t(x,y)\|^2
 \gamma_t(x)\varphi_d(y)\diff x\diff y
 =
 d\frac{(1+t)^{d-1}}{t^{d+1}}.
\end{equation}
Since \(f\le C_2\varphi_d\), these identities imply
\[
 \E_f\int|\hat R_{t,0}^G-R_t|^2\gamma_t
 \lesssim
 \frac1n\left(\frac{1+t}{t}\right)^d,
 \qquad
 \E_f\int\|\hat R_{t,1}^G-\nabla R_t\|^2\gamma_t
 \lesssim
 \frac{(1+t)^{d-1}}{nt^{d+1}}.
\]
By \eqref{eq:multi-At-derivative}, the denominator-estimation contribution in
Lemma~\ref{lem:multi-score-reduction} is bounded by
\[
 \frac1{(1+t)^2n}\left(\frac{1+t}{t}\right)^d
 \lesssim
 \frac{(1+t)^{d-1}}{nt^{d+1}}.
\]
The result follows.
\end{proof}

\begin{proof}[Proof of the upper bound in
Theorem~\ref{thm:multi-score-minimax}]
The low-noise estimator \eqref{eq:multi-low-upper} supplies the first branch
of \eqref{eq:multi-score-envelope}.  It remains to obtain the second branch
on \(h_n^2\le t\le1\).  Put
 $\sigma=\sqrt t,
 \
 \Lambda=1\vee\sqrt{\log_+(nt^{\alpha+d/2})}$.
Then \(N\ge1\).  Fix a sufficiently small constant \(\eta>0\).  If
\(\sigma\Lambda\ge\eta\), the global plug-in estimator of
Lemma~\ref{lem:multi-global-plugin} gives
\[
 \frac1{n\sigma^{2d+2}}
 \le
 \eta^{-d}\frac{\Lambda^d}{n\sigma^{d+2}},
\]
so it attains, up to a constant, the smaller of the intermediate and global
plug-in branches.

It remains to consider \(\sigma\Lambda<\eta\).  Choose a sufficiently large constant \(M\),
and then take \(\eta\le(2M)^{-1}\).  Generate independent
\(Z_i\sim N(0,I_d)\) and put \(Y_i=X_i+\sigma Z_i\), so that
\(Y_i\sim p_t=R_t\gamma_t\).  On the ball
\(\{x:\|x\|\le M\Lambda\}\), use the Gaussian plug-in estimators
\(\hat R_{t,0}^G\) and \(\hat R_{t,1}^G\).  Outside this ball, apply
Lemma~\ref{lem:multi-local-ratio} to \(Y_1,\ldots,Y_n\).  Since
\(M\Lambda\le(2\sigma)^{-1}\), Lemma~\ref{lem:multi-gaussian-moments} gives
\[
 \E_f\int_{\|x\|\le M\Lambda}
 \|\hat R_{t,1}^G-\nabla R_t\|^2\gamma_t
 \lesssim
 \frac{\Lambda^d}{n\sigma^{d+2}},
\qquad
 \E_f\int_{\|x\|\le M\Lambda}
 |\hat R_{t,0}^G-R_t|^2\gamma_t
 \lesssim
 \frac{\Lambda^d}{n\sigma^d}
 \le
 \frac{\Lambda^d}{n\sigma^{d+2}}.
\]
The representation \(R_t=\mathcal A_tr\), together with
\(D^\nu R_t=(1+t)^{-|\nu|}\mathcal A_t(D^\nu r)\) for
\(|\nu|\le m_\alpha\) and the corresponding H\"older contraction, implies
\(\|R_t\|_{C^\alpha(\R^d)}\lesssim1\) uniformly over \(0<t\le1\).
Thus Lemma~\ref{lem:multi-local-ratio} applies in the complement.

Let \(a_1=(d+2)/D\) and \(a_0=d/D\).  Uniformly over \(0<t\le1\), Gaussian
tails imply, after increasing \(M\) if necessary,
\[
 \int_{\|x\|>M\Lambda}\gamma_t(x)^{a_j}\diff x
 \lesssim
 \begin{cases}
 N^{-a_j},&N\ge e,\\
 1,&1\le N<e,
 \end{cases}
 \qquad j\in\{0,1\}.
\]
Equations \eqref{eq:multi-local-ratio-integrated} and
\eqref{eq:multi-local-ratio-integrated-zero} show that the two tail risks are
also bounded by \(\Lambda^d/(n\sigma^{d+2})\).  Combining the central and
tail estimators and applying Lemma~\ref{lem:multi-score-reduction} gives
\[
 \sup_{f\in\mathcal F_{\alpha,d}}
 \E_f\int\|\hat s_t^{\rm im}-s_t\|^2p_t
 \lesssim
 \frac{\Lambda^d}{n\sigma^{d+2}}.
\]
The auxiliary Gaussian randomization can be removed by taking the conditional
expectation of the resulting score estimator given
\(X_1,\ldots,X_n\), which does not increase squared risk.

Taking the best of the low-noise, intermediate-noise, and global plug-in
estimators gives \eqref{eq:multi-score-envelope} for \(0<t\le1\).  For
\(t>1\), Lemma~\ref{lem:multi-global-plugin} gives \(1/(nt^2)\).
\end{proof}

\subsection{Lower Bound for Theorem~\ref{thm:multi-score-minimax}}

\begin{proof}[Proof of the lower bound in
Theorem~\ref{thm:multi-score-minimax}]
We treat the three noise regimes separately.

\subsubsection{Low Noise}
Let \(h=h_n\), and choose a nonzero
\(\psi\in C_c^\infty(B(0,1))\) satisfying \(\int\psi=0\).  For
\(0\le u\le1\), put \(\psi_u=\psi*\varphi_u^{(d)}\).  By Plancherel's
identity,
\[
 \|\nabla\psi_u\|_2^2
 =
 \frac1{(2\pi)^d}
 \int_{\R^d}\|\xi\|^2e^{-u\|\xi\|^2}|\widehat\psi(\xi)|^2\diff\xi
 >0.
\]
Continuity in \(u\), compactness of \([0,1]\), and uniform Gaussian decay
therefore give a fixed \(A>1\) such that
\begin{equation}\label{eq:multi-low-bump-separation}
 \inf_{0\le u\le1}
 \int_{\|z\|\le A}\|\nabla\psi_u(z)\|^2\diff z
 \gtrsim1,
 \qquad
 \sup_{0\le u\le1}
 \int_{\|z\|\le A}|\psi_u(z)|^2\diff z
 \lesssim1.
\end{equation}
Choose \(m\asymp h^{-d}\) points
\(x_1,\ldots,x_m\in[-1/4,1/4]^d\), separated by at least \(Kh\), where
\(K>2A+2\) is fixed.  For
\(\theta\in\{-1,1\}^m\), define
\[
 f_\theta(x)
 =
 \varphi_d(x)
 +
 \delta h^\alpha
 \sum_{j=1}^m\theta_j
 \psi\left(\frac{x-x_j}{h}\right),
 \qquad
 r_\theta=\frac{f_\theta}{\varphi_d},
\]
where \(\delta>0\) is a sufficiently small constant.  Since the
perturbations have disjoint supports in a fixed compact set and
\(\int\psi=0\), the interior-point assumption implies
\(f_\theta\in\mathcal F_{\alpha,d}\) for every \(\theta\).

If \(\theta^{(j)}\) differs from \(\theta\) only in coordinate \(j\), then
\[
 \chi^2(f_\theta,f_{\theta^{(j)}})
 \lesssim
 \delta^2h^{2\alpha+d}.
\]
Since \(nh^{2\alpha+d}=1\), Pinsker's inequality gives
\[
 \|f_\theta^{\otimes n}-f_{\theta^{(j)}}^{\otimes n}\|_{\rm TV}
 \le\frac12
\]
after reducing \(\delta\).

For \(0<t\le h^2\), put \(u=t/h^2\) and
\[
 q_{j,t}(x)
 =
 \left\{
 \psi\left(\frac{\cdot-x_j}{h}\right)
 *\varphi_t^{(d)}
 \right\}(x)
 =
 \psi_u\left(\frac{x-x_j}{h}\right).
\]
Let
\(I_j=\{x:\|(x-x_j)/h\|\le A\}\).  These sets are pairwise disjoint.
On their union, uniformly over \(t\le h^2\), the smoothed densities are
bounded above and below and their gradients are uniformly bounded.  Writing
\(\bar p_{\theta,j,t}=(p_{\theta,t}+p_{\theta^{(j)},t})/2\), a direct
calculation gives
\[
 s_{\theta,t}-s_{\theta^{(j)},t}
 =
 \frac{2\delta h^\alpha\theta_j
 \{\nabla q_{j,t}\,\bar p_{\theta,j,t}
 -q_{j,t}\nabla\bar p_{\theta,j,t}\}}
 {p_{\theta,t}p_{\theta^{(j)},t}}.
\]
Equation \eqref{eq:multi-low-bump-separation} implies
\[
 \|\nabla q_{j,t}\|_{L^2(I_j)}^2\gtrsim h^{d-2},
 \qquad
 \|q_{j,t}\|_{L^2(I_j)}^2\lesssim h^d.
\]
Therefore, for all sufficiently large \(n\),
\begin{equation}\label{eq:multi-low-score-separation}
 \int_{I_j}
 \|s_{\theta,t}-s_{\theta^{(j)},t}\|^2\gamma_t
 \gtrsim
 \delta^2h^{2\alpha+d-2}.
\end{equation}
Assouad's reduction, applied over the disjoint sets \(I_j\), now yields
\[
 \inf_{\hat s_t}\sup_{f\in\mathcal F_{\alpha,d}}
 \E_f\int\|\hat s_t-s_t\|^2p_t
 \gtrsim
 mh^{2\alpha+d-2}
 \asymp
 h^{2(\alpha-1)}
 =
 n^{-\frac{2(\alpha-1)}{2\alpha+d}}.
\]

\subsubsection{Intermediate Noise}
Let \(h^2\le t\le1\), \(\sigma=\sqrt t\), and
\(\Lambda=1\vee\sqrt{\log_+(n\sigma^{2\alpha+d})}\).  Fix nonnegative, nonzero functions
\(a,b\in C_c^\infty(B(0,1))\) with disjoint supports.  For
\(x_0\in\R^d\) satisfying \(\sigma\|x_0\|\le\beta_0\), define
\[
 \lambda_{x_0,\sigma}
 =
 \frac{
 \int a(u)e^{-\sigma x_0^\top u-\sigma^2\|u\|^2/2}\diff u
 }{
 \int b(u)e^{-\sigma x_0^\top u-\sigma^2\|u\|^2/2}\diff u
 },
\qquad
 q_{x_0,\sigma}(x)
 =
 a\left(\frac{x-x_0}{\sigma}\right)
 -
 \lambda_{x_0,\sigma}
 b\left(\frac{x-x_0}{\sigma}\right).
\]
Then
\(\int q_{x_0,\sigma}\varphi_d=0\).  The parameters
\(\lambda_{x_0,\sigma}\) range over a fixed compact subinterval of
\((0,\infty)\), and
\begin{equation}\label{eq:multi-intermediate-bump-norms}
 \|D^\nu q_{x_0,\sigma}\|_\infty\lesssim\sigma^{-|\nu|},
 \qquad
 \int q_{x_0,\sigma}(x)^2\varphi_d(x)\diff x
 \asymp
 \sigma^d\varphi_d(x_0).
\end{equation}

Put \(Q_\lambda=a-\lambda b\).  For every admissible \(\lambda\) and every
\(v\in[1/2,1]\),
\(\nabla(Q_\lambda*\varphi_v^{(d)})\not\equiv0\).  Indeed, otherwise the
convolution would be a constant vanishing at infinity, hence zero; taking
Fourier transforms would imply \(Q_\lambda=0\), a contradiction.  By
compactness, there exists a fixed \(A>1\) such that
\begin{equation}\label{eq:multi-intermediate-convolution-separation}
 \inf_{\lambda,v}
 \int_{\|z\|\le A}
 \|\nabla(Q_\lambda*\varphi_v^{(d)})(z)\|^2\diff z
 \gtrsim1.
\end{equation}
For \(x=(1+t)(x_0+\sigma z)\), a change of variables in
\eqref{eq:multi-At} gives
\[
 \mathcal A_tq_{x_0,\sigma}(x)
 =
 (Q_{\lambda_{x_0,\sigma}}*\varphi_{1/(1+t)}^{(d)})(z).
\]
Let
\[
 I_{x_0,\sigma}
 =
 \left\{x:\left\|\frac{x}{1+t}-x_0\right\|\le A\sigma\right\}.
\]
Because \(\sigma\|x_0\|\le\beta_0\) and \(\|z\|\le A\),
 $\gamma_t((1+t)(x_0+\sigma z))\asymp\varphi_d(x_0)$
uniformly over all admissible \(x_0,\sigma,z\).  Therefore, the change of
variables above, together with \eqref{eq:multi-At-contraction} and
\eqref{eq:multi-intermediate-convolution-separation}, implies
\begin{equation}\label{eq:multi-intermediate-smoothed-bump}
 \int_{I_{x_0,\sigma}}
 \|\nabla\mathcal A_tq_{x_0,\sigma}(x)\|^2\gamma_t(x)\diff x
 \gtrsim
 \sigma^{d-2}\varphi_d(x_0),
\qquad
 \int_{\R^d}
 |\mathcal A_tq_{x_0,\sigma}(x)|^2\gamma_t(x)\diff x
 \lesssim
 \sigma^d\varphi_d(x_0).
\end{equation}

Set
 $B=\beta_0\min\{\Lambda,\sigma^{-1}\}$.
Choose a \(K\sigma\)-separated packing
\(x_1,\ldots,x_m\) of the ball of radius \(B\), with
\(K>2A+2\).  Then
\[
 m\asymp 1+\left(\frac B\sigma\right)^d
 \gtrsim\left(\frac B\sigma\right)^d,
\]
and the corresponding supports and sets \(I_j=I_{x_j,\sigma}\) are pairwise
disjoint.  Since \(B\le\beta_0\Lambda\), choosing \(\beta_0\) sufficiently
small gives $n\sigma^{2\alpha+d}\varphi_d(x_j)\gtrsim1$,
$1\le j\le m$.
Indeed, if \(\Lambda=1\), then \(1\le n\sigma^{2\alpha+d}\le e\) and the centers remain
in a fixed ball.  If \(\Lambda>1\), then
\(\|x_j\|^2\le\beta_0^2\log (n\sigma^{2\alpha+d})\), so
\(n\sigma^{2\alpha+d}\varphi_d(x_j)\gtrsim (n\sigma^{2\alpha+d})^{1-\beta_0^2/2}\).
Define
\[
 \delta_j
 =
 \kappa\{n\sigma^d\varphi_d(x_j)\}^{-1/2},
 \qquad
 r_\theta
 =
 1+\sum_{j=1}^m\theta_j\delta_jq_{x_j,\sigma},
 \qquad
 f_\theta=r_\theta\varphi_d.
\]
Then \(\delta_j\lesssim\kappa\sigma^\alpha\).  By
\eqref{eq:multi-intermediate-bump-norms}, disjointness of the supports, and
the standard scaling property of the H\"older norm, choosing \(\kappa\)
sufficiently small gives
\[
 f_\theta\in\mathcal F_{\alpha,d},
 \qquad
 \frac12\le r_\theta\le\frac32.
\]
For neighboring vertices,
 $D_{\rm KL}(f_\theta^{\otimes n}\|f_{\theta^{(j)}}^{\otimes n})
 \lesssim
 n\delta_j^2\sigma^d\varphi_d(x_j)
 \lesssim\kappa^2$,
so their total variation distance is at most \(1/2\) after reducing
\(\kappa\).

Let \(u_{j,t}=\mathcal A_tq_{x_j,\sigma}\).  For a neighboring pair, write
\[
 \bar R_{\theta,j,t}
 =
 \mathcal A_t\left(\frac{r_\theta+r_{\theta^{(j)}}}{2}\right).
\]
Then
 $R_{\theta,t}
 =\bar R_{\theta,j,t}+\theta_j\delta_ju_{j,t},
 R_{\theta^{(j)},t}
 =\bar R_{\theta,j,t}-\theta_j\delta_ju_{j,t}$,
while
 $\frac12\le\bar R_{\theta,j,t}\le\frac32,
 \|\nabla\bar R_{\theta,j,t}\|_\infty\lesssim\kappa$.
For the derivative bound, disjointness gives
\(\|\nabla\bar r_{\theta,j}\|_\infty
\lesssim\max_{k\ne j}\delta_k\sigma^{-1}
\lesssim\kappa\sigma^{\alpha-1}\lesssim\kappa\), and then
\eqref{eq:multi-At-derivative} applies.
Hence
\[
 s_{\theta,t}-s_{\theta^{(j)},t}
 =
 \frac{2\theta_j\delta_j
 \{\nabla u_{j,t}\,\bar R_{\theta,j,t}
 -u_{j,t}\nabla\bar R_{\theta,j,t}\}}
 {R_{\theta,t}R_{\theta^{(j)},t}}.
\]
Equations \eqref{eq:multi-intermediate-smoothed-bump} therefore give, after
reducing \(\kappa\),
\begin{equation}\label{eq:multi-intermediate-score-separation}
 \int_{I_j}
 \|s_{\theta,t}-s_{\theta^{(j)},t}\|^2\gamma_t
 \gtrsim
 \delta_j^2\sigma^{d-2}\varphi_d(x_j)
 \asymp
 \frac1{n\sigma^2}.
\end{equation}
Assouad's reduction and the disjointness of the \(I_j\)'s imply
\[
 \inf_{\hat s_t}\sup_{f\in\mathcal F_{\alpha,d}}
 \E_f\int\|\hat s_t-s_t\|^2p_t
 \gtrsim
 \frac{m}{n\sigma^2}
 \gtrsim
 \frac{B^d}{n\sigma^{d+2}}.
\]
Since
 $B^d
 \asymp
 \Lambda^d\wedge\sigma^{-d}$,
we obtain
\[
 \inf_{\hat s_t}\sup_{f\in\mathcal F_{\alpha,d}}
 \E_f\int\|\hat s_t-s_t\|^2p_t
 \gtrsim
 \frac{\Lambda^d}{n\sigma^{d+2}}
 \wedge
 \frac1{n\sigma^{2d+2}},
 \qquad h^2\le t\le1.
\]
Since the first line of \eqref{eq:multi-score-envelope} is bounded above by
this quantity, this proves the desired lower bound on this range. 

\subsubsection{Large Noise}
Choose \(q\in C_c^\infty(\R^d)\) such that
\[
 \int q\varphi_d=0,
 \qquad
 \mu_q:=\int\nabla q(x)\varphi_d(x)\diff x\ne0.
\]
For example, one may take \(q(x)=x_1\eta(x)\), where \(\eta\) is a
nonnegative smooth radial function of compact support.  Let
\(\delta=\kappa n^{-1/2}\),
\(r_\omega=1+\omega\delta q\), and
\(f_\omega=r_\omega\varphi_d\), \(\omega\in\{-1,1\}\).  For sufficiently
small \(\kappa\), both alternatives belong to \(\mathcal F_{\alpha,d}\).  Moreover,
\[
 D_{\rm KL}(f_+^{\otimes n}\|f_-^{\otimes n})
 \lesssim
 n\delta^2\int q(x)^2\varphi_d(x)\diff x
 \lesssim\kappa^2,
\]
so their product-sample total variation distance is bounded away from one.

Put \(u_t=\mathcal A_tq\).  Then
\[
 \nabla u_t=\frac1{1+t}\mathcal A_t(\nabla q),
\]
and conditional Jensen's inequality gives
\[
 \int\|\nabla u_t(x)\|^2\gamma_t(x)\diff x
 \ge
 \frac{\|\mu_q\|^2}{(1+t)^2}.
\]
Since
 $s_{+,t}-s_{-,t}
 =
 \frac{2\delta\nabla u_t}{1-\delta^2u_t^2}$,
the two-point testing reduction yields
\[
 \inf_{\hat s_t}\sup_{f\in\mathcal F_{\alpha,d}}
 \E_f\int\|\hat s_t-s_t\|^2p_t
 \gtrsim
 \frac1{n(1+t)^2}
 \asymp
 \frac1{nt^2},
 \qquad t>1.
\]
This completes the proof.
\end{proof}

\subsection{Proof of Theorem~\ref{thm:w2-projected-sampler}}

We now verify the matrix slope bounds underlying the projection in
Algorithm~\ref{algo:projected-reverse-sampler}.  Recall that
\(\bar\alpha=\alpha\wedge2\).

\begin{lemma}
\label{lem:multi-slope-projection}
There exists \(K_0<\infty\), depending only on
\((C_1,C_2,L,\alpha,d)\), such that
\[
 -K_0t^{\bar\alpha/2-1}I_d
 \preceq
 \nabla s_t(x)
 \preceq
 K_0t^{\bar\alpha/2-1}I_d,
 \qquad
 0<t\le1,
\]
and
\[
 -\frac1tI_d
 \preceq
 \nabla s_t(x)
 \preceq
 \left\{\frac{K_0}{t(1+t)}-\frac1t\right\}I_d,
 \qquad t>1.
\]
\end{lemma}

\begin{proof}
For \(0<t\le1\), Gaussian smoothing and the second-difference
characterization of \(C^{\bar\alpha}\) give
 $\|\nabla R_t\|_\infty\lesssim1,
 \|\nabla^2R_t\|_{\mathrm{op},\infty}
 \lesssim t^{\bar\alpha/2-1}$.
Since
\[
 \nabla s_t
 =
 -\frac1{1+t}I_d
 +\frac{\nabla^2R_t}{R_t}
 -\frac{\nabla R_t\nabla R_t^\top}{R_t^2},
\]
and \(C_1\le R_t\le C_2\), the first assertion follows.

For \(t>1\), the multivariate Tweedie identity gives
\[
 \nabla s_t(x)
 =
 \frac1{t^2}
 \operatorname{Cov}(X\mid X+\sqrt tG=x)
 -\frac1tI_d.
\]
The conditional density is a uniformly bounded perturbation of
\(N(x/(1+t),t(1+t)^{-1}I_d)\).  Hence
\[
 0\preceq
 \operatorname{Cov}(X\mid X+\sqrt tG=x)
 \preceq
 K_0\frac{t}{1+t}I_d,
\]
which proves the second assertion.
\end{proof}

Define \(\lambda_d(t)\) and \(\underline\lambda_d(t)\) by
\[
 \lambda_d(t)
 =
 K_0t^{\bar\alpha/2-1}\1_{\{t\le1\}}
 +
 \left\{\frac{K_0}{t(1+t)}-\frac1t\right\}\1_{\{t>1\}},
\qquad
 \underline\lambda_d(t)
 =
 K_0t^{\bar\alpha/2-1}\1_{\{t\le1\}}
 +
 t^{-1}\1_{\{t>1\}},
\]
and let
\[
 \mathcal C_{t,d}
 =
 \left\{
 b=\nabla V\in L^2(\gamma_t;\R^d):
 -\underline\lambda_d(t)I_d
 \preceq\nabla^2V
 \preceq\lambda_d(t)I_d
 \text{ weakly}
 \right\}.
\]
The set \(\mathcal C_{t,d}\) is closed and convex in
\(L^2(\gamma_t;\R^d)\).  Indeed, convergence in this weighted space implies
local \(L^2\) convergence, so the curl-free constraint and the distributional
matrix inequalities are preserved under limits.  Every element of
\(\mathcal C_{t,d}\) has a globally Lipschitz representative.  By
Lemma~\ref{lem:multi-slope-projection}, \(s_t\in\mathcal C_{t,d}\).

Let \(\hat s_t^{\rm agg}\) be the best, according to the deterministic bounds
above, of the low-noise, intermediate-noise, and global plug-in estimators.
The explicit constructions, including the Borel choice of \(Q_x\), may be
chosen jointly measurable in \((t,x)\).  Define
\[
 \widetilde s_t
 =
 \operatorname{Proj}^{L^2(\gamma_t;\R^d)}_{\mathcal C_{t,d}}
 \hat s_t^{\rm agg}.
\]
Then Hilbert projection and \(p_t\asymp\gamma_t\) give
\begin{equation}\label{eq:multi-projected-score-risk}
 \sup_{f\in\mathcal F_{\alpha,d}}
 \E_f\int_{\R^d}\|\widetilde s_t-s_t\|^2p_t
 \lesssim
 \mathfrak r_{n,d}(t),
 \qquad 0<t\le n.
\end{equation}
Moreover, for every \(x,y\in\R^d\),
\[
 \langle\widetilde s_t(x)-\widetilde s_t(y),x-y\rangle
 \le
 \lambda_d(t)\|x-y\|^2.
\]
This is the one-sided Lipschitz bound used in the synchronous-coupling
proof below.

\begin{lemma}[Integration of the multivariate score envelope]
\label{lem:multi-score-integration}
For \(\mathfrak r_{n,d}\) defined in
\eqref{eq:multi-score-envelope},
\[
 \int_0^n\frac{\sqrt{\mathfrak r_{n,d}(t)}}{1+t}\diff t+n^{-1}
 \lesssim
 \begin{cases}
 n^{-1/2}\log\log n,&d=1,\\[1mm]
 n^{-1/2}(\log n)^{3/2},&d=2,\\[1mm]
 n^{-\frac{\alpha+1}{2\alpha+d}},&d\ge3.
 \end{cases}
\]
\end{lemma}

\begin{proof}
Let \(t_0=h_n^2=n^{-2/(2\alpha+d)}\).  The contribution of
\(0<t\le t_0\) is
 $t_0n^{-\frac{\alpha-1}{2\alpha+d}}
 =
 n^{-\frac{\alpha+1}{2\alpha+d}}$.
On \(t_0\le t\le1\), the second and third branches of
\eqref{eq:multi-score-envelope} exchange dominance at a point
\(\tau_n\asymp(\log n)^{-1}\), determined by
\(t\Lambda_{n,d}(t)^2\asymp1\).  Hence the remaining contribution on
\((0,1]\) is bounded by
\[
 \frac1{\sqrt n}
 \int_{t_0}^{\tau_n}
 \Lambda_{n,d}(t)^{d/2}t^{-(d+2)/4}\diff t
 +
 \frac1{\sqrt n}
 \int_{\tau_n}^1t^{-(d+1)/2}\diff t.
\]
For \(d=1\), the first integral is \(O(n^{-1/2})\), while the second is
\(O(n^{-1/2}\log\log n)\).  For \(d=2\), the substitution
\(u=\log(nt^{\alpha+1})\) gives
\[
 \frac1{\sqrt n}
 \int_{t_0}^{\tau_n}
 \frac{\{\log(nt^{\alpha+1})\}^{1/2}}{t}\diff t
 \lesssim
 n^{-1/2}(\log n)^{3/2},
\]
while the second integral is only
\(O(n^{-1/2}(\log n)^{1/2})\).  If \(d>2\), put
\(t=t_0e^v\).  Since
\(nt^{\alpha+d/2}=e^{(2\alpha+d)v/2}\), the first integral is bounded by
\[
 \frac{t_0^{(2-d)/4}}{\sqrt n}
 \int_0^\infty
 (1+v)^{d/4}e^{-(d-2)v/4}\diff v
 \lesssim
 n^{-\frac{\alpha+1}{2\alpha+d}}.
\]
The second integral is
\(O(n^{-1/2}(\log n)^{(d-1)/2})\), which is of smaller order.  Finally,
the contribution of \(t>1\) is
\[
 \frac1{\sqrt n}\int_1^n\frac{\diff t}{t(1+t)}
 \lesssim n^{-1/2}.
\]
\end{proof}

\begin{proof}[Proof of Theorem~\ref{thm:w2-projected-sampler}]
Let \(X\) be the true reverse process
\[
 \diff X_u=s_{n-u}(X_u)\diff u+\diff B_u,
 \qquad X_0\sim p_n,
\]
constructed independently of \(\mathcal D_n\).  By time reversal of the heat
flow (see, e.g., \cite{haussmann1986time}), \(X_n\sim f\).  Conditional on
\(\mathcal D_n\), let \(Y\) solve
\[
 \diff Y_u=\widetilde s_{n-u}(Y_u)\diff u+\diff B_u,
 \qquad Y_0=X_0,
\]
with the same Brownian motion.  This gives a coupling of \(f\) and
\(\mathcal L(Y_n\mid\mathcal D_n)\).

Put \(\Delta_u=Y_u-X_u\).  The one-sided Lipschitz inequality for
\(\widetilde s_{n-u}\) gives, for almost every \(u\),
\[
 \frac{\diff}{\diff u}\|\Delta_u\|
 \le
 \lambda_d(n-u)\|\Delta_u\|
 +
 \|\widetilde s_{n-u}(X_u)-s_{n-u}(X_u)\|.
\]
Gronwall's inequality and the change of variables \(t=n-u\) therefore yield
\[
 \|\Delta_n\|
 \le
 \int_0^n
 \exp\left\{\int_0^t\lambda_d(q)\diff q\right\}
 \|\widetilde s_t(X_{n-t})-s_t(X_{n-t})\|
 \diff t.
\]
For \(0<t\le1\), the integral in the exponential is uniformly bounded.  For
\(t>1\),
\[
 \int_1^t\lambda_d(q)\diff q
 =
 K_0\log\left(\frac{2t}{1+t}\right)-\log t.
\]
Consequently,
\[
 \exp\left\{\int_0^t\lambda_d(q)\diff q\right\}
 \lesssim\frac1{1+t},
 \qquad t>0.
\]
Since \(X_{n-t}\sim p_t\), Minkowski's inequality and
\eqref{eq:multi-projected-score-risk} give
\[
\begin{aligned}
 \left\{\E_fW_2^2\bigl(f,\mathcal L(Y_n\mid\mathcal D_n)\bigr)\right\}^{1/2}
 &\le
 \left(\E_f\|\Delta_n\|^2\right)^{1/2}\lesssim
 \int_0^n\frac1{1+t}
 \left\{
 \E_f\int_{\R^d}\|\widetilde s_t-s_t\|^2p_t
 \right\}^{1/2}\diff t\\
 &\lesssim
 \int_0^n\frac{\sqrt{\mathfrak r_{n,d}(t)}}{1+t}\diff t.
\end{aligned}
\]

It remains to replace the initialization \(p_n\) by \(\gamma_n\).  Couple
\(Y_0\sim p_n\) and \(\widehat Y_0\sim\gamma_n\) optimally and drive the two
projected reverse diffusions with the same Brownian motion.  The same
one-sided Lipschitz calculation gives
\[
 W_2\bigl(\mathcal L(Y_n\mid\mathcal D_n),\widehat\mu_n\bigr)
 \lesssim
 \frac1{1+n}W_2(p_n,\gamma_n)
 \le
 \frac1{1+n}W_2(f,\varphi_d)
 \lesssim n^{-1}.
\]
The last bound is uniform over \(\mathcal F_{\alpha,d}\), since
\(f\le C_2\varphi_d\).  The triangle inequality and
Lemma~\ref{lem:multi-score-integration} complete the proof.
\end{proof}

\section{Proof of Section~\ref{sec:optimal-estimation}}
\label{app:optimal-estimation-proof}
\begin{proof}[Proof of Proposition~\ref{prop:plugin-brenier-map-rate}]
Since \(P,Q\in{\cal C}_{\lambda}^{\Lambda}\), the projection property and the triangle
inequality give
\[
W_2(P,\widetilde P_n)
\le
W_2(P,\widehat P_n)+W_2(\widehat P_n,\widetilde P_n)
\le
2W_2(P,\widehat P_n); \qquad W_2(Q,\widetilde Q_n)\le 2W_2(Q,\widehat Q_n).
\]
Since \(\widetilde P_n, \widetilde Q_n\in {\cal C}_{\lambda}^{\Lambda}\),
applying Caffarelli's contraction theorem of
\cite{caffarelli2000monotonicity} to \(\widetilde T^*_n\) and to its inverse
gives \(\widetilde T^*_n=\nabla\widetilde\varphi_{n}\), whose Brenier potential
\(\widetilde\varphi_{n}\) has two-sided Hessian bounds: it is
\(m_\varphi(\lambda,\Lambda)\)-strongly convex and
\(M_\varphi(\lambda,\Lambda)\)-smooth.

We now apply Theorem~\ref{thm:stability-smooth-ot} with reference map
\(\widetilde T_n^*\), reference marginals
\((\widetilde P_n,\widetilde Q_n)\), and perturbed marginals \((P,Q)\).  Since
the optimal coupling between \(P\) and \(Q\) is induced by \(T^*\), the theorem
controls
\(\int\|T^*(x)-\widetilde T_n^*(x)\|^2\,P(\diff x)\), and gives
\[
\E\|\widetilde T^*_n-T^*\|_{L^2(P)}^2
\lesssim_{m_\varphi,M_\varphi}
\E W_2^2(P,\widetilde P_n) + \E W_2^2(Q,\widetilde Q_n) \lesssim \rho_n^2.
\]
Substituting the projection bounds above yields the claimed
\(\rho_n^2\) rate.

Finally, if the two marginal densities belong to
\(\mathcal F_{\alpha,d}\) and
\(\widehat P_n,\widehat Q_n\) are generated as in
Algorithm~\ref{algo:projected-reverse-sampler}, then
Theorem~\ref{thm:w2-projected-sampler} gives
\(\rho_n=\epsilon_{n,d}\), yielding the stated plug-in rate.
\end{proof}

\begin{proof}[Proof of Theorem~\ref{thm:displacement-cost-stability}]
Define the duality gap evaluated at the reference potentials
    $\mathcal E(\widetilde P,\widetilde Q)
    :=
    W_c(\widetilde P,\widetilde Q)
    -
    \int \varphi\,\diff\widetilde P
    -
    \int \psi\,\diff\widetilde Q$.
Since \(\widetilde\pi\) is optimal for \(W_c(\widetilde P,\widetilde Q)\), we have
\[
    \mathcal E(\widetilde P,\widetilde Q)
    =
    \int
    \{c(y-x)-\varphi(x)-\psi(y)\}
    \,\diff\widetilde\pi(x,y)
    =
    \int D(x,y)\,\diff\widetilde\pi(x,y).
\]

We first prove a lower bound on \(\mathcal E(\widetilde P,\widetilde Q)\).
For a fixed \(x\), the point \(T^*(x)\) minimizes
    $y\mapsto D(x,y)$.
Hence
    $D(x,T^*(x))=0$,
    $\nabla_yD(x,T^*(x))=0$.
By the assumption \(\nabla^2_{yy}D(x,y)\succeq \mu_y I\), Taylor's theorem gives
\[
    D(x,y)
    \ge
    \mu_y \frac{\|y-T^*(x)\|^2}{2}.
\]
Therefore
\[
    \mathcal E(\widetilde P,\widetilde Q)
    =
    \int D(x,y)\,\diff\widetilde\pi(x,y)
    \ge
    \mu_y \int \frac{\|y-T^*(x)\|^2}{2}\,\diff\widetilde\pi(x,y).
\]

It remains to upper bound \(\mathcal E(\widetilde P,\widetilde Q)\).
Let
    $\gamma_P\in\Pi(P,\widetilde P)$
be optimal for \(W_2(P,\widetilde P)\), and let
    $\gamma_Q\in\Pi(Q,\widetilde Q)$
be optimal for \(W_2(Q,\widetilde Q)\). Also define
    $\pi^*:=(\id,T^*)_\#P\in\Pi(P,Q)$.
Since \(\gamma_P\) and \(\pi^*\) have the same first marginal \(P\), the
gluing lemma gives a probability measure on triples \((X,U,Y)\) such that
$(X,U)\sim\gamma_P$,
$(X,Y)\sim\pi^*$.
In particular, \(Y=T^*(X)\) almost surely. Since \(Y\sim Q\), and since the
first marginal of \(\gamma_Q\) is also \(Q\), another application of the gluing
lemma gives a joint distribution of \((X,U,Y,V)\) such that
$(X,U)\sim\gamma_P$,
    $(X,Y)\sim\pi^*$,
    $(Y,V)\sim\gamma_Q$.
Consequently, \((U,V)\in\Pi(\widetilde P,\widetilde Q)\). Therefore
    $W_c(\widetilde P,\widetilde Q)
    \le
    \E c(V-U)$.
Using the marginals of \(U\) and \(V\), we obtain
    $\mathcal E(\widetilde P,\widetilde Q)
    \le
    \E c(V-U)
    -
    \E\varphi(U)
    -
    \E\psi(V)
    =
    \E D(U,V)$.

Now \(D\ge 0\) everywhere, while
    $D(X,Y)=D(X,T^*(X))=0$.
Thus \((X,Y)\) is a global minimizer of \(D\), and hence
    $\nabla_xD(X,Y)=0$,
    $\nabla_yD(X,Y)=0$
almost surely. 
Taylor's theorem, together with the assumed Hessian bounds as well as 
    $\|\nabla^2_{xy}D(x,y)\|_{\mathrm{op}} = \|\nabla^2 c(y-x)\|_{\mathrm{op}} \le L_c$, gives
\[
    \E D(U,V)
    \le
    L_x\E\frac{\|U-X\|^2}{2}
    +
    L_y\E\frac{\|V-Y\|^2}{2}
    +
    L_c\E\big[\|U-X\|\|V-Y\|\big].
\]
Since \((X,U)\sim\gamma_P\) and \((Y,V)\sim\gamma_Q\), and since
\(\gamma_P,\gamma_Q\) are optimal for \(W_2\), the Cauchy--Schwarz inequality
gives
\[
    \mu_y\int \frac{\|y-T^*(x)\|^2}{2}\,\diff\widetilde\pi(x,y)
    \le
    L_x W_2^2(P,\widetilde P)
    +
    L_y W_2^2(Q,\widetilde Q)
    +
    2L_c W_2(P,\widetilde P)W_2(Q,\widetilde Q).
\]
This proves the claim.
\end{proof}

\begin{proof}[Proof of Corollary~\ref{cor:plugin-displacement-cost}]
Apply Theorem~\ref{thm:displacement-cost-stability} with
\((\widetilde P_n,\widetilde Q_n)\) as the reference pair and \((P,Q)\) as the
perturbed pair.  The optimal coupling between \(P\) and \(Q\) is induced by
\(T^c\), so the left-hand side in the theorem becomes the displayed
\(L^2(P)\) error.  The two Wasserstein perturbation terms are both bounded by
\(\rho_n\), which gives the claim.
\end{proof}

\section{Proof of Section~\ref{sec:general-cost-convergence} and ~\ref{sec:beyond-quadratic-contraction}}
\label{app:beyond-quadratic-contraction-proof}

\begin{proof}[Proof of Theorem~\ref{thm:convergence-general-cost}]
Since \(\E [c(Z_1^{(k)}-Z_0^{(k)})]\) is decreasing along the iteration, it suffices to find a subsequence
\(\{k_j\}_{j\in\mathbb N}\) such that
\[
\lim_{j\to\infty}\E [c(Z_1^{(k_j)}-Z_0^{(k_j)})]
=
\inf_{X\sim P,\;Y\sim Q}\E [c(Y-X)].
\]
By the \textit{c}-analogue of \eqref{eq:1K}, after passing to a subsequence, still denoted by \(k\), we may assume
$L_{\vv Z^{(k)}}(f^{\vv Z^{(k)}})\to 0$.

The tightness argument for \(\{(Z_0^{(k)},Z_1^{(k)})\}_{k\ge1}\) is exactly the same as in the proof of Theorem~\ref{thm:convergence-square}. Hence, after passing to a further subsequence,
$(Z_0^{(k)},Z_1^{(k)})\xrightarrow{w}(Z_0^*,Z_1^*)$
for some coupling \(\vv Z^*:=(Z_0^*,Z_1^*)\) of \((P,Q)\).

By the local \(C^{2+\alpha,1+\alpha}\) bound and the same diagonal Arzel\`a--Ascoli argument as in the proof of Theorem~\ref{thm:convergence-square}, after passing to another subsequence we may assume that there exists
\(f^*\in C^{2,1}(\mathbb R^d\times[0,1])\)
such that
$f^{\vv Z^{(k)}}\to f^*$
in $C^{2,1}(K)$
for every compact \(K\subset \mathbb R^d\times[0,1]\).
In particular, $\nabla_x f^{\vv Z^{(k)}}\to \nabla_x f^*$ locally uniformly on $\mathbb R^d\times[0,1]$.
Since \(\nabla c^*\) is continuous, it follows that
\[
g_k(x,s)=(\nabla c^*\circ \nabla f_s^{\vv Z^{(k)}})(x)
\to
g_*(x,s):=(\nabla c^*\circ \nabla f_s^*)(x)
\]
locally uniformly on \(\mathbb R^d\times[0,1]\).

Now, we apply Lemma~\ref{lem:continuity-integrated-loss} with
$h_k(x,t):=g_k(x,t)$,
$h_*(x,t):=g_*(x,t)$,
$\ell(a;b):=b_c(a;b)$.
Here $b_c$ is continuous by assumption, and the preceding argument gives
the required local uniform convergence $g_k\to g_*$.
Define
\[
V_k
:=
\int_0^1
b_c\bigl(
Z_1^{(k)}-Z_0^{(k)};
g_k(Z_t^{(k)},t)
\bigr)\diff t.
\]
By the assumed growth bound on $b_c$,
\begin{align*}
V_k
&\le
C\int_0^1
\left[
1+
\Phi(Z_1^{(k)}-Z_0^{(k)})
+
\Psi\bigl(g_k(Z_t^{(k)},t)\bigr)
\right]\diff t=
C\left[
1+
\Phi(Z_1^{(k)}-Z_0^{(k)})
+
\int_0^1
\Psi\bigl(g_k(Z_t^{(k)},t)\bigr)\diff t
\right].
\end{align*}
By assumption, the two nonconstant terms on the right-hand side form
uniformly integrable families. Hence $\{V_k\}_{k\ge1}$ is uniformly
integrable.
Lemma~\ref{lem:continuity-integrated-loss} therefore yields
\begin{align*}
L_{\vv Z^{(k)}}(f^{\vv Z^{(k)}})
&=
\E[V_k]\longrightarrow
\E\left[
\int_0^1
b_c\bigl(
Z_1^*-Z_0^*;
g_*(Z_t^*,t)
\bigr)\diff t
\right]=
L_{\vv Z^*}(f^*).
\end{align*}
Since
$L_{\vv Z^{(k)}}(f^{\vv Z^{(k)}})\longrightarrow0$,
we conclude that
$L_{\vv Z^*}(f^*)=0$.
Therefore,
$(Z_0^*,Z_1^*)$ is an optimal coupling of $(P,Q)$ under the cost
$c(y-x)$.

Finally, since \((Z_0^{(k)},Z_1^{(k)})\xrightarrow{w}(Z_0^*,Z_1^*)\), we have
\(Z_1^{(k)}-Z_0^{(k)}\xrightarrow{w} Z_1^{*}-Z_0^{*}\). Thus
\(c(Z_1^{(k)}-Z_0^{(k)})\xrightarrow{w} c(Z_1^{*}-Z_0^{*})\). Because \(\{c(Z_1^{(k)}-Z_0^{(k)})\}_{k\ge1}\) is uniformly integrable, Vitali's theorem yields
$\E [c(Z_1^{(k)}-Z_0^{(k)})]\to \E [c(Z_1^{*}-Z_0^{*})]$.
Since \(\vv Z^*\) is \(c\)-optimal, we obtain
\[
\lim_{k\to\infty}\E [c(Z_1^{(k)}-Z_0^{(k)})]
=
\inf_{X\sim P,\;Y\sim Q}\E [c(Y-X)].
\]
This completes the proof.
\end{proof}

\begin{proof}[Proof of Lemma~\ref{lemma:cost_upper_bound}]
    Recall that for \(P\)-a.s.~\(x\), and for all $y\in\R^d$, $c(y-x)-\psi(y)- \varphi(x)\ge 0$,
with equality at \(y=T^*(x)\). Thus, defining
    $D(x,y):=c(y-x)-\psi(y)-\varphi(x)$,
we have that \(T^*(x)\) is a global minimizer of \(D(x,\cdot)\), and hence
    $\nabla_y D(x, T^*(x))=0$
for \(P\)-a.e.~\(x\).

Now, because \(T_\#P=Q\) and \((T^*)_\#P=Q\),
    $\E[\psi(T(X))]
    =
    \E[\psi(T^*(X))]$.
Therefore,
\begin{align*}
\E[c(T(X)-X)-c(T^*(X)-X)] 
&=
\E\Big[
c(T(X)-X)-\psi(T(X))
-
c(T^*(X)-X)+\psi(T^*(X))
\Big] \\
&=
\E[
D(X,T(X))-D(X,T^*(X))
].
\end{align*}
We next bound the last expression. For each fixed \(x\), the Hessian of
\(D\) with respect to \(y\) is
    $\nabla_y^2 D(x,y)
    =
    \nabla^2 c(y-x)-\nabla^2\psi(y)\preceq L_y I_d$.
    % Using
    % $\nabla^2 c(y-x)\preceq L_c I_d$ and $\nabla^2\psi(y)\succeq -\mu_\psi I_d$, we obtain
    % $\nabla_y^2 F_x(y)
    % \preceq
    % (L_c+\mu_\psi)I_d$.
    Hence \(D\) is \(L_y\)-smooth in \(y\). Therefore,
    \[
    D(x,T(x))-D(x,T^*(x))\le
\frac{L_y}{2}\|T(x)-T^*(x)\|^2.
    \]
    Taking expectation with respect to \(P\), we get
\begin{equation}\label{eq:upper_step1}
        \E[c(T(X)-X)] - \E[c(T^*(X)-X)] \le \frac{L_y}{2} \|T - T^*\|_{P}^2.
    \end{equation}
    To relate the transport distance to the gradient difference, we invoke the strong convexity of $c$. Since $c$ is $\mu_c$-strongly convex, its gradient $\nabla c$ is strictly monotone:
    \[
        \langle \nabla c(u) - \nabla c(v), u - v \rangle \ge \mu_c \|u - v\|^2, \quad \forall u, v \in \mathbb{R}^d.
    \]
    Applying the Cauchy-Schwarz inequality, we have $\langle \nabla c(u) - \nabla c(v), u - v \rangle \le \|\nabla c(u) - \nabla c(v)\| \|u - v\|$. Combining these inequalities yields 
        $\mu_c\|u - v\| \le  \|\nabla c(u) - \nabla c(v)\|$.
    Letting $u = T(x) - x$ and $v = T^*(x) - x$, 
    \begin{equation}\label{eq:upper_step2}
        \|T - T^*\|_{P}^2 \le \frac{1}{\mu_c^2} \|\nabla c(T-\id) - \nabla c(T^*-\id)\|_{P}^2.
    \end{equation}
    Substituting \eqref{eq:upper_step2} into \eqref{eq:upper_step1} completes the proof.
\end{proof}

\begin{proof}[Proof of Theorem~\ref{thm:general-cost-one-step-decrease}]
Let
$\tau_L={(2(L+1))^{-1}}$.
By the same invertibility estimate as Proposition~\ref{prop:Tt-invertible},
\(T_t\) is injective for \(0\le t\le \tau_L\), and \(v_t\) is well-defined on
the support of \(P_t\).
    % Let $v_t$ be the velocity field associated with the \textit{c}-rectified flow. 
    Following the \textit{c}-rectified flow variational identity \eqref{eq:c-rect-variational}, 
    since $m_c(x,y)=c(x)+c^*(y)-\langle x, y\rangle$ is convex in $x$, combining with Jensen's inequality, we have 
    \[
        %\mathcal{E} := 
        \E[c(T(X_0)-X_0)] - \E[c(Z_1-Z_0)] \ge \int_0^1 \E [m_c(\dot X_s,\nabla f_s(X_s))] \diff s \ge \int_0^1 \E [m_c(v_s(X_s), \nabla f_s(X_s))]  \diff s,
    \]
    where $f_t$ denotes the dual velocity driving the flow.

    Consider a fixed $0\le t \le \tau_L$, and an arbitrary vector field $h_t \in {\mathcal S}_t$. Since $L^2(P_t;\mathbb R^d)={\mathcal G}_t\oplus {\mathcal S}_t$, we have $\E [\langle h_t(X_t), \nabla f_t(X_t) \rangle] = 0$. Consequently, since $m_c(v_t-h_t, \nabla f_t) \ge 0$,  
    \begin{align*}
        \int m_c(v_t, \nabla f_t)  \diff P_t 
        &\ge \int \left(m_c(v_t, \nabla f_t)-m_c(v_t-h_t, \nabla f_t) \right)  \diff P_t \\
        &= \int \left( c(v_t) - c(v_t - h_t) - \langle h_t, \nabla f_t\rangle \right)  \diff P_t = \int \left( c(v_t) - c(v_t - h_t) \right)  \diff P_t.
    \end{align*}
    Since the cost function $c$ is $L_c$-smooth, 
    % it satisfies the descent inequality $c(y) - c(x) \le \langle \nabla c(x), y-x \rangle + \frac{L_c}{2} \|y-x\|^2$. Setting $x = v_t$ and $y = v_t - h_t$ yields:
    \begin{align*}
        c(v_t) - c(v_t - h_t) &\ge \langle \nabla c(v_t), h_t \rangle - \frac{L_c}{2} \|h_t\|^2 \implies 
    \int m_c(v_t, \nabla f_t)  \diff P_t 
        \ge \langle \nabla c(v_t),h_t\rangle_{P_t}
    -
    \frac{L_c}{2}\|h_t\|_{P_t}^2.
    \end{align*}
    Then, we choose a specific test function $h_t$ via the pushforward of a field $u$ from the source domain $P$. Let $u\in {\mathcal S}_0$ be a $P$-divergence-free field  and define $h_t \circ T_t = (J T_t) u$. Using the change of variables $y = T_t(x)$ and noting that $v_t(T_t(x)) = T(x) - x$,
    \begin{align*}
\langle \nabla c(v_t),h_t\rangle_{P_t}
&=
\int
\left\langle
\nabla c(T(x)-x),JT_t(x)u(x)
\right\rangle
\diff P(x)                                      =
\left\langle
JT_t^\top \nabla c (T-\id),u
\right\rangle_P;\\
\|h_t\|_{P_t}^2 &= \int
\|J T_t(x) u(x)\|^2
\diff P(x).
\end{align*}
First, 
since $J T_t = I + t(J T - I) = I + t J(T-\id)$, by the chain rule,
$$JT_t^\top \nabla c (T-\id) = \nabla c (T-\id) + t J(T-\id)^\top \nabla c (T-\id) = \nabla c (T-\id) + t \nabla (c (T-\id)).$$
Since $u \in {\mathcal S}_0$, we have $u\perp \nabla (c (T-\id))$. As a result, 
$\langle \nabla c(v_t),h_t\rangle_{P_t} = \left\langle
\nabla c (T-\id),u
\right\rangle_P$.

Second,
since \(T\) is \(L\)-Lipschitz and \(t\le \tau_L\),
$\|JT_t(x)\|_{\rm op}
    \le
    (1-t)+tL
    \le 2$. Therefore,
$\|h_t\|_{P_t}^2
    \le
    4\|u\|_P^2$.

As a consequence, we have a cleaner lower bound
\[
    \int m_c(v_t, \nabla f_t)  \diff P_t 
        \ge \left\langle
\nabla c (T-\id),u
\right\rangle_P
    -
    {2L_c}\|u\|_{P}^2.
    \]

    Now we select $u$ proportional to the projection of the cost gradient onto the $P$-divergence-free subspace with  $\varepsilon = (4L_c)^{-1}$, such that
        $u = \varepsilon \proj_{{\mathcal S}_0}(\nabla c(T-\id))$.
    For this test function $u$, we have
    \[
        \int m_c(v_t, \nabla f_t)  \diff P_t \ge \frac{1}{8 L_c} \|\proj_{{\mathcal S}_0}(\nabla c(T-\id))\|_P^2.
    \]
    Finally, invoking Assumption \ref{ass:stability_general}, and integrating over \(t\in[0,\tau_L]\),
\begin{align*}
        %\mathcal{E} := 
        \E[c(T(X_0)-X_0)] - \E[c(Z_1-Z_0)] &\ge  \int_0^{\tau_L} \E [m_c(v_s(X_s), \nabla f_s(X_s))]  \diff s \\
        &\ge \frac{1}{16L_c(L+1)} \|\proj_{{\mathcal S}_0}(\nabla c(T-\id))\|_P^2 \\
        & \ge \frac{h}{16L_c(L+1)} \|\nabla c(T-\id)-\nabla c(T^*-\id)\|_{P}^2\\
        &\ge \frac{h\mu_c^2}{8L_cL_y(L+1)} \left( \E[c(T(X_0)-X_0)] - W_c(P,Q) \right),
    \end{align*}
    which concludes the proof.
\end{proof}

\begin{proof}[Proof of Theorem~\ref{thm:general-cost-iterated-contraction}]
Define
    $\mathcal E_k
    :=
    \E[c(Z_1^{(k)}-Z_0^{(k)})]-W_c(P,Q)$.
By optimality of \(W_c(P,Q)\), we have \(\mathcal E_k\ge0\). Applying
Theorem~\ref{thm:general-cost-one-step-decrease} to the coupling induced
by \(T_k\) for all $k\ge 1$ gives
\[
    \mathcal E_{k+1}
    \le
    \left(
    1-
    \frac{h\mu_c^2}{8L_cL_y(L+1)}
    \right)
    \mathcal E_k.
\]
Iterating this inequality proves the geometric estimate bound.
\end{proof}

\end{appendix}

\bibliographystyle{alpha}
\bibliography{ref-draft}

\end{document}